\documentclass{article}
\usepackage[utf8]{inputenc}
\usepackage{amsthm,amsmath,amsfonts,amssymb,bm,dsfont,enumitem,natbib}
\usepackage{algorithm,algorithmic}
\usepackage[usenames]{color}
\usepackage{geometry}
\usepackage[colorlinks,
            linkcolor=red,
            anchorcolor=blue,
            citecolor=blue
            ]{hyperref}
\usepackage{mystyle}

\DeclareMathOperator*{\argmin}{arg\,min}
\DeclareMathOperator*{\argmax}{arg\,max}
\def\P{{\mathbb{P}}}
\def\mP{{\mathcal{P}}}
\def\E{{\mathbb{E}}}
\def\R{{\mathbb{R}}}
\def\N{{\mathbb{N}}}
\def\mN{{\mathcal{N}}}
\def\mT{{\mathcal{T}}}
\def\bu{{\mathbf{u}}}
\def\rF{{\mathrm{F}}}
\def\TV{{d_{\rm TV}}}

\def\Tr{{{\rm Tr}}}
\def\b{{[~]}}
\def\%#1\%{\begin{align}#1\end{align}}
\def\$#1\${\begin{align*}#1\end{align*}}

\title{On the Provable Advantage of Unsupervised Pretraining}

\author
{\normalsize
Jiawei Ge\thanks{equal contribution}~\thanks{Department of Operations Research and Financial Engineering, Princeton University; 
\texttt{\{jg5300,shangetang,\newline jqfan\}@princeton.edu}}
\qquad Shange Tang \footnotemark[1]~\footnotemark[2]
\qquad Jianqing Fan\footnotemark[2]
\qquad Chi Jin\thanks{Department of Electrical and Computer Engineering, Princeton University; 
\texttt{chij@princeton.edu}}
}
\date{}

\begin{document}

\maketitle

\begin{abstract}

Unsupervised pretraining, which learns a useful representation using a large amount of unlabeled data to facilitate the learning of downstream tasks, is a critical component of modern large-scale machine learning systems. Despite its tremendous empirical success, the rigorous theoretical understanding of why unsupervised pretraining generally helps remains rather limited---most existing results are restricted to particular methods or approaches for unsupervised pretraining with specialized structural assumptions. This paper studies a generic framework,
where the unsupervised representation learning task is specified by an abstract class of latent variable models $\Phi$ and the downstream task is specified by a class of prediction functions $\Psi$. We consider a natural approach of using Maximum Likelihood Estimation (MLE) for unsupervised pretraining and Empirical Risk Minimization (ERM) for learning downstream tasks. We prove that, under a mild ``informative'' condition, our algorithm achieves an excess risk of $\tilde{\mathcal{O}}(\sqrt{\mathcal{C}_\Phi/m} + \sqrt{\mathcal{C}_\Psi/n})$ for downstream tasks, where $\mathcal{C}_\Phi, \mathcal{C}_\Psi$ are complexity measures of function classes $\Phi, \Psi$, and $m, n$ are the number of unlabeled and labeled data respectively. Comparing to the baseline of $\tilde{\mathcal{O}}(\sqrt{\mathcal{C}_{\Phi \circ \Psi}/n})$ achieved by performing supervised learning using only the labeled data, our result rigorously shows the benefit of unsupervised pretraining when $m \gg n$ and $\mathcal{C}_{\Phi\circ \Psi} > \mathcal{C}_\Psi$. This paper further shows that our generic framework covers a wide range of approaches for unsupervised pretraining, including factor models, Gaussian mixture models, and contrastive learning.

\end{abstract}
\section{Introduction}
Unsupervised pretraining aims to efficiently use a large amount of unlabeled data to learn a useful representation that facilitates the learning of downstream tasks. This technique has been widely used in modern machine learning systems including computer vision \citep{caron2019unsupervised,dai2021up}, natural language processing \citep{radford2018improving, devlin2018bert,song2019mass} and speech processing \citep{schneider2019wav2vec,baevski2020wav2vec}. 
Despite its tremendous empirical success, it remains elusive why pretrained representations, which are learned without the information of downstream tasks, often help to learn the downstream tasks.

There have been several recent efforts trying to understand various approaches of unsupervised pretraining from theoretical perspectives, including language models \cite{saunshi2020mathematical, wei2021pretrained}, contrastive learning \cite{arora2019theoretical, tosh2021contrastive2, tosh2021contrastive, haochen2021provable, saunshi2022understanding}, and reconstruction-based self-supervised learning \cite{lee2021predicting}. While this line of works justifies the use of unsupervised pretraining in the corresponding regimes, many of them do not prove the advantage of unsupervised learning, in terms of sample complexity, even when compared to the naive baseline of performing supervised learning purely using the labeled data. Furthermore, these results only apply to particular approaches of unsupervised pretraining considered in their papers, and crucially rely on the specialized structural assumptions, which do not generalize beyond the settings they studied. Thus, we raise the following question:

\textbf{Can we develop a generic framework which provably explains the advantage of unsupervised pretraining?}

This paper answers the above question positively.  

We consider the generic setup where the data $x$ and its label $y$ are connected by an unobserved representation $z$.  Concretely, we assume $(x,z)$ is sampled from a latent variable model $\phi^* $ in an abstract class $\Phi$, and the distribution of label $y$ conditioned on representation $z$ is drawn from distributions $\psi^* $ in class $\Psi$. We considers a natural approach of using Maximum Likelihood Estimation (MLE) for unsupervised pretraining, which approximately learns the latent variable model $\phi^* $ using $m$ unlabeled data. We then use the results of representation learning and Empirical Risk Minimization (ERM) to learn the downstream predictor $\psi^* $ using $n$ labeled data. Investigating this generic setup allows us to bypass the limitation of prior works that are restrictied to the specific approaches for unsupervised pretraining.

We prove that, under a mild ``informative'' condition (Assumption \ref{invariance}), our algorithm achieves a excess risk of $\tilde{\mathcal{O}}(\sqrt{\mathcal{C}_\Phi/m} + \sqrt{\mathcal{C}_\Psi/n})$ for downstream tasks, where $\mathcal{C}_\Phi, \mathcal{C}_\Psi$ are complexity measures of function classes $\Phi, \Psi$, and $m, n$ are the number of unlabeled and labeled data respectively. Comparing to the baseline of $\tilde{\mathcal{O}}(\sqrt{\mathcal{C}_{\Phi \circ \Psi}/n})$ achieved by performing supervised learning using only the labeled data, our result rigorously shows the benefit of unsupervised pretraining when we have abundant unlabeled data $m \gg n$ and when the complexity of composite class $\mathcal{C}_{\Phi\circ \Psi}$ is much greater than the complexity of downstream task alone $\mathcal{C}_\Psi$.

Finally, this paper proves that our generic framework (including the ``informative'' condition) captures a wide range of setups for unsupervised pretraining, including (1) factor models with linear regression as downstream tasks; (2) Gaussian mixture models with classification as downstream tasks; and (3) Contrastive learning with linear regression as downstream tasks.

\subsection{Related work}

\paragraph{Applications and methods for unsupervised pretraining.}
Unsupervised pretraining has achieved tremendous success in image recognition \citep{caron2019unsupervised}, objective detection \citep{dai2021up}, natural language processing \citep{devlin2018bert,radford2018improving,song2019mass} and speech recognition \citep{schneider2019wav2vec,baevski2020wav2vec}. Two most widely-used pretraining approaches are (1) feature-based approaches \citep{brown1992class,mikolov2013distributed,melamud2016context2vec,peters1802deep}, which pretrains a model to extract representations and directly uses the pretrained representations as inputs for the downstream tasks; (2) fine-tuning based approaches, \citep[see, e.g.,][]{devlin2018bert}, which fine-tunes all the model parameters in the neighborhood of pretrained representations based on downstream tasks. \citet{erhan2010does} provides the first empirical understanding on the role of pretraining. They argue that pretraining serves as a form of regularization that effectively guides the learning of downstream tasks.



A majority of settings where pretraining is used fall into the category of semi-supervised learning \citep[see, e.g., ][]{zhu2005semi}, where a large amount of unlabeled data and a small amount of labeled data are observed during the training process. Semi-supervised learning methods aim to build a better predictor by efficiently utilizing the unlabeled data. Some traditional methods include: generative models \citep[e.g.][]{ratsaby1995learning}, low-density separation \citep{joachims1999transductive,lawrence2004semi,szummer2002information}, and graph-based methods \citep{belkin2006manifold}. While most works in this line propose new methods and show favorable empirical performance, they do not provide rigorous theoretical understanding on the benefit of unsupervised pretraining.

\paragraph{Theoretical understanding of unsupervised pretraining.}
Recent years witness a surge of theoretical results that provide explanations for various unsupervised pretraining methods that extract representations from unlabeled data. For example, \citep{saunshi2020mathematical,wei2021pretrained} considers pretraining vector embeddings in the language models, 
while \citep{arora2019theoretical,tosh2021contrastive2,tosh2021contrastive,haochen2021provable,saunshi2022understanding,lee2021predicting} consider several Self-Supervised Learning (SSL) approaches for pretraining. In terms of results, \citet{wei2021pretrained} shows that linear predictor on the top of pretrained languange model can recover their ground truth model; \citet{arora2019theoretical, saunshi2020mathematical,tosh2021contrastive2,tosh2021contrastive,saunshi2022understanding} show that the prediction loss of downstream task can be bounded by the loss of unsupervised pretraining tasks. These two lines of results do not prove the sample complexity advantage of unsupervised learning when compared to the baseline of performing supervised learning purely using the labeled data.


The most related results are \citet{lee2021predicting,haochen2021provable}, which explicitly show the sample complexity advantage of certain unsupervised pretraining methods. However, \citet{lee2021predicting} focuses on reconstruction-based SSL, and critically relies on a conditional independency assumption on the feature and its reconstruction conditioned on the label; 
\citet{haochen2021provable} considers contrastive learning, and their results relies on deterministic feature map and the spectral conditions of the normalized adjacency matrix.
Both results only apply to the specific setups and approaches of unsupervised pretraining in their papers, which do not apply to other setups in general (for instance, the three examples in Section \ref{factor_model}, \ref{gmm}, \ref{contrastive_learning}). On the contrary, this paper develops a generic framework for unsupervised pretraining using only abstract function classes, which applies to a wide range of setups.

\paragraph{Other approaches for representation learning.} There is another line of recent theoretical works that learn representation via multitask learning. \citet{Baxter_2000} provides generalization bounds for multitask transfer learning assuming a generative model and a shared representation among tasks. \citet{JMLR:v17:15-242} theoretically analyses a general method for learning representations from multitasks and illustrates their method in a linear feature setting. \citet{tripuraneni2021provable,du2020few} provide sample efficient algorithms that solve the problem of multitask linear regression. \citet{tripuraneni2020theory} further considers generic nonlinear feature representations and shows sample complexity guarantees for diverse training tasks. Their results differ from our work because they learn representations by supervised learning using labeled data of other tasks, while our work learns representations by unsupervised learning using unlabeled data.


\section{Problem Setup} \label{sec:prob_setup}

\paragraph{Notation.} We denote by $\P(x)$ and $p(x)$ the cumulative distribution function and the probability density function defined on $x\in\mathcal{X}$, respectively. We define $[n]=\{1,2,\ldots,n\}$. The cardinality of set $\mathcal{A}$ is denoted by $|\mathcal{A}|$. Let $\|\cdot\|_2$ be the $\ell_2$ norm of a vector or the spectral norm of a matrix. We denote by $\|\cdot\|_{\rF}$ the Frobenius norm of a matrix. For a matrix $M\in\R^{m\times n}$, we denote by $\sigma_{\min}(M)$ and $\sigma_{\max}(M)$ the smallest singular value and the largest singular value of $M$, respectively. For two probability distributions $\P_1$ and $\P_2$, we denote the Total Variation (TV) distance and the Hellinger distance between these two distributions by $\TV(\P_1,\P_2)$ and $H(\P_1,\P_2)$, respectively. 


We denote by $x\in\mathcal{X}$ and $y\in\mathcal{Y}$ the input data and the objective of the downstream tasks, respectively. Our goal is to predict $y$ using $x$. We assume that $x$ is connected to $y$ through an unobserved latent variable $z\in\mathcal{Z}$ (which is also considered as a representation of $x$). Given the latent variable $z$, the data $x$ and the objective $y$ are independent of each other. To incorporate a large class of real-world applications, such as contrastive learning, we consider the setup where learning can possibly have access to some side information $s\in\mathcal{S}$. We assume that $(x,s,z)\sim\P_{\phi^* }(x,s,z)$ and $y|z\sim\P_{\psi^* }(y|z)$, where $\P_{\phi^* }$ and $\P_{\psi^* }$ are distributions 
indexed by $\phi^* \in\Phi$ and $\psi^* \in\Psi$. It then holds that 
\$
\P_{\phi^* ,\psi^* }(x,y)=\int \P_{\phi^* }(x,z)\P_{\psi^* }(y|z)\, dz,
\$ 
which implies the probability distribution of $(x,y)$ depends on both $\phi^* $ and $\psi^* $. Our setting includes the special case where $y=f^* (z)+\varepsilon$. Function $f^* \in\mathcal{F}$ is the ground truth function  and $\varepsilon\sim\mathcal{N}(0,\sigma^2)$ is a Gaussian noise independent of $z$. In this case, the conditional random variable $y|z\sim\mathcal{N}(f^* (z),\sigma^2)$, whose probability distribution $\P_{f^* }(y|z)$ is parameterized by $f^* \in\mathcal{F}$.


Let $\ell(\cdot,\cdot)$ be a loss function. For any pair $(\phi,\psi)\in\Phi\times\Psi$, the optimal predictor $g_{\phi,\psi}$ is defined as follows,
\%\label{opt_est}
g_{\phi,\psi}\leftarrow\argmin_{g}\E_{\P_{\phi,\psi}}\big[\ell\big(g(x),y\big)\big],
\%
where the minimum is taken on all the possible functions and $\E_{\P_{\phi,\psi}}:=\E_{(x,y)\sim\P_{\phi,\psi}(x,y)}$. Our prediction function class is therefore given by
\$
\mathcal{G}_{\Phi,\Psi}:=\big\{ g_{\phi,\psi} \big| \phi \in \Phi, \psi \in \Psi \big\}.
\$
In particular, if $\ell(\cdot,\cdot)$ is the squared loss function, then the optimal predictor has a closed form solution $g_{\phi,\psi}(x)=\E_{\P_{\phi,\psi}}[y|x]$ and the prediction function class $\mathcal{G}_{\Phi,\Psi}=\{\E_{\P_{\phi,\psi}}[y|x]\,|\, \phi \in \Phi, \psi \in \Psi \}$.

Given an estimator pair  $(\hat\phi,\hat\psi)$, we define the excess risk with respect to loss $\ell(\cdot,\cdot)$ as
\%\label{error}
{\rm Error}_{\ell}(\hat\phi,\hat\psi):=\E_{\P_{\phi^* ,\psi^* }}\big[\ell\big(g_{\hat\phi,\hat\psi}(x),y\big)\big]-\E_{\P_{\phi^* ,\psi^* }}\big[\ell\big(g_{\phi^* ,\psi^* }(x),y\big)\big],
\%
where $\phi^* $ and $\psi^* $ are the ground truth parameters. By the definition of $g_{\phi^* ,\psi^* }$, we have ${\rm Error}(\hat\phi,\hat\psi)\geq0$. We aim to learn an estimator pair $(\hat\phi,\hat\psi)$ from data that achieves smallest order of the excess risk.


We consider the setting where the latent variable $z$ cannot be observed. Specifically, we are given many unlabeled data and its corresponding side information $\{x_i,s_i\}^m_{i=1}$ that are sampled i.i.d from an unknown distribution $\P_{\phi^* }(x,s)$ and only a few labeled data $\{x_j,y_j\}^n_{j=1}$ that are sampled i.i.d from an unknown distribution $\P_{\phi^* ,\psi^* }(x,y)$. Here we assume that the labeled data $\{x_j,y_j\}^n_{j=1}$ is independent with the unlabeled data $\{x_i,s_i\}^m_{i=1}$ with understanding $m\gg n$.

\paragraph{Learning algorithm.} We consider a natural learning algorithm consisting of two phases (Algorithm \ref{mle+erm}). In the unsupervised pretraining phase, we use MLE to estimate $\phi^* $ based on the unlabeled data $\{x_i,s_i\}^m_{i=1}$. In the downstream tasks learning phase, we use ERM to estimate $\psi^* $ based on pretrained $\hat\phi$ and the labeled data $\{x_j,y_j\}^{n}_{j=1}$. See algorithm \ref{mle+erm} for details.

\begin{algorithm}[t]
\caption{Two-Phase MLE+ERM}\label{mle+erm}
\begin{algorithmic}[1]
\STATE {\bf Input:} $\{x_i,s_i\}^m_{i=1}$, $\{x_j,y_j\}^n_{j=1}$
\STATE Use unlabeled data and its corresponding side information $\{x_i,s_i\}^m_{i=1}$ to learn $\hat\phi$ via MLE:
\begin{equation}\label{mle}
\hat\phi\leftarrow\argmax_{\phi\in\Phi}\sum^m_{i=1}\log p_{\phi}(x_i,s_i).
\end{equation}
\STATE Fix $\hat\phi$ and use labeled data $\{x_j,y_j\}^n_{j=1}$ to learn $\hat\psi$ via ERM:
\begin{equation}\label{erm}
\hat\psi\leftarrow\argmin_{\psi\in\Psi}\sum^n_{j=1}\ell\big(g_{\hat\phi,\psi}(x_j),y_j\big).
\end{equation}
\STATE {\bf Output:} $\hat\phi$ and $\hat\psi$.
\end{algorithmic}
\end{algorithm}

We remark that another natural learning algorithm in our setting is to use a two-phase MLE. To be specific, in the unsupervised pretraining phase, we use MLE to estimate $\phi^* $ based on the unlabeled data $\{x_i,s_i\}^m_{i=1}$ as \eqref{mle}. In the downstream tasks learning phase, we again use MLE to estimate $\psi^* $ based on pretrained $\hat\phi$ and the labeled data $\{x_j,y_j\}^{n}_{j=1}$. However, we can show that this two-phase MLE scheme fails in the worst case. See Appendix \ref{counter_example} for the details.


\paragraph{Complexity measures.} Sample complexity guarantee for Algorithm \ref{mle+erm} will be phrased in terms of three complexity measurements, i.e., bracketing number, covering number and the Rademacher complexity, which are defined as follows. We denote by $\mP_{\mathcal{X}}(\Phi)$ a set of parameterized density functions $p_{\phi}(x)$ defined on $x\in\mathcal{X}$
\$
\mP_{\mathcal{X}}(\Phi):=\{p_{\phi}(x)\,|\,\phi\in\Phi\},
\$
where $\phi\in\Phi$ is the parameter.

\begin{definition}[$\epsilon$-Bracket and Bracketing Number] \label{def:bracketing}
Let $\epsilon>0$. Under $\|\cdot\|_1$ distance, a set of functions $\mN_{\b}(\mP_{\mathcal{X}}(\Phi),\epsilon)$ is an $\epsilon$-bracket of $\mP_{\mathcal{X}}(\Phi)$ if for any $p_{\phi}(x)\in\mP_{\mathcal{X}}(\Phi)$, there exists a function $\bar p_{\phi}(x)\in\mN_{\b}(\mP_{\mathcal{X}}(\Phi),\epsilon)$ such that the following two properties hold:
\begin{itemize}
    \item $\bar p_{\phi}(x)\geq p_{\phi}(x),~\forall x\in\mathcal{X}$
    \item $\|\bar p_{\phi}(x)-p_{\phi}(x)\|_1=\int |\bar p_{\phi}(x)- p_{\phi}(x)|\, dx\leq \epsilon$
\end{itemize}
Note that $\bar p_{\phi}(x)$ need not to belong to $\mP_{\mathcal{X}}(\Phi)$. The bracketing number $N_{\b}(\mP_{\mathcal{X}}(\Phi),\epsilon)$ is the cardinality of the smallest $\epsilon$-bracket needed to cover $\mP_{\mathcal{X}}(\Phi)$. The entropy is defined as the logarithm of the bracketing number.
\end{definition}

To measure the complexity of a function class, we consider the covering number and the Rademacher complexity defined as follows.

\begin{definition}[$\epsilon$-Cover and Covering Number]
Let $\mathcal{F}$ be a function class and $(\mathcal{F},\|\cdot\|)$ be a metric space. For each $\epsilon>0$, a set of functions $\mathcal{N}(\mathcal{F},\epsilon,\|\cdot\|)$ is called an $\epsilon$-cover of $\mathcal{F}$ if for any $f\in\mathcal{F}$, there exists a function $g\in\mathcal{N}(\mathcal{F},\epsilon,\|\cdot\|)$ such that $\|f-g\|\leq \epsilon$. The covering number $N(\mathcal{F},\epsilon,\|\cdot\|)$ is defined as the cardinality of the smallest $\epsilon$-cover needed to cover $\mathcal{F}$.
\end{definition}

\begin{definition}[Rademacher Complexity] \label{def:Rademacher}
Suppose that $x_1,\ldots,x_n$ are sampled i.i.d from a probability distribution $\mathcal{D}$ defined on a set $\mathcal{X}$. Let $\mathcal{G}$ be a class of functions mapping from $\mathcal{X}$ to $\R$. The empirical Rademacher complexity of $\mathcal{G}$ is defined as follows,
\$
\hat R_n(\mathcal{G}):=\E_{\{\sigma_i\}_{i=1}^n \sim \text{Unif}\{\pm 1\}}\bigg[\sup_{g\in\mathcal{G}}\frac{2}{n}\sum^{n}_{i=1}\sigma_i g(x_i)\bigg],
\$
where $\{\sigma_i\}^n_{i=1}$ are independent random variables drawn from the Rademacher distribution and the expectation is taken over the randomness of $\{\sigma_i\}^n_{i=1}$. The Rademacher complexity of $\mathcal{G}$ is defined as
\$
R_n(\mathcal{G}):=\E_{\{x_i\}_{i=1}^n \sim \mathcal{D}}[\hat R_n(\mathcal{G})].
\$
\end{definition}

\section{Main Results}\label{main_results}
In this section, we first introduce a mild ``informative'' condition for unsupervised pretraining. We show this ``informative'' condition is necessary for pretraining to benefit downstream tasks. We then provide our main results---statistical guarantees for unsupervised pretraining and downstream tasks for Algorithm \ref{mle+erm}. Finally, in Section \ref{weak_informative_model}, we generalize our results to a more technical but weaker version of the ``informative'' condition, which turns out to be useful in capturing our third example of contrastive learning (Section \ref{contrastive_learning}).




\paragraph{Informative pretraining tasks.}
We first note that under our generic setup, unsupervised pretraining may not benefit downstream tasks at all in the worst case if no further conditions are assumed.
\begin{proposition} \label{prop:informative_necessary}
There exist classes $(\Phi, \Psi)$ as in Section \ref{sec:prob_setup} such that, regardless of unsupervised pretraining algorithms used, pretraining using unlabeled data provides no additional information towards learning predictor $g_{\phi^* , \psi^* }$.
\end{proposition}
Consider the latent variable model $z=Ax$, where $x \sim \mathcal{N}(0,I_{d})$, $A\in \Phi$ is the parameter of the model. Then, no matter how many unlabeled $\{x_{i}\}$ we have, we can gain no information of $A$ from the data! In this case, unsupervised pretraining is not beneficial for any downstream task. 

Therefore, it's crucial to give an assumption that guarantees our unsupervised pretraining is informative. As a thought experiment, suppose that in the pretraining step, we find an exact density estimator $\hat{\phi}$ for the marginal distribution of $x,s$ , i.e., $p_{\hat{\phi}}(x,s)=p_{\phi^{*}}(x,s)$ holds for every $x,s$. We should expect that this estimator also fully reveals the relationship between $x$ and $z$, i.e., $p_{\hat{\phi}}(x,z)=p_{\phi^{*}}(x,z)$ holds for every $x,z$. Unfortunately, this condition does not hold in most practical setups and is often too strong. As an example, consider Gaussian mixture models, where $z \in [K]$ is the cluster that data point $x \in \mathbb{R}^{d}$ belongs to. Then in this case, it is impossible for us to ensure $p_{\hat{\phi}}(x,z)=p_{\phi^{*}}(x,z)$, since a permutation of $z$ makes no difference in the marginal distribution of $x$. However, notice that in many circumstances, a permutation of the class label will not affect the downstream task learning. In these cases, a permutation of the clusters is allowed. Motivated by this observation, we introduce the following informative assumption which allows certain ``transformation'' induced by the downstream task:


\begin{assumption}[$\kappa^{-1}$-informative condition]\label{invariance}
We assume that the model class $\Phi$ is $\kappa^{-1}$-informative with respect to a transformation group $\mathcal{T}_{\Phi}$. That is, for any $\phi\in\Phi$, there exists $T_1\in\mT_{\Phi}$ such that
\%\label{informative}
\TV\big(\P_{T_1\circ\phi}(x,z),\P_{\phi^* }(x,z)\big)\leq  \kappa \cdot\TV\big(\P_{\phi}(x,s),\P_{\phi^* }(x,s)\big).
\%
Here $\phi^* $ is the ground truth parameter. Furthermore, we assume that $\mathcal{T}_{\Phi}$ is induced by transformation group $\mathcal{T}_{\Psi}$ on $\Psi$, i.e., for any $T_1\in\mT_{\Phi}$, there exists $T_2\in\mT_{\Psi}$ such that for any $(\phi,\psi)\in\Phi\times\Psi$,
\%\label{101701}
\P_{\phi,\psi}(x,y)=\P_{T_1\circ\phi,T_2\circ\psi}(x,y).
\%
\end{assumption}

Under Assumption \ref{invariance}, if the pretrained $\hat{\phi}$ accurately estimates the marginal distribution of $x,s$ up to high accuracy, then it also reveals the correct relation between $x$ and representation $z$ up to some transformation $\mathcal{T}_{\Phi}$ which is allowed by the downstream task, which makes it possible to learn the downstream task using less labeled data. 

Proposition \ref{prop:informative_necessary} shows that the informative condition is necessary for pretraining to bring advantage since the counter example in the proposition is precisely $0$-informative. We will also show this informative condition is rich enough to capture a wide range of unsupervised pretraining methods in Section \ref{factor_model}, \ref{gmm}, \ref{contrastive_learning}, including factor models, Gaussian mixture models, and contrastive learning models.



\paragraph{Guarantees for unsupervised pretraining.}
Recall that $\mP_{\mathcal{X}\times\mathcal{S}}(\Phi):=\{p_{\phi}(x,s)\,|\,\phi\in\Phi\}$. We have the following guarantee for the MLE step (line 2) of Algorithm \ref{mle+erm}.

\begin{theorem}\label{tv_mle}
Let $\hat\phi$ be the maximizer defined in \eqref{mle}. Then, with probability at least $1-\delta$, we have
\$
\TV\big(\P_{\hat\phi}(x,s),\P_{\phi^* }(x,s)\big)\leq 3\sqrt{\frac{1}{m}\log\frac{N_{\b}(\mP_{\mathcal{X}\times\mathcal{S}}(\Phi),\frac{1}{m})}{\delta}},
\$
where $N_{\b}$ is the bracketing number as in Definition \ref{def:bracketing}.
\end{theorem}

Theorem \ref{tv_mle} claims that the TV error in estimating the joint distribution of $(x, s)$ decreases as $\mathcal{O}(\mathcal{C}_\Phi/m)$ where $m$ is the number of unlabeled data, and $\mathcal{C}_\Phi = \log N_{\b}(\mP_{\mathcal{X}\times\mathcal{S}}(\Phi), 1/m)$ measures the complexity of learning the latent variable models $\Phi$. 
This result mostly follows from standard analysis of MLE \citep{van2000empirical}. We include the proof in Appendix \ref{main1} for completeness.  If the model is $\kappa^{-1}$-informative, Theorem \ref{tv_mle} further implies that with probability at least $1-\delta$,
\$
&\min_{\psi}\E_{\P_{\phi^* ,\psi^* }}\big[\ell\big(g_{\hat\phi,\psi}(x),y\big)\big]-\E_{\P_{\phi^* ,\psi^* }}\big[\ell\big(g_{\phi^* ,\psi^* }(x),y\big)\big]\leq 12\kappa L\sqrt{\frac{1}{m}\log\frac{N_{\b}(\mP_{\mathcal{X}\times\mathcal{S}}(\Phi),1/m)}{\delta}}.
\$
See Lemma \ref{error_dtv} for the details. This inequality claims that if we learn a \emph{perfect} downstream predictor using the estimated representation $\hat{\phi}$, excess risk is small.

\paragraph{Guarantees for downstream task learning.}
In practice, we can only learn an approximate downstream predictor using a small amount of labeled data. We upper bound the excess risk of Algorithm \ref{mle+erm} as follows.
\begin{theorem}\label{error_bound}
Let $\hat\phi$ and $\hat\psi$ be the outputs of Algorithm \ref{mle+erm}. Suppose that the loss function $\ell:\mathcal{Y}\times\mathcal{Y}\rightarrow \R$ is $L$-bounded and our model is $\kappa^{-1}$-informative. Then, with probability at least $1-\delta$, the excess risk of Algorithm \ref{mle+erm} is bounded as:
\$
{\rm Error}_{\ell}(\hat\phi,\hat\psi)&\leq 2\max_{\phi\in\Phi} R_n(\ell\circ \mathcal{G}_{\phi,\Psi})+12\kappa L\cdot\sqrt{\frac{1}{m}\log\frac{2N_{\b}(\mP_{\mathcal{X}\times\mathcal{S}}(\Phi),1/m)}{\delta}}+2L\cdot\sqrt{\frac{2}{n}\log\frac{4}{\delta}}.
\$
Here $R_n(\cdot)$ denotes the Rademacher complexity, and
\$
\ell\circ \mathcal{G}_{\phi,\Psi}:=\big\{\ell\big( g_{\phi,\psi}(x),y\big):\mathcal{X}\times\mathcal{Y}\rightarrow [-L,L]\,\big|\,\psi\in\Psi\big\}.
\$
\end{theorem}

Note that the Rademacher complexity of a function class can be bounded by its metric entropy. We then have the following corollary.
\begin{corollary}\label{rc_covering}
Under the same preconditions as Theorem \ref{error_bound}, we have:
\$
{\rm Error}_{\ell}(\hat\phi,\hat\psi)&\leq \tilde{c} \max_{\phi\in\Phi} L \sqrt{\frac{\log N (\ell\circ \mathcal{G}_{\phi,\Psi},L/\sqrt{n},\|\cdot\|_{\infty})}{n}}+2L\sqrt{\frac{2}{n}\log\frac{4}{\delta}}\notag\\
 &~+12\kappa L\sqrt{\frac{1}{m}\log\frac{2N_{\b}(\mP_{\mathcal{X}\times\mathcal{S}}(\Phi),1/m)}{\delta}}\notag,
\$
where $\tilde{c}$ is an absolute constant, $N(\mathcal{F}, \delta, \| \cdot \|_{\infty})$ is the $\delta-$covering number of function class $\mathcal{F}$ with respect to the metric $\| \cdot \|_{\infty}$.
\end{corollary} 
By Corollary \ref{rc_covering}, 
the excess risk of our Algorithm \ref{mle+erm} is approximately $\tilde{\mathcal{O}}(\sqrt{\mathcal{C}_\Phi/m} + \sqrt{\mathcal{C}_\Psi/n})$, where $\mathcal{C}_\Phi$ and $\mathcal{C}_\Psi$ are roughly the log bracketing number of class $\Phi$ and the log covering number of $\Psi$. Note that excess risk for the baseline algorithm that learns downstream task using only labeled data is $\tilde{\mathcal{O}}( \sqrt{\mathcal{C}_{\Phi \circ \Psi}/n})$, where $\mathcal{C}_{\Phi \circ \Psi}$ is the log covering number of composite function class $\Phi \circ \Psi$.  In many practical scenarios such as training a linear predictor on top of a pretrained deep neural networks, the complexity $\mathcal{C}_{\Phi \circ \Psi}$ is much larger than $\mathcal{C}_{\Psi}$. We also often have significantly more unlabeled data than labeled data ($m \gg n$). In these scenarios, our result rigorously shows the significant advantage of unsupervised pretraining compared to the baseline algorithm which directly performs supervised learning without using unlabeled data.
\subsection{Guarantees for weakly informative models}\label{weak_informative_model}


We introduce a relaxed version of Assumption \ref{invariance}, which allows us to capture a richer class of examples.
\begin{assumption}[$\kappa^{-1}$-weakly-informative condition]\label{weak_invariance}
We assume model $(\Phi, \Psi)$ is $\kappa^{-1}$-weakly-informative, that is, for any $\phi\in\Phi$, there exists $\psi\in\Psi$ such that
\%\label{weak_informative}
\TV\big(\P_{\phi,\psi}(x,y),\P_{\phi^* ,\psi^* }(x,y)\big)\leq \kappa \cdot H\big(\P_{\phi}(x,s),\P_{\phi^* }(x,s)\big).
\%
Here we denote by $\phi^{*}, \psi^{*}$ the ground truth parameters.
\end{assumption}

Assumption \ref{weak_invariance} relaxes Assumption \ref{invariance} by making two modifications: (i) replace the LHS of \eqref{informative} by the TV distance between the joint distribution of $(x, y)$; (ii) replace the TV distance on the RHS by the Hellinger distance. See more on the relation of two assumptions in Appendix \ref{relation}.


In fact, Assumption \ref{weak_invariance} is sufficient for us to achieve the same theoretical guarantee as that in Theorem \ref{error_bound}.

\begin{theorem}\label{weak_error_bound}
Theorem \ref{error_bound} still holds under the $\kappa^{-1}$-weakly-informative assumptions.
\end{theorem}

The proof of Theorem \ref{weak_error_bound} requires a stronger version of MLE guarantee than Theorem \ref{tv_mle}, which guarantees the closeness in terms of Hellinger distance. We leave the details in Appendix \ref{main4}.




\section{Pretraining via Factor Models}\label{factor_model}

High-dimensional data is very common in modern statistics and machine learning, and we often suffer from the curse of dimensionality when directly analyzing data in high-dimensional spaces. To tackle the problem, we usually assume that high-dimensional data has some low-dimensional structures. One of the widely studied models in this setting is the factor model, which models the high-dimensional measurements by low-rank plus sparse structures in data matrices to decorrelate the covariates. Learning this latent structure falls into the framework of unsupervised statistical learning. In this section, we instantiate our theoretical framework using the factor model with linear regression as a downstream task. We rigorously show how unsupervised pretraining can help reduce sample complexity in this case.

\paragraph{Model Setup.}
Factor model \citep[see, e.g.,][]{lawley1971factor,bai2002determining,forni2005generalized,fan2021robust} is widely used in finance, computational biology, and sociology, where the high-dimensional measurements are strongly correlated. For the latent variable model, we consider the factor model with standard Gaussian components, which is defined as follows.
\begin{definition}[Factor Model]
Suppose that we have $d$-dimensional random vector $x$, whose dependence is driven by $r$ factors $z$. The factor model assumes 
\$
x=B^* z+\mu,
\$
where $B^* $ is a $d\times r$ factor loading matrix. Here $\mu\sim N(0, I_d)$ is the idiosyncratic component that is uncorrelated with the common factor $z\sim N(0,I_r)$. We assume that the ground truth parameters $B^* \in\mathcal{B}$, where $\mathcal{B}:=\{B\in\R^{d\times r}\,|\,\|B\|_2\leq D\}$ for some $D>0$.
\end{definition}

For the downstream task, we assume that the latent factors $z$ influence on the response $y$ in a similar manner as on $x$ and hence consider the following linear regression problem
\$
y = \beta^{* T}z+\nu,
\$
where $\nu\sim N(0,\varepsilon^2)$  is a Gaussian noise that is uncorrelated with the factor $z$ and the idiosyncratic component $\mu$. We assume that the ground truth parameters $\beta^* \in\mathcal{C}$, where $\mathcal{C}:=\{\beta\in\R^r\,|\,\|\beta\|_2\leq D\}$ for some $D>0$. The latent variable model (i.e., $\Phi$) and the the prediction class (i.e.,$\Psi$) are then represented by $\mathcal{B}$ and $\mathcal{C}$, respectively. In the sequel, we consider the case where no side information is available, i.e., we only have access to i.i.d unlabeled data $\{x_i\}^m_{i=1}$ and i.i.d labeled data $\{x_j,y_j\}^n_{j=1}$.

For regression models, it is natural to consider the squared loss function $\ell (x,y):=(y-x)^2$.
Then, the optimal predictor $g_{B,\beta}$ under the distribution $\P_{B,\beta}$ has the following closed form solution,
\$
g_{B,\beta}(x)=\E_{\P_{B,\beta}}[y\,|\,x]=\beta^{T}B^{T}(BB^T+ I_d)^{-1}x.
\$
And the excess risk is now defined as
\$
{\rm Error}_{\ell}(\hat B,\hat\beta):=\E_{\P_{B^* ,\beta^* }}\big[\big(y-g_{\hat B,\hat\beta}(x)\big)^2\big]-\E_{\P_{B^* ,\beta^* }}\big[\big(y-g_{B^* ,\beta^* }(x)\big)^2\big].
\$


\paragraph{Informative condition.}
We first show that Assumption \ref{invariance} holds for the factor model with linear regression as downstream tasks. The idea of the factor model is to learn a low-dimensional representation $z$, where a rotation over $z$ is allowed since in the downstream task, we can also rotate $\beta$ to adapt to the rotated $z$.

\begin{lemma}\label{factor_ti}
Factor model with linear regression as downstream tasks is $\kappa^{-1}$-informative, where
\$
\kappa=\frac{c_1(\sigma^* _{\max}+1)^4}{(\sigma^* _{\min})^3}.
\$
Here $c_1$ is some absolute constants, $\sigma^* _{\max}$ and $\sigma^* _{\min}$ are the largest and smallest singular value of $B^* $, respectively.
\end{lemma}

\paragraph{Theoretical results.}
Recall that in Theorem \ref{error_bound}, we assume a $L$-bounded loss function to guarantee the performance of Algorithm \ref{mle+erm}. Thus, instead of directly applying Algorithm \ref{mle+erm} to the squared loss function, we consider Algorithm \ref{mle+erm} with truncated squared loss, i.e.,
\begin{equation}\label{eq:truncated_loss}
\tilde{\ell} (x,y) := (y-x)^{2}\cdot \mathds{1}_{\{(y-x)^2\leq L\}} + L \cdot \mathds{1}_{\{(y-x)^2 > L\}}.
\end{equation}
Here $L$ is a carefully chosen truncation level. To be more specific, in the first phase, we still use MLE to learn an estimator $\hat B$ as that in line 2 of Algorithm \ref{mle+erm}. In the second phase, we apply ERM to the truncated squared loss to learn an estimator $\hat\beta$, i.e.,
\$
\hat\beta\leftarrow\argmin_{\beta\in\mathcal{C}}\sum^n_{j=1}\tilde{\ell}\big(g_{\hat B,\beta}(x_j),y_j\big).
\$
We then have the following theoretical guarantee.
\begin{theorem}\label{factor_main}
We consider Algorithm \ref{mle+erm} with truncated squared loss \eqref{eq:truncated_loss}
with $L=(D^{2}+1)^{3}\log n$. Let $\hat B, \hat\beta$ be the outputs of Algorithm \ref{mle+erm}. Then, for factor models with linear regression as downstream tasks, with probability at least $1-\delta$, the excess risk can be bounded as follows,
\$
{\rm Error}_{\ell}(\hat B,\hat\beta)\leq \Tilde{\mathcal{O}}\bigg(\kappa L \sqrt{\frac{dr}{m}}+L\sqrt{\frac{r}{n}}\bigg),
\$
where $D$ is defined in the sets ${\cal B}$ and ${\cal C}$, and $\kappa$ is specified in Lemma \ref{factor_ti}.
Here $\tilde{\mathcal{O}}(\cdot)$ omits absolute constants and the polylogarithmic factors in $m, d, r, D, 1/\delta$.
\end{theorem}

Notice that the rate we obtain in Theorem \ref{factor_main} is not optimal for this specific task: by the nature of squared loss, if we consider a direct $d-$dimensional linear regression (from $x$ to $y$) with $n$ data, we can usually achieve the fast rate, where excess risk decreases as $\tilde{\mathcal{O}}(d/n)$. To fill this gap, we consider Algorithm \ref{mle+erm} with $\Phi=\R^{d\times r}$ and $\Psi=\R^{r}$ and denote $D:=\max\{\|B^* \|_2,\|\beta^* \|_2\}$. Following a more refined analysis other than using a uniform concentration technique (which is suitable for general problems but not optimal in this specific task), we achieve the following theoretical guarantee with a sharper risk rate: 

\begin{theorem}[Fast rate] \label{factor_fast_rate}
Let $\hat B, \hat\beta$ be the outputs of Algorithm \ref{mle+erm}. Then, if $m \gtrsim (D^{2}+1)^{2} d \log (1/\delta)$, $n \gtrsim (D^{2}+1)^{2} r \log (1/\delta)$, for factor models with linear regression as downstream tasks, with probability at least $1-\delta$, the excess risk can be bounded as follows, 
\begin{align*}
{\rm Error}_{\ell}(\hat B,\hat\beta) \leq \mathcal{O} \bigg((D^{2}+1)^{6}(D^{4}+\sigma_{\min}^{* -4})\frac{d \log (1/\delta)}{m}+ (D^{2}+1)^{2}\frac{r \log(4/\delta)}{n }\bigg).
\end{align*}
Here $\mathcal{O}(\cdot)$ omits some absolute constants.
\end{theorem}


Theorem \ref{factor_fast_rate} shows the benefit of unsupervised pretraining in the following sense. Assuming $D$ and $\sigma^{*}_{\min}$ are both constants. The price paid for learning the loading matrix is $\Tilde{\mathcal{O}}(d/m) $, which is small when $m$ is very large. Notice that, since $x$ is a $d$-dimensional vector, the usual linear regression will have a risk of $\Tilde{\mathcal{O}}(d/n)$. In the risk bound provided by Theorem \ref{factor_fast_rate},  the risk related to $n$ scales as $\Tilde{\mathcal{O}}(r/n)$. Usually, the factor is assumed to be low-dimensional compared with the input ($d \gg r$). Then when $m \gg n$, the risk bound $\Tilde{\mathcal{O}}(d/m + r/n)$ is much better than $\Tilde{\mathcal{O}}(d/n)$.


\section{Pretraining via Gaussian Mixture Models}\label{gmm}
In this section, we show how pretraining using Gaussian Mixture Models (GMMs) can benefit the downstream classification tasks, under our theoretical framework.

\paragraph{Model setup.}
For the latent variable model, we consider a $d$-dimensional GMM with $K$ components and equal weights. To be specific, the latent variable $z$ that represents the cluster is sampled uniformly from $[K]$. In each cluster, the data is sampled from a standard Gaussian distribution, i.e., $x|z=i\sim\mN(u^* _i,I_d)$ for any $i\in[K]$. It then holds that
\begin{equation*}
x\sim\sum^K_{i=1}\frac{1}{K}\mN(u^* _i,I_d).
\end{equation*}
We denote by $\mathcal{U}$ the parameter space with each element consisting of $K$ centers ($d$-dimensional vectors). 

We assume that the set of parameters $\mathcal{U}$ satisfies the normalization condition---there exists $D>0$ such that for any $\mathbf{u}=\{u_i\}^K_{i=1}\in\mathcal{U}$, we have $\|u_i\|_2\leq D\sqrt{d\log K},~\forall i\in[K]$. We further assume the ground-truth centers $\{u^* _i\}^K_{i=1}\in\mathcal{U}$ satisfy the following separation condition.

\begin{assumption}[Separation condition]\label{gmm_separation}
The true parameters $\{u^* _i\}^K_{i=1}\in\mathcal{U}$ satisfies
\$
\|u^* _i-u^* _j\|_2\geq 100\sqrt{d\log K},~\forall i\neq j.
\$
\end{assumption}

For the downstream task, we consider the binary classification problems with label $y\in\{0,1\}$. We denote by $\Psi$ the set of $2^K$ classifiers such that for each $\psi\in\Psi$, and any $i \in [K]$, we have either $\P_{\psi}(y=1|z=i)=1-\varepsilon$ or $\P_{\psi}(y=0|z=i)=1-\varepsilon$, 
where $\varepsilon$ represents the noise. Then, the latent variable model and the prediction class are represented by $\mathcal{U}$ and $\Psi$, respectively. In the sequel, we consider the case where no side information is available, i.e., we only have access to i.i.d unlabeled data $\{x_i\}^m_{i=1}$ and i.i.d labeled data $\{x_j,y_j\}^n_{j=1}$. For classification problems, it is natural to consider the $0-1$ loss function $\ell(x,y):=\mathds{1}_{\{x\neq y\}}$ which is bounded by $1$.



\paragraph{Informative condition.} We prove that Assumption \ref{invariance} for the above model. We have the following guarantee.

\begin{lemma}\label{gmm_ti} 
Let $\tilde{\mathcal{U}} = \{\bu\in\mathcal{U} ~|~ \TV(p_{\bu}(x),p_{\bu^* }(x))\leq 1/(4K)\}$.
Under Assumption \ref{gmm_separation}, GMMs with parameters in $\tilde{\mathcal{U}}$ is $\mathcal{O}(1)$-informative with respect to the transformation group induced by downstream classification tasks.
\end{lemma}

\paragraph{Theoretical results}
We have the following theoretical guarantee.
\begin{theorem}\label{gmm_main}
Let $\hat\bu,\hat\psi$ be the outputs of Algorithm \ref{mle+erm}. Suppose that Assumption \ref{gmm_separation} holds and $m=\tilde{\Omega}(dK^3)$. Then, for the Gaussian mixture model with classification as downstream tasks, with probability at least $1-\delta$, the excess risk can be bounded as follows,
\begin{equation*}
{\rm Error}_{\ell}(\hat\bu,\hat\psi)\leq\tilde{\mathcal{O}}\bigg(\sqrt{\frac{dK}{m}}+\sqrt{\frac{K}{n}}\bigg),
\end{equation*}
Here $\tilde{\mathcal{O}}(\cdot)$ omits some constants and the polylogarithmic factors in $m, d, K, D, 1/\delta$.
\end{theorem}

Theorem \ref{gmm_main} shows the power of unsupervised pretraining under this setting in the following sense: Note that the number of parameters of a GMM is $dK$, therefore if we directly do classification without unsupervised pretraining, the risk will scale as $\tilde{\mathcal{O}}(\sqrt{dK/n})$. When $d$ is large and $m \gg n$, we achieve a better risk bound than supervised learning that only uses the labeled data.


\section{Pretraining via Contrastive Learning}\label{contrastive_learning}

For human beings, when given many pictures of different animals, we are able to infer which pictures show the same animals even if we do not have any prior knowledge about the animals. In this process, we inadvertently learn a representation for each picture that can be used to capture the similarity between different pictures. Contrastive learning mimics the way human learns. To be more specific, based on positive and negative pairs, contrastive learning learns to embed data into some space where similar sample pairs stay close to each other and dissimilar ones are far apart. In this section, we show how pretraining (learning the embedding function) can benefit the downstream linear regression tasks under our theoretical framework.

\paragraph{Model setup.}
In the setting of contrastive learning, we assume that $x$ and $x'$ are sampled independently from the same distribution $\P(x)$. The similarity between $x$ and $x'$ is captured by a representation function $f_{\theta^* }:\mathcal{X}\rightarrow \R^{r}$ in the following sense, 
\$
&\P(t=1\,|\,x,x')=\frac{1}{1+e^{-f_{\theta^* }(x)^Tf_{\theta^* }(x')}},\notag\\
&\P(t=-1\,|\,x,x')=\frac{1}{1+e^{f_{\theta^* }(x)^Tf_{\theta^* }(x')}}.
\$
Here $t$ is a random variable that labels the similarity between $x$ and $x'$. If the data pair $(x,x')$ is similar, then $t$ tends to be $1$. If the data pair $(x,x')$ is not similar (negative samples), then $t$ tends to be $-1$. We assume $(x,x',t)\sim\P_{f_{\theta^* }}(x,x',t)$. Here, $(x',t)$ can be viewed as side information. The latent variable $z$ is defined as $z:=f_{\theta^* }(x)+\mu$, where $\mu\sim\mN(0,I_r)$ is a Gaussian noise that is uncorrelated with $x$. We denote $(x,z)\sim\P_{f_{\theta^* }}(x,z)$.

For the downstream task, we consider the following linear regression problem
\$
y=\beta^{* T}z+\nu,
\$
where $\nu\sim\mN(0,1)$ is a Gaussian noise. We assume that the true parameters $\theta^* \in\Theta$ and $\beta^* \in\mathcal{B}$, which satisfy a standard normalization assumption, i.e., $\|f_{\theta}(x)\|_2\leq 1$ for any $\theta\in\Theta$ and $x\in\mathcal{X}$ and $\|\beta\|_2\leq D$ for any $\beta\in\mathcal{B}$. We have access to i.i.d unlabeled data $\{x_i,x'_i,t_i\}^{m}_{i=1}$ and i.i.d labeled data $\{x_j,y_j\}^n_{j=1}$. Here $(x'_i,t_i)$ is the side information corresponding to $x_i$.


In the sequel, we consider the squared loss function $\ell(x,y):=(y-x)^2$. We use the same form of truncated squared loss as in \eqref{eq:truncated_loss}.




\paragraph{Weakly informative condition.}
We first prove that the above model satisfies Assumption \ref{weak_invariance}:

\begin{lemma}\label{contrastive_ti}
Contrastive learning with linear regression as downstream tasks is $\kappa^{-1}$-weakly-informative, where
\$
\kappa=c_3\cdot\sqrt{\frac{1}{\sigma_{\min}(\E[f_{\theta^* }(x)f_{\theta^* }(x)^{T}])}}.
\$
Here $c_3$ is an absolute constant.
\end{lemma}



\paragraph{Theoretical results.}
We define a set of density functions $\mP_{\mathcal{X}\times\mathcal{S}}(\mathcal{F}_{\theta}):=\{p_{f_{\theta}}(x,x',t)\,|\,\theta\in\Theta\}$. We then have the following theoretical guarantee.

\begin{theorem}\label{contrastive_main}
We consider Algorithm \ref{mle+erm} with truncated squared loss \eqref{eq:truncated_loss}
where $L=36(D^2+1)\log n$. Let $\hat\theta, \hat\beta$ be the outputs of Algorithm \ref{mle+erm}. Then, for contrastive learning with linear regression as downstream tasks, with probability at least $1-\delta$, the excess risk can be bounded as follows,
\begin{align*}
{\rm Error}_{\ell}(\hat \theta,\hat\beta)\leq\Tilde{\mathcal{O}}\bigg(\kappa L\sqrt{\frac{\log N_{\b}\big(\mP_{\mathcal{X}\times\mathcal{S}}(\mathcal{F}_{\theta}),1/m^2\big)}{m}}+L\sqrt{\frac{1}{n}}\bigg),
\end{align*}
where $L=36(D^2+1)\log n$ and $\kappa$ is specified in Lemma \ref{contrastive_ti}.
Here $\tilde{\mathcal{O}}(\cdot)$ omits some constants and the polylogarithmic factors in $1/\delta$.
\end{theorem}

Note that the excess risk of directly training with labeled data strongly depends on the complexity of the function class $\mathcal{F}_{\theta}$. In the case that $m\gg n$, the excess risk of Theorem \ref{contrastive_main} scales as $\tilde{O}(\sqrt{1/n})$, which beats the pure supervised learning if the complexity of $\mathcal{F}_{\theta}$ is quite large. Thus, the utility of unsupervised pretraining is revealed for contrastive learning.

\section{Conclusions}
This paper proposes a generic theoretic framework for explaining the statistical benefits of unsupervised pretraining.
We study the natural scheme of using MLE for unsupervised pretraining and ERM for downstream task learning. We identify a natural ``informative'' condition, under which our algorithm achieves an excess risk bound that significantly improves over the baseline achieved by purely supervised learning in the typical practical regimes.
We further instantiate our theoretical framework with three concrete approaches for unsupervised pretraining and provide corresponding guarantees.

\newpage
\bibliography{reference}
\bibliographystyle{icml2023}

\newpage
\appendix
\section{Proofs for Section \ref{main_results}}

In Section \ref{main1}, we prove Theorem \ref{tv_mle}, which gives a TV distance guarantee for the MLE step in Algorithm \ref{mle+erm}. Our proof is inspired by \citet{van2000empirical,Zhang_2006}, and largely follows \citet{agarwal2020flambe,pmlr-v178-liu22f}. In Section \ref{main2}, we prove Theorem \ref{error_bound} that guarantees the performance of Algorithm \ref{mle+erm} by upper bounding the excess risk. The proof relies on the fact that the labeled data $\{x_j,y_j\}^n_{j=1}$ are independent of the unlabeled data $\{x_i,s_i\}^m_{i=1}$. In Section \ref{main3}, we prove Corollary \ref{rc_covering} based on the analysis of Gaussian complexity. In Section \ref{main4}, we prove Theorem \ref{weak_error_bound} by first showing that the MLE step in Algorithm \ref{mle+erm} actually guarantees an upper bound on the Hellinger distance, which is stronger than the TV distance guarantee mentioned in Theorem \ref{tv_mle}.

\subsection{Proofs for Theorem \ref{tv_mle}}\label{main1}

In the sequel, we prove Theorem \ref{tv_mle}.

\begin{proof}[Proof of Theorem \ref{tv_mle}]
For notation simplicity, we denote $\bm{x}:=(x,s)$.
Recall that we define $\mP_{\mathcal{X} \times \mathcal{S}}(\Phi):=\{p_{\phi}(x,s)\,|\,\phi\in\Phi\}$. Let $\mN_{\b}(\mP_{\mathcal{X} \times \mathcal{S}}(\Phi),\epsilon)$ be the smallest $\epsilon$-bracket of $\mP_{\mathcal{X} \times \mathcal{S}}(\Phi)$. We have $|\mN_{\b}(\mP_{\mathcal{X} \times \mathcal{S}}(\Phi),\epsilon)|=N_{\b}(\mP_{\mathcal{X} \times \mathcal{S}}(\Phi),\epsilon)$, where $N_{\b}(\mP_{\mathcal{X} \times \mathcal{S}}(\Phi),\epsilon)$ is the bracketing number of $\mP_{\mathcal{X} \times \mathcal{S}}(\Phi)$. By Markov inequality and Boole's inequality, it holds with probability at least $1-\delta$ that for all $\bar p_{\phi}(\bm{x})\in\mN_{\b}(\mP_{\mathcal{X} \times \mathcal{S}}(\Phi),\epsilon)$
\%
\frac{1}{2}\sum^m_{i=1}\log\frac{\bar p_{\phi}(\bm{x}_i)}{p_{\phi^* }(\bm{x}_i)}\leq\log\E\Big[e^{\frac{1}{2}\sum^m_{i=1}\log\frac{\bar p_{\phi}(\bm{x}_i)}{p_{\phi^* }(\bm{x}_i)}}\Big]+\log\frac{N_{\b}(\mP_{\mathcal{X} \times \mathcal{S}}(\Phi),\epsilon)}{\delta}.
\%
Note that $\hat\phi$ is the maximizer of the likelihood function, i.e.
\$
\hat\phi\leftarrow\argmax_{\phi\in\Phi}\sum^m_{i=1}\log p_{\phi}(\bm{x}_i),
\$
which implies 
\%
\frac{1}{2}\sum^m_{i=1}\log\frac{\bar p_{\hat\phi}(\bm{x}_i)}{p_{\phi^* }(\bm{x}_i)}\geq 0.
\%
Then we have with probability at least $1-\delta$ that
\%
0&\leq\log\E\Big[e^{\frac{1}{2}\sum^m_{i=1}\log\frac{\bar p_{\hat\phi}(\bm{x}_i)}{p_{\phi^* }(\bm{x}_i)}}\Big]+\log\frac{N_{\b}(\mP_{\mathcal{X} \times \mathcal{S}}(\Phi),\epsilon)}{\delta},\notag\\
&=m\log\E\bigg[\sqrt{\frac{\bar p_{\hat\phi}(\bm{x})}{p_{\phi^* }(\bm{x})}}\bigg]+\log\frac{N_{\b}(\mP_{\mathcal{X} \times \mathcal{S}}(\Phi),\epsilon)}{\delta},\notag\\
&=m\log\int \sqrt{\bar p_{\hat\phi}(\bm{x})p_{\phi^* }(\bm{x})} \,d\bm{x} +\log\frac{N_{\b}(\mP_{\mathcal{X} \times \mathcal{S}}(\Phi),\epsilon)}{\delta},\notag\\
&\leq m\bigg(\int \sqrt{\bar p_{\hat\phi}(\bm{x})p_{\phi^* }(\bm{x})} \,d\bm{x} -1\bigg)+\log\frac{N_{\b}(\mP_{\mathcal{X} \times \mathcal{S}}(\Phi),\epsilon)}{\delta},
\%
where the last inequality follows from the fact that $\log x\leq x-1$. By rearranging the terms, we have
\%\label{092101}
1-\int \sqrt{\bar p_{\hat\phi}(\bm{x})p_{\phi^* }(\bm{x})} \,d\bm{x} \leq \frac{1}{m}\log\frac{N_{\b}(\mP_{\mathcal{X} \times \mathcal{S}}(\Phi),\epsilon)}{\delta}.
\% 

By the definition of bracket, we obtain
\$
\int \bar p_{\hat\phi}(\bm{x})d\bm{x}=\int (\bar p_{\hat\phi}(\bm{x})-p_{\hat\phi}(\bm{x}))d\bm{x}+\int  p_{\hat\phi}(\bm{x})d\bm{x}\leq \epsilon+1,
\$
which implies 
\%\label{092102}
&\int \Big(\sqrt{\bar p_{\hat\phi}(\bm{x})}-\sqrt{p_{\phi^* }(\bm{x})}\Big)^2\, d\bm{x}\leq 2\bigg(1-\int\sqrt{\bar p_{\hat\phi}(\bm{x})p_{\phi^* }(\bm{x})}d\bm{x}\bigg)+\epsilon
\%
and 
\%\label{092103}
\int\Big(\sqrt{\bar p_{\hat\phi}(\bm{x})}+\sqrt{p_{\phi^* }(\bm{x})}\Big)^2\, d\bm{x}\leq 2\int\bar p_{\hat\phi}(\bm{x})+p_{\phi^* }(\bm{x})\, d\bm{x}\leq 2\epsilon+4.
\%
Combining \eqref{092101} and \eqref{092102}, we show that
\%\label{092104}
\int\Big(\sqrt{\bar p_{\hat\phi}(\bm{x})}-\sqrt{p_{\phi^* }(\bm{x})}\Big)^2\, d\bm{x}\leq \frac{2}{m}\log\frac{N_{\b}(\mP_{\mathcal{X} \times \mathcal{S}}(\Phi),\epsilon)}{\delta}+\epsilon.
\%
By Cauchy-Schwarz inequality, it then holds that
\%\label{092105}
\bigg(\int|\bar p_{\hat\phi}(\bm{x})-p_{\phi^* }(\bm{x})|\, d\bm{x}\bigg)^2&\leq \int\Big(\sqrt{\bar p_{\hat\phi}(\bm{x})}+\sqrt{p_{\phi^* }(\bm{x})}\Big)^2\, d\bm{x}\cdot\int\Big(\sqrt{\bar p_{\hat\phi}(\bm{x})}-\sqrt{p_{\phi^* }(\bm{x})}\Big)^2\, d\bm{x},\notag\\
&\leq (2\epsilon+4)\cdot\bigg(\frac{2}{m}\log\frac{N_{\b}(\mP_{\mathcal{X} \times \mathcal{S}}(\Phi),\epsilon)}{\delta}+\epsilon\bigg),
\%
where the last inequality follows from \eqref{092103} and \eqref{092104}. Note that
\%\label{092106}
&\bigg(\int| p_{\hat\phi}(\bm{x})-p_{\phi^* }(\bm{x})|\, d\bm{x}\bigg)^2-\bigg(\int|\bar p_{\hat\phi}(\bm{x})-p_{\phi^* }(\bm{x})|\, d\bm{x}\bigg)^2\notag\\
&=\bigg(\int | p_{\hat\phi}(\bm{x})-p_{\phi^* }(\bm{x})|+| \bar p_{\hat\phi}(\bm{x})-p_{\phi^* }(\bm{x})|\,d\bm{x}\bigg)\cdot\bigg(\int | p_{\hat\phi}(\bm{x})-p_{\phi^* }(\bm{x})|-| \bar p_{\hat\phi}(\bm{x})-p_{\phi^* }(\bm{x})|\,d\bm{x}\bigg)\notag\\
&\leq\bigg(\int | p_{\hat\phi}(\bm{x})-p_{\phi^* }(\bm{x})|+| \bar p_{\hat\phi}(\bm{x})-p_{\phi^* }(\bm{x})|\,d\bm{x}\bigg)\cdot\int | p_{\hat\phi}(\bm{x})- \bar p_{\hat\phi}(\bm{x})|\,d\bm{x}\notag\\
&\leq (\epsilon+4)\cdot\epsilon.
\%
Adding \eqref{092105} and \eqref{092106} together, we have
\%
\bigg(\int| p_{\hat\phi}(\bm{x})-p_{\phi^* }(\bm{x})|\, d\bm{x}\bigg)^2\leq (2\epsilon+4)\cdot\bigg(\frac{2}{m}\log\frac{N_{\b}(\mP_{\mathcal{X} \times \mathcal{S}}(\Phi),\epsilon)}{\delta}+\epsilon\bigg)+(\epsilon+4)\cdot\epsilon,
\%
which implies 
\%
\TV\big(\P_{\hat\phi}(\bm{x}),\P_{\phi^* }(\bm{x})\big)&=\frac{1}{2}\int| p_{\hat\phi}(\bm{x})-p_{\phi^* }(\bm{x})|\, d\bm{x}\notag\\
&\leq\frac{1}{2} \sqrt{(2\epsilon+4)\cdot\bigg(\frac{2}{m}\log\frac{N_{\b}(\mP_{\mathcal{X} \times \mathcal{S}}(\Phi),\epsilon)}{\delta}+\epsilon\bigg)+(\epsilon+4)\cdot\epsilon}.
\%
Setting $\epsilon=1/m$, we have with probability at least $1-\delta$ that
\%
\TV\big(\P_{\hat\phi}(\bm{x}),\P_{\phi^* }(\bm{x})\big)&\leq\frac{1}{2} \sqrt{\bigg(\frac{2}{m}+4\bigg)\cdot\bigg(\frac{2}{m}\log\frac{N_{\b}(\mP_{\mathcal{X} \times \mathcal{S}}(\Phi),1/m)}{\delta}+\frac{1}{m}\bigg)+\bigg(\frac{1}{m}+4\bigg)\cdot\frac{1}{m}}\notag\\
&\leq 3\cdot\sqrt{\frac{1}{m}\log\frac{N_{\b}(\mP_{\mathcal{X} \times \mathcal{S}}(\Phi),1/m)}{\delta}}.
\%
Thus, we prove Theorem \ref{tv_mle}.
\end{proof}

\subsection{Proofs for Theorem \ref{error_bound}}\label{main2}
Before proving the theorem, we first present some useful results that will be used in the proof of Theorem \ref{error_bound}. Lemma \ref{error_rademacher} upper bounds the difference between empirical loss and population loss by an application of bounded difference inequality and a standard symmetrization argument. Lemma \ref{error_dtv} relates excess risks with the total variation distance between probability distributions. For notation simplicity, we denote $\E_{(x,y)\sim\P_{\phi,\psi}(x,y)}$ by $\E_{\phi,\psi}$ in the following. We further denote by $\E$ the expectation taken over the  ground truth parameter, i.e., $\E:=\E_{(x,y)\sim\P_{\phi^* ,\psi^* }(x,y)}$.

\begin{lemma} \label{error_rademacher}
Suppose that $\ell(\cdot,\cdot)$ is a $L$-bounded loss function. For any given $\phi \in \Phi$, with probability at least $1-\delta$, 
\%
\sup_{\psi \in \Psi}\bigg|\mathbb{E} [\ell (g_{\phi, \psi}(x),y)] - \frac{1}{n}\sum_{j=1}^{n} \ell (g_{\phi, \psi}(x_{j}),y_{j}) \bigg| \leq R_{n} (\ell \circ \mathcal{G}_{\phi, \Psi}) + L\sqrt{\frac{2\log (2/\delta)}{n}},
\%
where $R_{n} (\ell \circ \mathcal{G}_{\phi, \Psi})$ is the Rademacher complexity of the function class $\ell \circ \mathcal{G}_{\phi, \Psi}$ defined in Theorem \ref{error_bound}.
\end{lemma}
\begin{proof}[Proof of Lemma \ref{error_rademacher}]
First notice that, 
when a pair $(x_{j},y_{j})$ changes, since $\ell$ is $L$-bounded, the random variable 
\%
\sup_{\psi \in \Psi} \bigg( \mathbb{E} [\ell (g_{\phi, \psi}(x),y)] - \frac{1}{n}\sum_{j=1}^{n} \ell (g_{\phi, \psi}(x_{j}),y_{j}) \bigg )
\%
can change by no more than $2L/n$. McDiarmid’s inequality implies that with probability at least $1-\delta/2$,
\begin{align}
& \sup_{\psi \in \Psi} \bigg ( \mathbb{E}[\ell (g_{\phi, \psi}(x),y)] - \frac{1}{n}\sum_{j=1}^{n} \ell (g_{\phi, \psi}(x_{j}),y_{j}) \bigg ) \notag \\
& \leq \mathbb{E} \bigg[\sup_{\psi \in \Psi} \bigg ( \mathbb{E}[\ell (g_{\phi, \psi}(x),y)] - \frac{1}{n}\sum_{j=1}^{n} \ell (g_{\phi, \psi}(x_{j}),y_{j})\bigg) \bigg]+ L \sqrt{\frac{2\log (2/\delta)}{n}}.
\end{align}
Let $\{x_{j}',y_{j}'\}_{j=1}^{n}$ be independent copies of $\{x_j,y_j\}^{n}_{j=1}$ and $\{\sigma_{j}\}_{j=1}^{n}$ be i.i.d. Rademacher random variables. Using the standard symmetrization technique, we have
\begin{align}
& \mathbb{E} \bigg[\sup_{\psi \in \Psi}\bigg( \mathbb{E}[\ell (g_{\phi, \psi}(x),y)] - \frac{1}{n}\sum_{j=1}^{n} \ell (g_{\phi, \psi}(x_{j}),y_{j}) \bigg) \bigg] \notag \\
& = \mathbb{E} \bigg[\sup_{\psi \in \Psi} \mathbb{E} \bigg[\frac{1}{n}\sum_{j=1}^{n} \ell (g_{\phi, \psi}(x_{j}'),y_{j}') - \frac{1}{n}\sum_{j=1}^{n} \ell (g_{\phi, \psi}(x_{j}),y_{j}) \bigg| \{x_{j},y_{j}\}_{j=1}^{n} \bigg] \bigg] \notag \\
& \leq \mathbb{E} \bigg[\sup_{\psi \in \Psi} \bigg( \frac{1}{n}\sum_{j=1}^{n} \ell (g_{\phi, \psi}(x_{j}'),y_{j}') - \frac{1}{n}\sum_{j=1}^{n} \ell (g_{\phi, \psi}(x_{j}),y_{j}) \bigg) \bigg] \notag \\
& \leq  \mathbb{E} \bigg[\sup_{\psi \in \Psi}  \frac{1}{n}\sum_{j=1}^{n} \sigma_{j}\bigg( \ell (g_{\phi, \psi}(x_{j}'),y_{j}') -  \ell (g_{\phi, \psi}(x_{j}),y_{j}) \bigg) \bigg] \notag \\
&\leq 2 \mathbb{E} \bigg[\sup_{\psi \in \Psi}  \frac{1}{n}\sum_{j=1}^{n} \sigma_{j} \ell (g_{\phi, \psi}(x_{j}),y_{j}) \bigg] \notag \\
& = R_{n} (\ell \circ \mathcal{G}_{\phi, \Psi}).
\end{align}
Therefore, with probability at least $1-\delta/2$, 
\%
\sup_{\psi \in \Psi}\bigg( \mathbb{E} [\ell (g_{\phi, \psi}(x),y)] - \frac{1}{n}\sum_{j=1}^{n} \ell (g_{\phi, \psi}(x_{j}),y_{j}) \bigg) \leq R_{n} (\ell \circ \mathcal{G}_{\phi, \Psi}) + L\sqrt{\frac{2\log (2/\delta)}{n}}
\%
Similarly, with probability at least $1-\delta/2$, 
\%
\sup_{\psi \in \Psi}\bigg( \frac{1}{n}\sum_{j=1}^{n} \ell (g_{\phi, \psi}(x_{j}),y_{j}) - \mathbb{E} [\ell (g_{\phi, \psi}(x),y)] \bigg) \leq R_{n} (\ell \circ \mathcal{G}_{\phi, \Psi}) + L\sqrt{\frac{2\log (2/\delta)}{n}}
\%
Combine these together,  we prove Lemma \ref{error_rademacher}. 
\end{proof}

\begin{lemma} \label{error_dtv}
Suppose that $\ell(\cdot,\cdot)$ is a $L$-bounded loss function. Then, it holds for any $\phi \in \Phi, \psi \in \Psi$ that
\begin{align} 
\mathbb{E}[\ell (g_{\phi, \psi}(x),y)] - \mathbb{E} [\ell (g_{\phi^{*},\psi^{*}}(x),y)] \leq 4L \cdot d_\mathrm{TV} (\P_{\phi,\psi}(x,y),\P_{\phi^{*},\psi^{*}}(x,y)).
\end{align}

\end{lemma}
\begin{proof}[Proof of Lemma \ref{error_dtv}]
\begin{align}\label{092701}
&\mathbb{E}_{\phi^{*}, \psi^{*}} [\ell (g_{\phi, \psi}(x),y)] - \mathbb{E}_{\phi^{*},\psi^{*}} [\ell (g_{\phi^{*},\psi^{*}}(x),y)] \notag \\ 
&=\mathbb{E}_{\phi^{*}, \psi^{*}} [\ell (g_{\phi, \psi}(x),y)] - \mathbb{E}_{\phi, \psi} [\ell (g_{\phi, \psi}(x),y)] \notag \\
&\quad+ \mathbb{E}_{\phi, \psi} [\ell (g_{\phi, \psi}(x),y)] - \mathbb{E}_{\phi, \psi} [\ell (g_{\phi^{*}, \psi^{*}}(x),y)] \notag \\
&\quad+ \mathbb{E}_{\phi, \psi} [\ell (g_{\phi^{*}, \psi^{*}}(x),y)] - \mathbb{E}_{\phi^{*},\psi^{*}} [\ell (g_{\phi^{*},\psi^{*}}(x),y)].
\end{align}
First notice that, by definition of $g_{\phi,\psi}$, 
\%\label{092702}
\mathbb{E}_{\phi, \psi} [\ell (g_{\phi, \psi}(x),y)] - \mathbb{E}_{\phi, \psi} [\ell (g_{\phi^{*}, \psi^{*}}(x),y)]  \leq 0.
\%
For the other two terms, based on the fact that $\ell$ is $L$-bounded, we have
\begin{align}\label{092703}
&|\mathbb{E}_{\phi^{*}, \psi^{*}} [\ell (g_{\phi, \psi}(x),y)] - \mathbb{E}_{\phi, \psi} [\ell (g_{\phi, \psi}(x),y)]| \notag \\ 
&= \left| \int \ell (g_{\phi, \psi}(x),y) p_{\phi_{*}, \psi_{*}} (x,y) \mathrm{d}x\mathrm{d}y - \int \ell (g_{\phi, \psi}(x),y) p_{\phi, \psi} (x,y) \mathrm{d}x\mathrm{d}y \right| \notag \\
&= \left| \int \ell (g_{\phi, \psi}(x),y) (p_{\phi_{*}, \psi_{*}} (x,y)-p_{\phi, \psi} (x,y)) \mathrm{d}x\mathrm{d}y \right| \notag \\
& \leq \int |\ell (g_{\phi, \psi}(x),y)| |(p_{\phi_{*}, \psi_{*}} (x,y)-p_{\phi, \psi} (x,y))| \mathrm{d}x\mathrm{d}y \notag \\
& \leq \int L |(p_{\phi_{*}, \psi_{*}} (x,y)-p_{\phi, \psi} (x,y))| \mathrm{d}x\mathrm{d}y \notag \\
&= 2L \cdot d_\mathrm{TV} (P_{\phi,\psi}(x,y),P_{\phi^{*},\psi^{*}}(x,y)).
\end{align}
Similarly, it holds that
\%\label{092704}
|\mathbb{E}_{\phi^{*}, \psi^{*}} [\ell (g_{\phi^{*}, \psi^{*}}(x),y)] - \mathbb{E}_{\phi, \psi} [\ell (g_{\phi^{*}, \psi^{*}}(x),y)]| \leq 2L \cdot d_\mathrm{TV} (P_{\phi,\psi}(x,y),P_{\phi^{*},\psi^{*}}(x,y)).
\%
Combining \eqref{092701}, \eqref{092702}, \eqref{092703} and \eqref{092704}, we obtain
\begin{align}
&\mathbb{E}_{\phi^{*}, \psi^{*}} [\ell (g_{\phi, \psi}(x),y)] - \mathbb{E}_{\phi^{*},\psi^{*}} [\ell (g_{\phi^{*},\psi^{*}}(x),y)] \leq 4L \cdot d_\mathrm{TV} (P_{\phi,\psi}(x,y),P_{\phi^{*},\psi^{*}}(x,y)).
\end{align}
\end{proof}

With Lemma \ref{error_rademacher} and Lemma \ref{error_dtv}, we are able to state our proofs for Theorem \ref{error_bound} in the following. The main idea of the proof is decomposing the risk. And a key observation is that the labeled data $\{x_j,y_j\}^{n}_{j=1}$ are independent of the pretrained $\hat\phi$, which is learned from the unlabeled data $\{x_i\}^m_{i=1}$.
\begin{proof}[Proof of Theorem  \ref{error_bound}] 
Let 
\begin{align}
\tilde{\psi} := \argmin_{\psi \in \Psi} d_\mathrm{TV} (P_{\hat{\phi},\psi}(x,y),P_{\phi^{*},\psi^{*}}(x,y)).
\end{align}
And for any $\phi \in \Phi, \psi \in \Psi$, we define
\%
\Delta_{\phi,\psi}:= \mathbb{E}_{\phi^{*}, \psi^{*}} [\ell (g_{\phi, \psi}(x),y)] - \frac{1}{n}\sum_{j=1}^{n} \ell (g_{\phi, \psi}(x_{j}),y_{j}).
\%
Recall that the excess risk is defined in \eqref{error}. It then holds that
\begin{align}\label{092705}
\operatorname{Error}_{\ell}(\hat{\phi}, \hat{\psi}) &= \mathbb{E}_{\phi^{*}, \psi^{*}} [\ell (g_{\hat{\phi}, \hat{\psi}}(x),y)] - \mathbb{E}_{\phi^{*},\psi^{*}} [\ell (g_{\phi^{*},\psi^{*}}(x),y)] \notag \\
&= \mathbb{E}_{\phi^{*}, \psi^{*}} [\ell (g_{\hat{\phi}, \hat{\psi}}(x),y)] - \frac{1}{n}\sum_{j=1}^{n} \ell (g_{\hat{\phi}, \hat{\psi}}(x_{j}),y_{j}) \notag \\
&\quad+ \frac{1}{n}\sum_{j=1}^{n} \ell (g_{\hat{\phi}, \hat{\psi}}(x_{j}),y_{j}) - \frac{1}{n}\sum_{j=1}^{n} \ell (g_{\hat{\phi}, \tilde{\psi}}(x_{j}),y_{j}) \quad (\leq 0, \text{ by ERM in Algorithm \ref{mle+erm}}) \notag \\
&\quad+ \frac{1}{n}\sum_{j=1}^{n} \ell (g_{\hat{\phi}, \tilde{\psi}}(x_{j}),y_{j}) - \mathbb{E}_{\phi^{*}, \psi^{*}} [\ell (g_{\hat{\phi}, \tilde{\psi}}(x),y)] \notag \\
&\quad+ \mathbb{E}_{\phi^{*}, \psi^{*}} [\ell (g_{\hat{\phi}, \tilde{\psi}}(x),y)] - \mathbb{E}_{\phi^{*},\psi^{*}} [\ell (g_{\phi^{*},\psi^{*}}(x),y)] \notag \\
& \leq\Delta_{\hat{\phi}, \hat{\psi}} - \Delta_{\hat{\phi}, \tilde{\psi}} + \mathbb{E}_{\phi^{*}, \psi^{*}} [\ell (g_{\hat{\phi}, \tilde{\psi}}(x),y)] - \mathbb{E}_{\phi^{*},\psi^{*}} [\ell (g_{\phi^{*},\psi^{*}}(x),y)].
\end{align}
By lemma \ref{error_dtv}, we have
\begin{align}\label{092706}
&\mathbb{E}_{\phi^{*}, \psi^{*}} [\ell (g_{\hat{\phi}, \tilde{\psi}}(x),y)] - \mathbb{E}_{\phi^{*},\psi^{*}} [\ell (g_{\phi^{*},\psi^{*}}(x),y)] \notag \\
&\leq 4L \cdot d_\mathrm{TV} (P_{\hat{\phi},\tilde{\psi}}(x,y),P_{\phi^{*},\psi^{*}}(x,y)) \notag \\
&= 4L \cdot \min_{\psi \in \Psi} d_\mathrm{TV} (P_{\hat{\phi},\psi}(x,y),P_{\phi^{*},\psi^{*}}(x,y)) \quad (\text{by definition of } \tilde{\psi}) \notag \\
& \leq 4\kappa L  \cdot d_\mathrm{TV} (P_{\hat{\phi}}(x,s),P_{\phi^{*}}(x,s)).
\end{align}
The last line holds, since by Assumption \ref{invariance}, for any $\hat{\phi} \in\Phi$, we choose $T_1$ that satisfies \eqref{informative} and $T_2$ that satisfies \eqref{101701}. Let $\psi=T^{-1}_2\circ\psi^* $. It then holds that
\begin{align}
 \min_{\psi \in \Psi} d_\mathrm{TV} (P_{\hat{\phi},\psi}(x,y),P_{\phi^{*},\psi^{*}}(x,y)) &\leq 
\TV\big(\P_{\hat{\phi},\psi}(x,y),\P_{\phi^* ,\psi^* }(x,y)\big) \notag \\
&=\TV\big(\P_{T_1\circ\hat{\phi},\psi^* }(x,y),\P_{\phi^* ,\psi^* }(x,y)\big) \notag \\
&\leq\TV\big(\P_{T_1\circ\hat{\phi}}(x,z),\P_{\phi^* }(x,z)\big) \notag \\
&\leq \kappa \cdot\TV\big(\P_{\hat{\phi}}(x,s),\P_{\phi^* }(x,s)\big).
\end{align}
Combining \eqref{092705} and \eqref{092706}, we have
\% \label{decomposition}
\operatorname{Error}_{\ell}(\hat{\phi}, \hat{\psi}) \leq \Delta_{\hat{\phi}, \hat{\psi}} - \Delta_{\hat{\phi}, \tilde{\psi}} + 4\kappa L  \cdot d_\mathrm{TV} (P_{\hat{\phi}}(x,s),P_{\phi^{*}}(x,s)).
\%

We define the following events
\%
D:= \left\{ d_\mathrm{TV} (P_{\hat{\phi}}(x,s),P_{\phi^{*}}(x,s)) \leq 3 \sqrt{\frac{1}{m} \log \frac{2N(\mathcal{P}_{\mathcal{X} \times \mathcal{S}}(\Phi), 1 / m)}{\delta}} \right\}
\%
and 
\%
R:= \left\{ \sup_{\psi \in \Psi}|\Delta_{\hat{\phi},\psi}| \leq R_{n} (\ell \circ \mathcal{G}_{\hat{\phi}, \Psi}) + L\sqrt{\frac{2\log (4/\delta)}{n}}  \right\}.
\%
It holds that
\begin{align}\label{092710}
\P(D \cap R) & = \mathbb{E} [\mathds{1}_{D \cap R} ]= \mathbb{E} [ \mathbb{E}[ \mathds{1}_{D} \mathds{1}_{R} | \hat{\phi} ] ] = \mathbb{E} [ \mathds{1}_{D} \mathbb{E}[ \mathds{1}_{R} | \hat{\phi} ] ] = \mathbb{E} [\mathds{1}_{D} \mathbb{P}(R | \hat{\phi} ) ],
\end{align}
where the third equation follows from the fact that $D$ is $\hat{\phi}$-measurable. Note that $\{x_{j},y_{j}\}_{j=1}^{n}$ is independent of $\hat{\phi}$. By Lemma \ref{error_rademacher}, for any given $\hat{\phi}$, with probability at least $1-\delta/2$, 
\%
\sup_{\psi \in \Psi}|\Delta_{\hat{\phi},\psi}| \leq R_{n} (\ell \circ \mathcal{G}_{\hat{\phi}, \Psi}) + L\sqrt{\frac{2\log (4/\delta)}{n}},
\%
i.e.,
\%\label{092708}
\P(R|\hat{\phi}) \geq 1-\delta/2.
\%
By Lemma \ref{tv_mle}, with probability at least $1-\delta/2$, the output of the first step of our algorithm $\hat{\phi}$, satisfies
\%
d_\mathrm{TV} (P_{\hat{\phi}}(x,s),P_{\phi^{*}}(x,s)) \leq 3\sqrt{\frac{1}{m} \log \frac{2N(\mathcal{P}_{\mathcal{X} \times \mathcal{S}}(\Phi), 1 / m)}{\delta}}
\%
i.e.,
\%\label{092709}
\P(D) \geq 1-\delta/2.
\%
By \eqref{092710}, \eqref{092708} and \eqref{092709}, we have
\begin{align}
\P(D \cap R) \geq (1-\delta/2)^{2} \geq 1-\delta.
\end{align}
Then, under event $D\cap R$, by our decomposition (\ref{decomposition}), we have
\begin{align}
\operatorname{Error}_{\ell}(\hat{\phi}, \hat{\psi}) &\leq \Delta_{\hat{\phi}, \hat{\psi}} - \Delta_{\hat{\phi}, \tilde{\psi}} + 4\kappa L  \cdot d_\mathrm{TV} (P_{\hat{\phi}}(x,s),P_{\phi^{*}}(x,s)) \notag \\
&\leq 2 \sup_{\psi \in \Psi}|\Delta_{\hat{\phi},\psi}|+ 4\kappa L  \cdot d_\mathrm{TV} (P_{\hat{\phi}}(x,s),P_{\phi^{*}}(x,s)) \notag \\
&\leq 2 R_{n} (\ell \circ \mathcal{G}_{\hat{\phi}, \Psi}) + 2L\sqrt{\frac{2\log (4/\delta)}{n}} + 12\kappa  L \sqrt{\frac{1}{m} \log \frac{2N(\mathcal{P}_{\mathcal{X} \times \mathcal{S}}(\Phi), 1 / m)}{\delta}} \notag \\
&\leq 2\max_{\phi \in \Phi} R_{n} (\ell \circ \mathcal{G}_{\phi, \Psi}) + 2L\sqrt{\frac{2\log (4/\delta)}{n}} + 12\kappa L \sqrt{\frac{1}{m} \log \frac{2N(\mathcal{P}_{\mathcal{X} \times \mathcal{S}}(\Phi), 1 / m)}{\delta}}.
\end{align}
Thus, we prove Theorem \ref{error_bound}.

\end{proof}

\subsection{Proofs for Corollary \ref{rc_covering}}\label{main3}
In the following, we give the proof of Corollary \ref{rc_covering}, which is based on the analysis of Gaussian complexity.
\begin{proof}
By Theorem \ref{error_bound}, we have
\begin{align}\label{092713}
{\rm Error}_{\ell}(\hat\phi,\hat\psi)\leq 2\max_{\phi\in\Phi} R_n(\ell\circ \mathcal{G}_{\phi,\Psi})+2L\cdot\sqrt{\frac{2}{n}\log\frac{4}{\delta}}+12\kappa L\cdot\sqrt{\frac{1}{m}\log\frac{2N_{\b}(\mP_{\mathcal{X} \times \mathcal{S}}(\Phi),1/m)}{\delta}}.   
\end{align}
Therefore, it remains to bound the Rademacher complexity term. By \cite{ledoux1991probability}, the Rademacher complexity is upper bounded by the Gaussian complexity, i.e.,
\begin{align}\label{092712}
R_{n}(\mathcal{F}) \leq c \cdot G_{n}(\mathcal{F}) = c\cdot\mathbb{E}\hat{G_{n}}(\mathcal{F}),
\end{align}
where $c$ is some absolute constants. Here $G_{n}(\mathcal{F})$ is the Gaussian complexity,  and it's empirical version is defined as
\begin{align}
\hat{G_{n}}(\mathcal{F}):=\mathbb{E}_{g_i} \bigg[\sup_{f \in \mathcal{F}}\bigg|\frac{2}{n}\sum_{i=1}^{n}g_{i}f(x_{i})\bigg|\,\bigg|\,x_{1}, \cdots, x_{n}\bigg]  
\end{align}
where $g_{1}, \cdots, g_{n}$ are i.i.d. $\mathcal{N}(0,1)$ random variables.
By (5.36) in \cite{wainwright2019high}, we have
\begin{align}\label{092711}
\hat{G_{n}}( {\ell \circ \mathcal{G}_{\phi, \Psi}})
& \leq\frac{1}{\sqrt{n}}\cdot \min_{\delta \in [0,L]}\bigg\{\delta \sqrt{n} + 2L \sqrt{\log N(\ell \circ \mathcal{G}_{\phi, \Psi},\delta, \|\cdot\|_{\infty})}\bigg\}   \notag \\
& \leq\frac{1}{\sqrt{n}}\bigg( L+ 2L \sqrt{\log N(\ell \circ \mathcal{G}_{\phi, \Psi},L/\sqrt{n}, \|\cdot\|_{\infty})}\bigg) \quad \text{(Take }\delta=L/\sqrt{n})\notag\\
&\leq 3L \sqrt{\frac{\log N(\ell \circ \mathcal{G}_{\phi, \Psi},L/\sqrt{n}, \|\cdot\|_{\infty})}{n}}.
\end{align}
Combining \eqref{092712} and \eqref{092711}, we obtain
\%\label{092714}
R_{n}(\ell \circ \mathcal{G}_{\phi, \Psi})\leq 3cL \sqrt{\frac{\log N(\ell \circ \mathcal{G}_{\phi, \Psi},L/\sqrt{n}, \|\cdot\|_{\infty})}{n}}.
\%

By \eqref{092713} and \eqref{092714}, we finish the proof.
\end{proof}

\subsection{Proofs for Theorem \ref{weak_error_bound}}\label{main4}
In this section, we first show the relation of Assumption \ref{invariance} and Assumption \ref{weak_invariance}. We then show that the MLE step in line 2 of Algorithm \ref{mle+erm} guarantees an upper bound on the Hellinger distance $H(\P_{\hat\phi}(x,s),\P_{\phi^* }(x,s))$. Then, using the same techniques as that in the proof of Theorem \ref{error_bound}, we prove Theorem \ref{weak_error_bound}.

\subsubsection{Relation of Assumption \ref{invariance} and Assumption \ref{weak_invariance}}\label{relation}

Assumption \ref{weak_invariance} is actually a relaxation of Assumption \ref{invariance}. To see this, by Assumption \ref{invariance}, for any $\phi\in\Phi$, we choose $T_1$ that satisfies \eqref{informative} and $T_2$ that satisfies \eqref{101701}. Let $\psi=T^{-1}_2\circ\psi^* $. It then holds that
\$
&\TV\big(\P_{\phi,\psi}(x,y),\P_{\phi^* ,\psi^* }(x,y)\big)\notag\\
&=\TV\big(\P_{T_1\circ\phi,\psi^* }(x,y),\P_{\phi^* ,\psi^* }(x,y)\big)\\
&\leq\TV\big(\P_{T_1\circ\phi}(x,z),\P_{\phi^* }(x,z)\big)\notag\\
&\leq \kappa \cdot\TV\big(\P_{\phi}(x,s),\P_{\phi^* }(x,s)\big).
\$
Note that the TV distance can be upper bounded by the Hellinger distance. Thus, Assumption \ref{invariance} directly implies Assumption \ref{weak_invariance}.

\subsubsection{Hellinger Distance Guarantee}
Suppose that $\hat\phi$ is the output of the MLE step in Algorithm \ref{mle+erm}, which satisfies 
\%
\hat\phi\leftarrow \argmax_{\phi\in\Phi}\sum^{m}_{i=1}\log p_{\phi}(x_i,s_i).
\%
We have the following theoretical guarantee on the Hellinger distance between $\P_{\hat\phi}(x,s)$ and $\P_{\phi^* }(x,s)$.
\begin{lemma}\label{mle_hellinger}
Let $\hat\phi$ be the output of Algorithm \ref{mle+erm}. It then holds that with probability at least $1-\delta$ that
\%
H\big(\P_{\hat\phi}(x,s),\P_{\phi^* }(x,s)\big)\leq \sqrt{\frac{2}{m}\log\frac{N_{\b}\big(\mP_{\mathcal{X}\times\mathcal{S}}(\Phi),1/m^2\big)}{\delta}},
\%
where we denote $\mP_{\mathcal{X}\times\mathcal{S}}(\Phi):=\{p_{\phi}(x,s)\,|\,\phi\in\Phi\}$.
\end{lemma}
\begin{proof}[Proof of Lemma \ref{mle_hellinger}]
For notation simplicity, we denote $\bm{x}:=(x,s)$. Let $\epsilon>0$. Similar to the proof of Theorem \ref{tv_mle}, we obtain with probability at least $1-\delta$
\%\label{101508}
1-\int\sqrt{\bar{p}_{\hat\phi}(\bm{x})p_{\phi^* }(\bm{x})}\,d\bm{x}\leq \frac{1}{m}\log\frac{N_{\b}\big(\mP_{\mathcal{X}\times\mathcal{S}}(\Phi),\epsilon\big)}{\delta}.
\%
Here $\bar{p}_{\hat\phi}(\bm{x})\in\mN_{\b}(\mP_{\mathcal{X}\times\mathcal{S}}(\Phi),\epsilon)$ that satisfies $\bar{p}_{\hat\phi}(\bm{x})\geq p_{\phi^* }(\bm{x})$ for any $\bm{x}$ and 
\%
\int \bar{p}_{\hat\phi}(\bm{x})-p_{\phi^* }(\bm{x})\,d\bm{x}\leq \epsilon.
\%
Note that
\%\label{101509}
&1-\int\sqrt{{p}_{\hat\phi}(\bm{x})p_{\phi^* }(\bm{x})}\,d\bm{x}-\bigg(1-\int\sqrt{\bar{p}_{\hat\phi}(\bm{x})p_{\phi^* }(\bm{x})}\,d\bm{x}\bigg)\notag\\
&=\int\Big(\sqrt{\bar{p}_{\hat\phi}(\bm{x})}-\sqrt{{p}_{\hat\phi}(\bm{x})}\Big)\sqrt{{p}_{\phi^* }(\bm{x})}\,d\bm{x}\notag\\
&\leq \sqrt{\int\Big(\sqrt{\bar{p}_{\hat\phi}(\bm{x})}-\sqrt{{p}_{\hat\phi}(\bm{x})}\Big)^2\,d\bm{x}}\notag\\
&=\sqrt{\int \bar{p}_{\hat\phi}(\bm{x})+{p}_{\hat\phi}(\bm{x})-2\sqrt{\bar{p}_{\hat\phi}(\bm{x}){p}_{\hat\phi}(\bm{x})} \,d\bm{x}}\notag\\
&\leq \sqrt{\int \bar{p}_{\hat\phi}(\bm{x})-{p}_{\hat\phi}(\bm{x})\,d\bm{x}}\notag\\
&\leq \sqrt{\epsilon}.
\%
Here the first inequality follows from Cauchy-Schwarz inequality and the second follows from the fact that $\sqrt{\bar{p}_{\hat\phi}(\bm{x}){p}_{\hat\phi}(\bm{x})}\geq{p}_{\hat\phi}(\bm{x})$. By \eqref{101508} and \eqref{101509}, we have
\%
1-\int\sqrt{{p}_{\hat\phi}(\bm{x})p_{\phi^* }(\bm{x})}\,d\bm{x}\leq \sqrt{\epsilon}+\frac{1}{m}\log\frac{N_{\b}\big(\mP_{\mathcal{X}\times\mathcal{S}}(\Phi),\epsilon\big)}{\delta},
\%
which implies that
\%
H^2\big(\P_{\hat{\phi}}(\bm{x}),\P_{\phi^{*}}(\bm{x})\big)=1-\int\sqrt{{p}_{\hat\phi}(\bm{x})p_{\phi^* }(\bm{x})}\,d\bm{x}\leq \sqrt{\epsilon}+\frac{1}{m}\log\frac{N_{\b}\big(\mP_{\mathcal{X}\times\mathcal{S}}(\Phi),\epsilon\big)}{\delta}.
\%
Set $\epsilon=1/m^2$. We have
\%
H^2\big({\P}_{\hat\phi}(x,s),\P_{\phi^* }(x,s)\big)\leq \frac{2}{m}\log\frac{N_{\b}\big(\mP_{\mathcal{X}\times\mathcal{S}}(\Phi),1/m^2\big)}{\delta}.
\%
\end{proof}

\subsubsection{Proof of Theorem \ref{weak_error_bound}}
With Lemma \ref{mle_hellinger} in hand, we are ready to prove Theorem \ref{weak_error_bound}.
\begin{proof}[Proof of Theorem \ref{weak_error_bound}]
Let $\hat\phi$ be the output of the MLE step in Algorithm \ref{mle+erm}. And for any $\phi \in \Phi, \psi \in \Psi$, we define
\%
\Delta_{\phi,\psi}:= \mathbb{E}_{\phi^{*}, \psi^{*}} [\ell (g_{\phi, \psi}(x),y)] - \frac{1}{n}\sum_{j=1}^{n} \ell (g_{\phi, \psi}(x_{j}),y_{j}).
\%
Following the same arguments as that in the proof of Theorem \ref{error_bound}, we have with probability at least $1-\delta$,
\%\label{111501}
H\big(\P_{\hat{\phi}}(x,s),\P_{\phi^{*}}(x,s)\big) \leq  \sqrt{\frac{2}{m} \log \frac{2N(\mathcal{P}_{\mathcal{X} \times \mathcal{S}}(\Phi), 1 / m^2)}{\delta}}
\%
and 
\%\label{111502}
\sup_{\psi \in \Psi}|\Delta_{\hat{\phi},\psi}| \leq R_{n} (\ell \circ \mathcal{G}_{\hat{\phi}, \Psi}) + L\sqrt{\frac{2\log (4/\delta)}{n}}.
\%
Moreover, as mentioned in \eqref{092705}, we have
\begin{align}\label{111504}
\operatorname{Error}_{\ell}(\hat{\phi},\hat\psi) &\leq\Delta_{\hat{\phi}, \hat{\psi}} - \Delta_{\hat{\phi}, \tilde{\psi}} + \mathbb{E}_{\phi^{*}, \psi^{*}} [\ell (g_{\hat{\phi}, \tilde{\psi}}(x),y)] - \mathbb{E}_{\phi^{*},\psi^{*}} [\ell (g_{\phi^{*},\psi^{*}}(x),y)]\notag\\
&\leq 2R_{n} (\ell \circ \mathcal{G}_{\hat{\phi}, \Psi}) + 2L\sqrt{\frac{2\log (4/\delta)}{n}}\notag\\
&\quad+\mathbb{E}_{\phi^{*},\psi^{*}} [\ell (g_{\phi^{*},\psi^{*}}(x),y)]-\mathbb{E}_{\phi^{*},\psi^{*}} [\ell (g_{\phi^{*},\psi^{*}}(x),y)],
\end{align}
where $\tilde{\psi} := \argmin_{\psi \in \Psi} d_\mathrm{TV} (\P_{\hat{\phi},\psi}(x,y),\P_{\phi^{*},\psi^{*}}(x,y))$ and the second inequality follows from \eqref{111502}. By lemma \ref{error_dtv}, we have
\begin{align}\label{111503}
&\mathbb{E}_{\phi^{*}, \psi^{*}} [\ell (g_{\hat{\phi}, \tilde{\psi}}(x),y)] - \mathbb{E}_{\phi^{*},\psi^{*}} [\ell (g_{\phi^{*},\psi^{*}}(x),y)] \notag \\
&\leq 4L \cdot d_\mathrm{TV} (\P_{\hat{\phi},\tilde{\psi}}(x,y),\P_{\phi^{*},\psi^{*}}(x,y)) \notag \\
&= 4L \cdot \min_{\psi \in \Psi} d_\mathrm{TV} (\P_{\hat{\phi},\psi}(x,y),\P_{\phi^{*},\psi^{*}}(x,y)) \quad (\text{by definition of } \tilde{\psi}) \notag \\
& \leq_{1)} 4\kappa L  \cdot H (\P_{\hat{\phi}}(x,s),\P_{\phi^{*}}(x,s))\notag\\
&\leq_{2)} 4\kappa L\sqrt{\frac{2}{m} \log \frac{2N(\mathcal{P}_{\mathcal{X} \times \mathcal{S}}(\Phi), 1 / m^2)}{\delta}},
\end{align}
where 1) follows from Assumption \ref{weak_invariance} and 2) follows from \eqref{111501}.
Combining \eqref{111504} and \eqref{111503}, we have
\%
\operatorname{Error}_{\ell}(\hat{\phi},\hat\psi)&\leq 2R_{n} (\ell \circ \mathcal{G}_{\hat{\phi}, \Psi}) + 2L\sqrt{\frac{2\log (4/\delta)}{n}}+4\kappa L\sqrt{\frac{2}{m} \log \frac{2N(\mathcal{P}_{\mathcal{X} \times \mathcal{S}}(\Phi), 1 / m^2)}{\delta}}\notag\\
&\leq 2\max_{\phi\in\Phi}R_{n} (\ell \circ \mathcal{G}_{{\phi}, \Psi}) + 2L\sqrt{\frac{2\log (4/\delta)}{n}}+4\kappa L\sqrt{\frac{2}{m} \log \frac{2N(\mathcal{P}_{\mathcal{X} \times \mathcal{S}}(\Phi), 1 / m^2)}{\delta}}.
\%
\end{proof}
\section{Proofs for Section \ref{factor_model}}

In Section \ref{factor1}, by analysing the total variation distance between two high-dimensional Gaussians and applying the Davis-Kahan theorem, we show that factor model with linear regression as downstream tasks has $\kappa$-transferability (Lemma \ref{factor_ti}), where $\kappa$ depends on the largest and smallest singular value of the ground truth parameter $B^* $. In Section \ref{factor2} and Section \ref{factor3}, we prove two lemmas that will be used in the proof of Theorem \ref{factor_main}. To be specific, in Section \ref{factor2}, we upper bound the bracketing number of the set $\mP(\mathcal{B})$ by using $\epsilon$-discretization (Lemma \ref{factor_bn}). In Section \ref{factor3}, we prove Lemma \ref{factor_rc}, which will be used to upper bound the Rademacher complexity of the function class $\ell\circ\mathcal{G}_{B,\mathcal{C}}$. In Section \ref{factor4}, we prove Theorem \ref{factor_main}.
Finally, in Section \ref{fast_rate}, we provide a refined analysis for proving Theorem \ref{factor_fast_rate}.

\subsection{Proofs for Lemma \ref{factor_ti}}\label{factor1}
First of all, we present some useful lemmas that will be used in the proof of Lemma \ref{factor_ti}. Given two high-dimensional Gaussians, we can bound their  total variation distance as follows.
\begin{lemma}[Theorem 1.2 and Proposition 2.1 in \cite{devroye2018total}]\label{tv_norm}
Suppose that $d>1$. Let $\mu_1\neq\mu_2\in\R^d$. Then, we have
\$
\frac{1}{200}\leq \frac{\TV\big(\mN(\mu_1,I_d),\mN(\mu_2,I_d)\big)}{\min\{1,\|\mu_1-\mu_2\|_2\}}\leq 1.
\$
\end{lemma}

\begin{lemma}[Theorem 1.1 in \cite{devroye2018total}]\label{tv_norm_2}
Suppose that $d>1$. Let $\mu\in\R^d$ and $\Sigma_1\neq\Sigma_2$ be positive definite $d\times d$ matrices. Then, we have
\$
\frac{1}{100}\leq \frac{\TV\big(\mN(\mu,\Sigma_1),\mN(\mu,\Sigma_2)\big)}{\min\{1,\|\Sigma^{-1/2}_1\Sigma_2\Sigma^{-1/2}_1-I_d\|_{\rF}\}}\leq \frac{3}{2}.
\$
\end{lemma}

Recall that we define $\mathcal{B}:=\{B\in\R^{d\times r}\,|\,\|B\|_2\leq D\}$. Let $B\in\mathcal{B}$ and $B^* $ be the ground truth parameter. We denote by $\sigma^* _{\max}$ and $\sigma^* _{\min}$ the largest and smallest singular value of $B^* $, respectively. Moreover, we denote the singular value decomposition of $B$ and $B^* $ by $B=U\Sigma V$ and $B^* =U^* \Sigma^*  V^* $, respectively. Here $\Sigma, \Sigma^* \in\R^{r\times r}$ are diagonal matrices and $U,U^* \in\R^{d\times r}$, $V,V^* \in\R^{r\times d}$ are matrices with orthogonal columns. Let 
\%\label{092401}
M:=BB^T=U\Lambda U^T,\quad M^* :=B^* B^{* T}=U^* \Lambda^*  U^{* T}, 
\%
where $\Lambda:=\Sigma\Sigma^T$ and $\Lambda^{*}:=\Sigma^* \Sigma^{* T}$. We define
\%\label{092402}
O:=\argmin_{O\in\mathcal{O}^{r\times r}}\|UO-U^* \|_{\rF}.
\%
Then, we have the following lemmas.
\begin{lemma}\label{D-K}
For $M,M^* $ defined in \eqref{092401} and $O$ defined in \eqref{092402}, there exists some absolute constants $c>1$ such that
\$
\|UO-U^* \|_{\rF}\leq \frac{c}{(\sigma^* _{\min})^2}\|M-M^* \|_{\rF}.
\$
Here $\sigma^* _{\min}$ is the smallest singular value of the true parameter $B^* $.
\end{lemma}
\begin{proof}
An application of Davis-Kahan Theorem \citep{davis1970rotation}.
\end{proof}
\begin{lemma}\label{yuxin}
For $M,M^* $ defined in \eqref{092401} and $O$ defined in \eqref{092402}, there exists some absolute constants $c$ such that
\$
\|\Lambda^{1/2}O-O\Lambda^{* 1/2}\|_{\rF}\leq \frac{4c(\sigma^* _{\max})^2}{(\sigma^* _{\min})^3}\|M-M^* \|_{\rF}.
\$
Here $\sigma^* _{\min}$ is the smallest singular value of the true parameter $B^* $.
\end{lemma}
\begin{proof}[Proof of Lemma \ref{yuxin}]
Our proof is inspired by \cite{ma2018implicit}. By Lemma 2.1 in \cite{schmitt1992perturbation}, we have
\%\label{092403}
\|\Lambda^{1/2}O-O\Lambda^{* 1/2}\|_{\rF}\leq \frac{1}{\sqrt{\sigma_{\min}(M^* )}}\|O^{T}\Lambda O-\Lambda^* \|_{\rF}=\frac{1}{\sigma^* _{\min}}\|O^{T}\Lambda O-\Lambda^* \|_{\rF}.
\%
Note that $\Lambda=U^TMU$ and $\Lambda^* =U^{* T}M^* U^* $. Thus, we have
\%\label{092404}
\|O^{T}\Lambda O-\Lambda^* \|_{\rF}&=\|O^TU^TMUO-U^{* T}M^* U^* \|_{\rF}\notag\\
&\leq \|O^TU^TMUO-O^TU^TM^* UO\|_{\rF}+\|O^TU^TM^* UO-U^{* T}M^* UO\|_{\rF}\notag\\
&\quad+\|U^{* T}M^* UO-U^{* T}M^* U^* \|_{\rF}\notag\\
&\leq \|M-M^* \|_{\rF}+2\|M^* \|_2\|UO-U^* \|_{\rF}\notag\\
&\leq \|M-M^* \|_{\rF}+2c\bigg(\frac{\sigma^* _{\max}}{\sigma^* _{\min}}\bigg)^2\|M-M^* \|_{\rF}\notag\\
&\leq 4c\bigg(\frac{\sigma^* _{\max}}{\sigma^* _{\min}}\bigg)^2\|M-M^* \|_{\rF},
\%
where the third inequality follows from Lemma \ref{D-K}. Combing \eqref{092403} and \eqref{092404}, we have
\$
\|\Lambda^{1/2}O-O\Lambda^{* 1/2}\|_{\rF}\leq \frac{4c(\sigma^* _{\max})^2}{(\sigma^* _{\min})^3}\|M-M^* \|_{\rF}.
\$

\end{proof}

Now we are ready to prove Lemma \ref{factor_ti}.

\begin{proof}[Proof of Lemma \ref{factor_ti}]
Let $\mathcal{O}^{r\times r}:=\{O\in\R^{r\times r}\,|\, OO^T=O^TO=I_r\}$. First of all, we show that for any $(B,\beta,O)\in\mathcal{B}\times\mathcal{C}\times\mathcal{O}$, it holds that $\P_{B,\beta}(x,y)=\P_{BO,O^T\beta}(x,y)$. This can be easily seen by the following observation,
\$
\P_{BO,O^T\beta}\sim\mN\bigg(0,
  \begin{bmatrix}
    BO(BO)^{T} & BOO^{T}\beta\\
    \beta^TOO^TB^T &  (O^{T}\beta)^{T}O^{T}\beta
  \end{bmatrix}
\bigg)=\mN\bigg(0,
  \begin{bmatrix}
    BB^T & B\beta\\
    \beta^TB^T &  \beta^{T}\beta
  \end{bmatrix}
  \bigg)\sim\P_{B,\beta}.
\$

By Lemma \ref{D-K}, it holds for some constant $c>1$ that
\%\label{092302}
\|UO-U^* \|_{\rF}\leq \frac{c}{(\sigma^* _{\min})^2}\|BB^T-B^*  B^{* T}\|_{\rF}.
\%
By Lemma \ref{yuxin}, it holds for some constant $c>1$ that
\%\label{092301}
\|\Sigma O-O\Sigma^* \|_{\rF}\leq \frac{4c(\sigma^* _{\max})^2}{(\sigma^* _{\min})^3}\|BB^T-B^*  B^{* T}\|_{\rF}.
\%
Let $\hat O:=V^{-1}OV^* \in\mathcal{O}^{r\times r}$. By \eqref{092302} and \eqref{092301}, we have
\%\label{092303}
\|B\hat O-B^* \|_{\rF}&=\|U\Sigma OV^* -U^* \Sigma^* V^* \|_{\rF}\notag\\
&\leq \|U\Sigma O-U^* \Sigma^* \|_{\rF}\notag\\
&\leq \|U\Sigma O-UO\Sigma^* \|_{\rF}+\|UO\Sigma^* -U^* \Sigma^* \|_{\rF}\notag\\
&\leq \|\Sigma O-O\Sigma^* \|_{\rF}+\|UO-U^* \|_{\rF}\|\Sigma^* \|_2\notag\\
&\leq c\cdot \bigg(\frac{4(\sigma^* _{\max})^2}{(\sigma^* _{\min})^3}+\frac{\sigma^* _{\max}}{(\sigma^* _{\min})^2}\bigg)\cdot\|BB^T-B^*  B^{* T}\|_{\rF}\notag\\
&\leq \frac{5c(\sigma^* _{\max})^2}{(\sigma^* _{\min})^3}\cdot\|BB^T-B^*  B^{* T}\|_{\rF}.
\%
Note that
\%\label{092304}
\TV\big(\P_{B\hat O}(x,z),\P_{B^* }(x,z)\big)&=\int|p_{B\hat O}(x\,|\,z)-p_{B^* }(x\,|\,z)|p(z)\,dxdz\notag\\
&=\int\TV\big(\mN(B\hat O z,I_d),\mN(B^*  z,I_d)\big)p(z)\,dz\notag\\
&\leq \int \min\{1, \|B\hat O z-B^* z\|_2\}p(z)\,dz\notag\\
&\leq \min\big\{1, \E[\|B\hat O z-B^* z\|_2]\big\},
\%
where the first inequality follows from Lemma \ref{tv_norm}. We can show that
\%\label{092305}
\E[\|B\hat O z-B^* z\|_2]&\leq \Big(\E\big[\|B\hat O z-B^* z\|^2_2\big]\Big)^{1/2}\notag\\
&=\Big(\E\big[z^T(B\hat O-B^* )^{T}(B\hat O-B^* )z\big]\Big)^{1/2}\notag\\
&=\Big(\E\big[{\rm Tr}\big((B\hat O-B^* )^{T}(B\hat O-B^* )zz^T\big)\big]\Big)^{1/2}\notag\\
&=\Big({\rm Tr}\big((B\hat O-B^* )^{T}(B\hat O-B^* )\big)\Big)^{1/2}\notag\\
&=\|B\hat O-B^* \|_{\rF}.
\%
By \eqref{092303}, \eqref{092304} and \eqref{092305}, it holds that
\%\label{092309}
&\TV\big(\P_{B\hat O}(x,z),\P_{B^* }(x,z)\big)\notag\\
&\quad\leq \min\big\{1,\|B\hat O-B^* \|_{\rF}\big\}\notag\\
&\quad\leq \min\bigg\{1, \frac{5c(\sigma^* _{\max})^2}{(\sigma^* _{\min})^3}\cdot\|BB^T-B^*  B^{* T}\|_{\rF}\bigg\}\notag\\
&\quad\leq \frac{5c(\sigma^* _{\max})^2}{(\sigma^* _{\min})^3}\cdot \big((\sigma^* _{\max})^2+1\big)\cdot\min\bigg\{1,\frac{\|BB^T-B^* B^{* T}\|_{\rF}}{(\sigma^* _{\max})^2+1}\bigg\},
\%
where the last inequality follows from $c>1$ and
\$
\frac{(\sigma^* _{\max})^2+1}{\sigma^* _{\min}}\geq\frac{2\sigma^* _{\max}}{\sigma^* _{\min}}>1.
\$
By Lemma \ref{tv_norm_2}, we have 
\%\label{092306}
&\TV(p_{B}(x),p_{B^* }(x))\notag\\
&\quad\geq \frac{1}{100}\min\big\{1,\|(B^* B^{* T}+I_d)^{-1/2}(BB^T-B^* B^{* T})(B^* B^{* T}+I_d)^{-1/2}\|_{\rF}\big\}.
\%
Note that
\%\label{092307}
&\|(B^* B^{* T}+I_d)^{-1/2}(BB^T-B^* B^{* T})(B^* B^{* T}+I_d)^{-1/2}\|_{\rF}\notag\\
&\quad\geq \frac{\|BB^T-B^* B^{* T}\|_{\rF}}{\|B^* B^{* T}+I_d\|_2}\geq\frac{\|BB^T-B^* B^{* T}\|_{\rF}}{(\sigma^* _{\max})^2+1}.
\%
Thus, by \eqref{092306} and \eqref{092307}, it holds that
\%\label{092308}
\TV(p_{B}(x),p_{B^* }(x))\geq \frac{1}{100}\min\bigg\{1,\frac{\|BB^T-B^* B^{* T}\|_{\rF}}{(\sigma^* _{\max})^2+1}\bigg\}
\%

Finally, by \eqref{092309} and \eqref{092308}, we have
\$
\TV\big(\P_{B\hat O}(x,z),\P_{B^* }(x,z)\big)&\leq \frac{500c(\sigma^* _{\max})^2\big((\sigma^* _{\max})^2+1\big)}{(\sigma^* _{\min})^3}\cdot\TV(p_{B}(x),p_{B^* }(x))\notag\\
&\leq\frac{500c(\sigma^* _{\max}+1)^4}{(\sigma^* _{\min})^3}\cdot\TV(p_{B}(x),p_{B^* }(x)).
\$

\end{proof}

\subsection{Bracketing Number}\label{factor2}
By an application of $\epsilon$-discretization technique, we upper bound the bracketing number of $\mP(\mathcal{B})$ as follows.

\begin{lemma}\label{factor_bn}
Let $\mP_{\mathcal{X}} (\mathcal{B}):=\{\mN (0,BB^T+ I_d)\,|\, B\in\mathcal{B}\}$, where $\mathcal{B}=\{B\in\R^{d\times r}\,|\,\|B\|_2\leq D\}$ for some $D>0$. Then the entropy can be bounded as follows,
\$
\log N_{\b}(\mP_{\mathcal{X}}(\mathcal{B}),1/m)\leq 4dr\log\big(24mdr(D^2+1)\big).
\$
\end{lemma}



\begin{proof}[Proof of Lemma \ref{factor_bn}]
We consider a set of Gaussian distribution
\$
\mP_{\mathcal{X}}(\mathcal{B}):=\bigg\{p_{\Sigma}(x)=\frac{1}{\sqrt{(2\pi)^d|\Sigma|}}e^{-\frac{1}{2}x^{T}\Sigma^{-1}x}\,\bigg|\, \Sigma=BB^{T}+I_d, B\in\mathcal{B}\bigg\},
\$
where $\mathcal{B}=\{B\in\R^{d\times r}\,|\,\|B\|_2\leq D\}$. Note that 
\%
\lambda_{\max}(\Sigma^{-1})=\big(\lambda_{\min}(\Sigma)\big)^{-1}=1,~\lambda_{\min}(\Sigma^{-1})=\big(\lambda_{\max}(\Sigma)\big)^{-1}\geq \frac{1}{D^2+1}.
\%
Here we denote by $\lambda_{\max}(\Sigma^{-1})$ and $\lambda_{\min}(\Sigma^{-1})$ the largest eigenvalue and the smallest eigenvalue of $\Sigma^{-1}$, respectively. Our goal is to find a $1/m$-bracket $\mN_{\b}(\mP_{\mathcal{X}}(\mathcal{B}),1/m)$ of $\mP_{\mathcal{X}}(\mathcal{B})$. In other words, for any $p_{\Sigma}(x)\in\mP_{\mathcal{X}}(\mathcal{B})$, we need to define $\bar p_{\Sigma}(x)\in\mN_{\b}(\mP_{\mathcal{X}}(\mathcal{B}),1/m)$ such that
\begin{itemize}
    \item $\bar p_{\Sigma}(x)\geq p_{\Sigma}(x),~\forall x\in\R^d$
    \item $\int |\bar p_{\Sigma}(x)-p_{\Sigma}(x)|\,dx\leq 1/m$.
\end{itemize}
Note that rank$(BB^{T})=r<d$ and $ \Sigma=BB^{T}+I_d$. Thus, the eigendecomposition of $\Sigma^{-1}$ has the following form
\%
\Sigma^{-1}=V 
  \begin{bmatrix}
    \lambda_1 & & & & &\\
    & \ddots & & & &\\
    & & \lambda_{r} & & &\\
    & & & 1 & &\\
    & & & & \ddots &\\
    & & & & & 1
  \end{bmatrix}
  V^{T}
  =U
  \begin{bmatrix}
    \lambda_1-1 & &\\
    & \ddots &\\
    & & \lambda_r-1
  \end{bmatrix}
  U^{T}+I_d,
\%
where $VV^T=V^{T}V=I_d$ and $U\in\R^{d\times r}$ is the first $r$ columns of $V$. For notation simplicity, we denote
\$
\Lambda:=  
  \begin{bmatrix}
    \lambda_1-1 & &\\
    & \ddots &\\
    & & \lambda_r-1
  \end{bmatrix}.
\$
Thus, we have $\Sigma^{-1}=U\Lambda U^{T}+I_d$. For some fixed $0<\epsilon\leq (D^2+1)^{-1}/2$ (which we will choose later), if $\lambda_i\in[k\epsilon, (k+1)\epsilon)$ for some $k\in\mathbb{Z}$, we define $\bar \lambda_i:=(k-1)\epsilon$. Note that $\lambda_i\geq \lambda_{\min}(\Sigma^{-1})\geq (D^2+1)^{-1}$. Thus, it holds that $k\geq 2$ and $\bar \lambda_i=(k-1)\epsilon\geq\epsilon>0$. Moreover, we have $\epsilon\leq\lambda_i-\bar \lambda_i\leq 2\epsilon$. We define
\$
\bar{\Lambda}:=  
  \begin{bmatrix}
    \bar{\lambda}_1-1 & &\\
    & \ddots &\\
    & & \bar{\lambda_r}-1
  \end{bmatrix}.
\$
For the matrix $U=(u_{i,j})\in\R^{d\times r}$, if $u_{i,j}\in[\frac{k\epsilon}{3\sqrt{dr}},\frac{(k+1)\epsilon}{3\sqrt{dr}})$ for some $k\in\mathbb{Z}$, we define $\bar{u}_{i,j}:=\frac{k\epsilon}{3\sqrt{dr}}$ and $\bar{U}:=(\bar{u}_{i,j})\in\R^{d\times r}$. It then holds that
\%
\|U-\bar{U}\|_2\leq \|U-\bar{U}\|_{F}=\sqrt{\sum_{i,j}|u_{i,j}-\bar{u}_{i,j}|^2}\leq \sqrt{dr}\cdot\frac{\epsilon}{3\sqrt{dr}}=\frac{\epsilon}{3}.
\%
We define
\%
\overline{\Sigma^{-1}}:=\bar{U}\bar{\Lambda}\bar{U}^{T}+I_d.
\%
Note that $(D^2+1)^{-1}\leq \lambda_i\leq 1$ and $|u_{i,j}|\leq 1$. Thus, we totally have 
\%\label{092209}
\bigg(\frac{1-(D^2+1)^{-1}}{\epsilon}\bigg)^{r}\cdot\bigg(\frac{6\sqrt{dr}}{\epsilon}\bigg)^{dr}=\bigg(\frac{D^2}{(D^2+1)\epsilon}\bigg)^r\cdot\bigg(\frac{6\sqrt{dr}}{\epsilon}\bigg)^{dr}
\%
many $\bar{\Sigma}^{-1}$. Note that for any $\|x\|_2=1$, we have
\$
x^{T}(\Sigma^{-1}-\overline{\Sigma^{-1}})x&=x^{T}(U^T\Lambda U-\bar{U}\bar{\Lambda}\bar{U}^{T})x\\
&=x^{T}U^{T}(\Lambda-\bar{\Lambda})Ux+x^T(U-\bar{U})^T\bar{\Lambda}(U+\bar{U})x\\
&\geq \lambda_{\min}(\Lambda-\bar{\Lambda})-\|(U-\bar{U})^T\bar{\Lambda}(U+\bar{U})\|_2\\
&\geq \lambda_{\min}(\Lambda-\bar{\Lambda})-\|U-\bar{U}\|_2\cdot\|\bar{\Lambda}(U+\bar{U})\|_2\\
&\geq \epsilon-3\bigg(2\epsilon+\frac{D^2}{D^2+1}\bigg)\|U-\bar{U}\|_2\\
&\geq \epsilon-3\bigg(2\epsilon+\frac{D^2}{D^2+1}\bigg)\cdot\frac{\epsilon}{3}\geq0,
\$
where the third inequality follows from
\$
\|\bar{\Lambda}(U+\bar{U})\|_2\leq \|\bar{\Lambda}\|_2\|U+\bar{U}\|_2\leq \bigg(2\epsilon+1-\frac{1}{D^2+1}\bigg)\cdot\bigg(2+\frac{\epsilon}{3}\bigg)\leq 3\bigg(2\epsilon+\frac{D^2}{D^2+1}\bigg).
\$
and the last inequality follows from our assumption $\epsilon\leq (D^2+1)^{-1}/2$. Thus, for any $x\in\R^{d}$, it holds that
\%\label{092208}
x^{T}(\Sigma^{-1}-\overline{\Sigma^{-1}})x\geq 0.
\%

We consider $\bar{p}_{\Sigma}(x)$ of the following form
\$
\bar{p}_{\Sigma}(x)=c\sqrt{\frac{|\overline{\Sigma^{-1}}|}{(2\pi)^d}}e^{-\frac{1}{2}x^{T}\overline{\Sigma^{-1}}x}.
\$
By \eqref{092208}, we have:  $\bar{p}_{\Sigma}(x)\geq {p}_{\Sigma}(x)$ holds for any $x\in\R^d$ if and only if
\$
c\geq\sqrt{\frac{|\Sigma^{-1}|}{|\overline{\Sigma^{-1}}|}}=\sqrt{\frac{\lambda_1\ldots\lambda_r}{\bar\lambda_1\ldots\bar\lambda_r}}.
\$
Note that
\$
\frac{\lambda_i}{\bar\lambda_i}\leq \frac{(k+1)\epsilon}{(k-1)\epsilon}=1+\frac{2}{k-1}\leq 1+\frac{4}{k}\leq 1+4(D^2+1)\epsilon,
\$
where the second inequality follows from $k\geq2$ and the last inequality follows from $k\epsilon\geq (D^2+\sigma^2)^{-1}$. We then obtain that
\$
\sqrt{\frac{\lambda_1\ldots\lambda_r}{\bar\lambda_1\ldots\bar\lambda_r}}\leq \big(1+4(D^2+1)\epsilon\big)^{r/2}.
\$
Let $c=(1+4(D^2+1)\epsilon)^{r/2}$. It then holds that
\$
c\geq \sqrt{\frac{\lambda_1\ldots\lambda_r}{\bar\lambda_1\ldots\bar\lambda_r}},
\$
which implies $\bar{p}_{\Sigma}(x)\geq {p}_{\Sigma}(x)$ holds for any $x\in\R^d$.
Note that
\$
\int |\bar p_{\Sigma}(x)-p_{\Sigma}(x)|\,dx=c-1= (1+4(D^2+1)\epsilon)^{r/2}-1\leq4(D^2+1)\epsilon r,
\$
where the last inequality follow from $(1+x)^{r/2}\leq 1+rx$ for $x\leq r^{-1}$. Let 
\%\label{092210}
\epsilon=\frac{1}{4(D^2+1)m r}.
\%
We have
\$
\int |\bar p_{\Sigma}(x)-p_{\Sigma}(x)|\,dx\leq4(D^2+1)\epsilon r=\frac{1}{m}.
\$
By \eqref{092209} and \eqref{092210}, we show that
\$
N_{\b}(\mP_{\mathcal{X}}(\mathcal{B}),1/m)\leq(4rmD^2)^r\cdot\big(24rm(D^2+1)\sqrt{dr}\big)^{dr},
\$
which implies 
\$
\log N_{\b}(\mP_{\mathcal{X}}(\mathcal{B}),1/m)\leq 4dr\log\big(24mdr(D^2+1)\big).
\$
\end{proof}
\subsection{Rademacher Complexity}\label{factor3}


Note that for fixed $B$ the prediction function class 
\$
\mathcal{G}_{B,\mathcal{C}}:=\big\{g_{B,\beta}(x)=\beta^{T}B^{T}(BB^T+\sigma^2 I_d)^{-1}x\,\big|\,\beta\in\mathcal{C}\big\}
\$
belongs to a linear hypothesis class. For a linear hypothesis class $\mathcal{H}$, we can bound its empirical Rademacher complexity as follows.

\begin{lemma}\label{factor_rc}
For a linear hypothesis class $\mathcal{H}=\{h_{\beta}(x)=\beta^{T}x\,|\,\beta\in\R^r, \|\beta\|_2\leq D\}$, where $x\in\R^r$ and $\|x\|_2\leq X$, the empirical Rademacher complexity can be bounded as follows,
\$
\hat R_n(\mathcal{H})\leq \frac{2DX}{\sqrt{n}}.
\$
\end{lemma}

\begin{proof}[Proof of Lemma \ref{factor_rc}]
Note that
\$
\hat R_n(\mathcal{H})&=\frac{2}{n}\E_{\sigma_i}\bigg[\sup_{\|\beta\|_2\leq D}\sum^n_{i=1}\sigma_i\cdot \beta^{T}x_i\bigg]=\frac{2}{n}\E_{\sigma_i}\bigg[\sup_{\|\beta\|_2\leq D}\beta^{T}\bigg(\sum^n_{i=1}\sigma_ix_i\bigg)\bigg]\\
&\leq \frac{2}{n}\E_{\sigma_i}\bigg[\sup_{\|\beta\|_2\leq D}\|\beta\|_2\bigg\|\sum^n_{i=1}\sigma_ix_i\bigg\|_2\bigg]\leq\frac{2D}{n}\E_{\sigma_i}\bigg[\sqrt{\sum_{i,j}\sigma_i\sigma_j x^T_ix_j}\bigg].
\$
By Jensen's inequality, we then have
\$
\hat R_n(\mathcal{H})\leq\frac{2D}{n}\E_{\sigma_i}\bigg[\sqrt{\sum_{i,j}\sigma_i\sigma_j x^T_ix_j}\bigg]\leq\frac{2D}{n}\sqrt{E_{\sigma_i}\bigg[\sum_{i,j}\sigma_i\sigma_j x^T_ix_j\bigg]}=\frac{2D}{n}\sqrt{\sum^n_{i=1} \|x_i\|^2}\leq \frac{2DX}{\sqrt{n}}.
\$

\end{proof}
\subsection{Proofs for Theorem \ref{factor_main}}\label{factor4}
In this section, we verify the utility of Algorithm \ref{mle+erm} by proving Theorem \ref{factor_main}. Recall that the truncated squared loss is defined as 
\begin{align}
\tilde{\ell} (x,y) := (y-x)^{2} \mathbb{I}_{\{(y-x)^2\leq L\}} + L \cdot \mathbb{I}_{\{(y-x)^2 > L\}},
\end{align}
which is $L-$bounded and $2\sqrt{L}-$Lipschitz w.r.t. the first argument. Before proving Theorem \ref{factor_main}, we need to state some core lemmas. 
Recall the definition of $g_{B,\beta}(x)$: 
\begin{align}
g_{B,\beta}(x):=\argmin_{g}\mathbb{E}_{B,\beta} [\ell(g(x),y)].
\end{align}
Since $\ell$ is the squared loss, it's obvious that 
\begin{align}
g_{B, \beta}(x):=\argmin_{g}\mathbb{E}_{B,\beta}[\ell(g(x),y)]=\mathbb{E}_{\P_{B, \beta}(x, y)}[y \mid x]=\beta^{T}B^{T}(BB^{T}+I_{d})^{-1}x.    
\end{align}
The next lemma shows that the optimal predictor under the squared loss $\ell$ and the truncated squared loss $\tilde{\ell}$ stays the same.
\begin{lemma} \label{same_prediction}
We denote by $\tilde{g}_{B,\beta}$ the optimal predictor under truncated squared loss, i.e.,
\%
\tilde{g}_{B,\beta}\leftarrow\argmin_{g}\mathbb{E}_{B,\beta}[\tilde{\ell}(g(x),y)].
\%
It then holds that
\begin{align}
\tilde{g}_{B,\beta}(x)=\mathbb{E}_{\P_{B, \beta}(x, y)}[y \mid x]={g}_{B,\beta}(x).
\end{align}
\end{lemma}
\begin{proof}[Proof of Lemma \ref{same_prediction}]
Notice that, the distribution (under parameter $B,\beta$) of $y$ given $x$ is a Gaussian distribution with mean $\mu=\mathbb{E}_{P_{B, \beta}(x, y)}[y \mid x]$ and variance $v^{2}$ (which is of no importance). We define function $f$ as
\begin{align}
f(a )&:=\mathbb{E}_{B,\beta}[\tilde{\ell}(a,y) \mid x]\notag \\
&=  \int_{a-\sqrt{L}}^{a+\sqrt{L}}(y-a)^{2}  \frac{1}{v\sqrt{2 \pi}}e^{-\frac{(y-\mu)^{2}}{2v^{2}}}\mathrm{d}y +  \int_{a+\sqrt{L}}^{+\infty} L \frac{1}{v\sqrt{2 \pi}}e^{-\frac{(y-\mu)^{2}}{2v^{2}}}\mathrm{d}y \notag\\
&\quad+ \int_{-\infty}^{a-\sqrt{L}} L \frac{1}{v\sqrt{2 \pi}}e^{-\frac{(y-\mu)^{2}}{2v^{2}}}\mathrm{d}y.
\end{align}
Then, it holds that
\begin{align}
f'(a)&=\frac{L}{v\sqrt{2\pi}}e^{-\frac{(a-\mu+\sqrt{L})^{2}}{2v^{2}}} - \frac{L}{v\sqrt{2\pi}}e^{-\frac{(a-\mu-\sqrt{L})^{2}}{2v^{2}}} + \int_{a-\sqrt{L}}^{a+\sqrt{L}}2(a-y)\frac{1}{v\sqrt{2 \pi}}e^{-\frac{(y-\mu)^{2}}{2v^{2}}}\mathrm{d}y \notag \\
&\quad - \frac{L}{v\sqrt{2\pi}}e^{-\frac{(a-\mu+\sqrt{L})^{2}}{2v^{2}}} + \frac{L}{v\sqrt{2\pi}}e^{-\frac{(a-\mu-\sqrt{L})^{2}}{2v^{2}}} \notag \\
& = \int_{a-\sqrt{L}}^{a+\sqrt{L}}2(a-y)\frac{1}{v\sqrt{2 \pi}}e^{-\frac{(y-\mu)^{2}}{2v^{2}}}\mathrm{d}y \notag \\
&= \int_{a-\sqrt{L}}^{a}2(a-y)\frac{1}{v\sqrt{2 \pi}}e^{-\frac{(y-\mu)^{2}}{2v^{2}}}\mathrm{d}y + \int_{a}^{a+\sqrt{L}}2(a-y)\frac{1}{v\sqrt{2 \pi}}e^{-\frac{(y-\mu)^{2}}{2v^{2}}}\mathrm{d}y \notag \\
&= \int_{0}^{\sqrt{L}}2z \frac{1}{v\sqrt{2 \pi}}e^{-\frac{(a-z-\mu)^{2}}{2v^{2}}}\mathrm{d}z - \int_{0}^{\sqrt{L}}2z \frac{1}{v\sqrt{2 \pi}}e^{-\frac{(a+z-\mu)^{2}}{2v^{2}}}\mathrm{d}z \notag \\
&= \int_{0}^{\sqrt{L}} \frac{2z}{v\sqrt{2 \pi}}(e^{-\frac{(a-z-\mu)^{2}}{2v^{2}}}-e^{-\frac{(a+z-\mu)^{2}}{2v^{2}}}) \mathrm{d}z.
\end{align}
Notice that for $z\in [0,\sqrt{L}]$,
\begin{align}
e^{-\frac{(a-z-\mu)^{2}}{2v^{2}}}-e^{-\frac{(a+z-\mu)^{2}}{2v^{2}}}>0 \text{ when } a>\mu,
\end{align}
\begin{align}
e^{-\frac{(a-z-\mu)^{2}}{2v^{2}}}-e^{-\frac{(a+z-\mu)^{2}}{2v^{2}}}<0 \text{ when } a<\mu.
\end{align}
Therefore, we have $f'(a)<0$ when $a<\mu$, $f'(a)>0$ when $a>\mu$, which implies that $a=\mu$ is the unique minimizer of $f(a)$, i.e.,
\begin{align}
\tilde{g}_{B,\beta}(x)=\mathbb{E}_{\P_{B, \beta}(x, y)}[y \mid x]={g}_{B,\beta}(x).
\end{align}
\end{proof}

The following lemma shows that the truncation has no significant influence on the excess risk.
\begin{lemma} \label{truncation}
There exist $c_{2}=(D^{2}+1)^3$, such that 
\begin{align}
\operatorname{Error}_{\ell}(\hat{B}, \hat{\beta}) \leq  
\mathbb{E}_{B^{*},\beta^{*}}[\tilde{\ell} (g_{\hat{B},\hat{\beta}}(x),y)] - \mathbb{E}_{B^{*},\beta^{*}}[\tilde{\ell} (g_{B^{*},\beta^{*}}(x),y)] + \sqrt{\frac{2Lc_{2}}{\pi}}e^{-\frac{L}{2c_{2}}}.
\end{align}
\end{lemma}
\begin{proof}[Proof of Lemma \ref{truncation}]
\begin{align}\label{092720}
\operatorname{Error}_{\ell}(\hat{B}, \hat{\beta}) &= \mathbb{E}_{B^{*},\beta^{*}}[\ell (g_{\hat{B},\hat{\beta}}(x),y)] - \mathbb{E}_{B^{*},\beta^{*}}[\ell (g_{B^{*},\beta^{*}}(x),y)] \notag \\
& =  \mathbb{E}_{B^{*},\beta^{*}}[\ell (g_{\hat{B},\hat{\beta}}(x),y)] - \mathbb{E}_{B^{*},\beta^{*}}[\tilde{\ell} (g_{\hat{B},\hat{\beta}}(x),y)]\notag \\
& \quad +\mathbb{E}_{B^{*},\beta^{*}}[\tilde{\ell} (g_{\hat{B},\hat{\beta}}(x),y)] - \mathbb{E}_{B^{*},\beta^{*}}[\tilde{\ell} (g_{B^{*},\beta^{*}}(x),y)] \notag \\
& \quad + \mathbb{E}_{B^{*},\beta^{*}}[\tilde{\ell} (g_{B^{*},\beta^{*}}(x),y)] - \mathbb{E}_{B^{*},\beta^{*}}[\ell (g_{B^{*},\beta^{*}}(x),y)] \quad (\leq 0 \text{ since } \tilde{\ell} \leq \ell ) \notag \\
& \leq \sup_{B,\beta} \{ \mathbb{E}_{B^{*},\beta^{*}}[\ell (g_{B,\beta}(x),y)] - \mathbb{E}_{B^{*},\beta^{*}}[\tilde{\ell} (g_{B,\beta}(x),y)] \} \notag \\ 
& \quad + \mathbb{E}_{B^{*},\beta^{*}}[\tilde{\ell} (g_{\hat{B},\hat{\beta}}(x),y)] - \mathbb{E}_{B^{*},\beta^{*}}[\tilde{\ell} (g_{B^{*},\beta^{*}}(x),y)]
\end{align}
For the first term, we have
\begin{align}\label{092721}
& \quad \sup_{B,\beta} \{ \mathbb{E}_{B^{*},\beta^{*}}[\ell (g_{B,\beta}(x),y)] - \mathbb{E}_{B^{*},\beta^{*}}[\tilde{\ell} (g_{B,\beta}(x),y)] \} \notag \\
&= \sup_{B,\beta} \{\mathbb{E}_{B^{*},\beta^{*}} ((g_{B,\beta}(x)-y)^{2}-L) \mathds{1}_{\{(g_{B,\beta}(x)-y)^{2} \geq L\}}  \}.
\end{align}
Notice that
\begin{align}
g_{B,\beta}(x)-y = \beta^{T}B^{T}(BB^{T}+I_{d})^{-1}x-y \sim \mathcal{N}(0, \lambda^{2}),
\end{align}
where 
\begin{align}
\lambda^{2}&=  Var_{B^{*},\beta^{*}} [g_{B,\beta}(x)-y]  \notag \\
&=\mathbb{E}_{B^{*},\beta^{*}} (\beta^{T}B^{T}(BB^{T}+I_{d})^{-1}x-y)^{2} \notag \\
&= \epsilon^{2}+\beta^{T}B^{T}(BB^{T}+I_{d})^{-1} (B^{*}B^{* T}+I_{d})(BB^{T}+I_{d})^{-1}B\beta \notag \\
&\quad + \beta^{* T}\beta^{*}- 2\beta^{T}B^{T}(BB^{T}+I_{d})^{-1}B^{*}\beta^{*} \notag \\
& \leq \epsilon^{2}+\beta^{* T}\beta^{*} + \|(BB^{T}+I_{d})^{-1}\|_{2}^{2} \cdot \|B^{*}B^{* T}+I_{d} \|_{2} \cdot \|B\beta\|_{2}^{2} \notag \\
&\quad + 2 \|(BB^{T}+I_{d})^{-1}\|_{2} \cdot \|B^{*}\beta^{*}\|_{2} \cdot  \|B\beta\|_{2} \notag \\
& \leq \epsilon^{2}+ \beta^{* T}\beta^{*} + D^{4} \|B^{*}B^{* T}+I_{d} \|_{2} + 2D^{2} \|B^{*}\beta^{*}\|_{2} \notag \\
& \leq 1+ D^{2} + D^{4}(D^{2}+1)+2D^{4} \notag \\
& \leq c_{2}.
\end{align}
Therefore 
\begin{align}\label{092722}
& \quad \sup_{B,\beta} \{\mathbb{E}_{B^{*},\beta^{*}} ((g_{B,\beta}(x)-y)^{2}-L) \mathds{1}_{\{(g_{B,\beta}(x)-y)^{2} \geq L\}}  \}   \notag \\
&= \sup_{\lambda} 2\int_{\sqrt{L}}^{+\infty}\frac{1}{\lambda\sqrt{2\pi}}(x^{2}-L)e^{-\frac{x^{2}}{2\lambda^{2}}} \mathrm{d}x \notag \\
&= 2\sup_{\lambda} \bigg\{ -\frac{\lambda}{\sqrt{2\pi}}xe^{-\frac{x^{2}}{2\lambda^{2}}} \bigg|_{\sqrt{L}}^{+\infty} + (\lambda^{2}-L) \int_{\sqrt{L}}^{+\infty}\frac{1}{\lambda\sqrt{2\pi}}e^{-\frac{x^{2}}{2\lambda^{2}}}  \mathrm{d}x \bigg\} \notag \\
& = 2\sup_{\lambda} \bigg\{ \sqrt{\frac{L}{2\pi}} \lambda e^{-\frac{L}{2\lambda^{2}}} + (\lambda^{2}-L) \int_{\sqrt{L}}^{+\infty}\frac{1}{\lambda\sqrt{2\pi}}e^{-\frac{x^{2}}{2\lambda^{2}}}  \mathrm{d}x \bigg\} \notag \\
& \leq 2\sup_{\lambda} \bigg\{ \sqrt{\frac{L}{2\pi}} \lambda e^{-\frac{L}{2\lambda^{2}}} \bigg\} \quad (\text{since }L \geq c_{2} \geq \lambda^{2}) \notag \\
&= \sqrt{\frac{2L c_{2}}{\pi}} e^{-\frac{L}{2c_{2}}}.
\end{align}
The last equation holds since $\lambda e^{-\frac{L}{2\lambda^{2}}}$ monotone increases w.r.t. $\lambda$, and $\lambda \leq \sqrt{c_{1}}$. Combining \eqref{092720}, \eqref{092721} and \eqref{092722}, we finish the proof.
\end{proof}

Now we are ready to prove Theorem \ref{factor_main}.
\begin{proof}[Proof of Theorem \ref{factor_main}]
Note that $\tilde{l}$ is $L-$bounded. By Lemma \ref{same_prediction},  we can apply Theorem \ref{error_bound} to $\tilde{l}$, which gives
\begin{align} \label{apply_main}
&\mathbb{E}_{B^{*},\beta^{*}}[\tilde{\ell} (g_{\hat{B},\hat{\beta}}(x),y)] - \mathbb{E}_{B^{*},\beta^{*}}[\tilde{\ell} (g_{B^{*},\beta^{*}}(x),y)] \notag \\
&\leq 2 \max _{B \in \mathcal{B}} R_n\left(\tilde{\ell} \circ \mathcal{G}_{B , \mathcal{C}}\right)+L \cdot \sqrt{\frac{2}{n} \log \frac{4}{\delta}}+12 \kappa L \cdot \sqrt{\frac{1}{m} \log \frac{2N_{\b}(\mathcal{P}_{\mathcal{X}}(\mathcal{B}), 1 / m)}{\delta}} . 
\end{align}
Here $\kappa={c_{1}(\sigma_{max}^{*}+1)^{4}}/{\sigma_{min}^{*3}}$ is the transferability defined in Lemma \ref{factor_ti}.

By Lemma \ref{factor_bn}, we have
\begin{align} \label{bracketing_bound}
\log N_{\b}(\mathcal{P}(\mathcal{B}), 1 / m) \leq 4dr \log (24mdr(D^{2}+1)).
\end{align}

Since $\tilde{l}$ is $2\sqrt{L}-$Lipschitz w.r.t. the first argument, the contraction principle (Theoerem 4.12 in \cite{ledoux1991probability}) gives
\begin{align}
R_n\left(\tilde{\ell} \circ \mathcal{G}_{B , \mathcal{C}}\right) \leq 2\sqrt{L} R_n\left(\mathcal{G}_{B , \mathcal{C}}\right).
\end{align}
Therefore it remains to bound $R_n\left(\mathcal{G}_{B , \mathcal{C}}\right).$ By Lemma \ref{factor_rc}, for fixed $B$,
\begin{align}\label{092730}
R_n\left(\mathcal{G}_{B, \mathcal{C}}\right) &= \mathbb{E}_{\{x_{j}\}_{j=1}^{n}} \mathbb{E}_{\{\sigma_{j}\}_{j=1}^{n}}  [\sup_{\beta}\frac{2}{n}\sum_{j=1}^{n} \sigma_{j} g_{B,\beta}(x_{j})] \notag \\
&= \mathbb{E}_{\{x_{j}\}_{j=1}^{n}} \mathbb{E}_{\{\sigma_{j}\}_{j=1}^{n}} [\sup_{\beta}\frac{2}{n}\sum_{j=1}^{n} \sigma_{j} \beta^{T} B^{T} (BB^{T}+I_{d})^{-1} x_{j}] \notag \\
& \leq \mathbb{E}_{\{x_{j}\}_{j=1}^{n}} [\frac{2D}{\sqrt{n}} \sup_{j} \|B^{T} (BB^{T}+I_{d})^{-1} x_{j}\|_{2}]  \quad \text{(By Lemma \ref{factor_rc}, since }\|\beta\|_2\leq D ) \notag \\
&= \frac{2D}{\sqrt{n}} \mathbb{E}_{\{x_{j}\}_{j=1}^{n}} [ \sup_{j} \|B^{T} (BB^{T}+I_{d})^{-1} x_{j}\|_{2}].
\end{align}
Note that $x_{j}\sim \mathcal{N}(0,B^{*}B^{* T}+ I_{d})$. Therefore $B^{T} (BB^{T}+I_{d})^{-1} x_{j} \sim \mathcal{N}(0,\Sigma)$, where
\begin{align}
\Sigma:= B^{T} (BB^{T}+I_{d})^{-1}  (B^{*}B^{* T}+ I_{d}) (BB^{T}+I_{d})^{-1} B.
\end{align}
Thus, we have
\begin{align}
\Sigma^{-\frac{1}{2}} B^{T} (BB^{T}+I_{d})^{-1} x_{j} \sim \mathcal{N}(0,I_{r}).
\end{align}
Let $u_{j}:=\Sigma^{-\frac{1}{2}} B^{T} (BB^{T}+I_{d})^{-1} x_{j}$, then
\begin{align}\label{092731}
&\mathbb{E}_{\{x_{j}\}_{j=1}^{n}} [ \sup_{j} \|B^{T} (BB^{T}+I_{d})^{-1} x_{j}\|_{2}] \notag \\
&= \mathbb{E}_{\{x_{j}\}_{j=1}^{n}} [ \sup_{j} \|\Sigma^{\frac{1}{2} }u_{j}\|_{2}] \notag \\
&\leq \mathbb{E}_{\{x_{j}\}_{j=1}^{n}} [  \sup_{j} \|\Sigma^{\frac{1}{2}}\|_{2} \| u_{j}\|_{2}] \notag \\ &\leq \sup \|\Sigma^{\frac{1}{2}}\|_{2}\mathbb{E}_{\{x_{j}\}_{j=1}^{n}} [  \sup_{j}  \| u_{j}\|_{2}] .
\end{align}
By the Theorem 3.1.1 in \cite{vershynin2018high}, $\|u_{j}\|-\sqrt{r}$ is $c_{4}-$subGaussian for some absolute constant $c_{4}$. Therefore, for any $t>0$,
\begin{align}
e^{\mathbb{E}[t \sup_{j}\|u_{j}\|_{2}]} &\leq \mathbb{E} [e^{t \sup_{j}\|u_{j}\|_{2}}] \quad \text{(by Jensen's inequality)}\notag \\ 
& \leq \sum_{j=1}^{n} \mathbb{E} [e^{t \|u_{j}\|_{2}}] \notag \\
& = \sum_{j=1}^{n} \mathbb{E} [e^{t \|u_{j}\|_{2}-\sqrt{r}}]e^{t\sqrt{r}} \notag \\
& \leq \sum_{j=1}^{n} e^{\frac{t^{2}}{2}c_{4}}e^{t\sqrt{r}} \notag \\
&= n e^{t\sqrt{r}+\frac{t^{2}}{2}c_{4}}.
\end{align}
Taking log on both sides, we have
\begin{align}
\mathbb{E}[\sup_{j}\|u_{j}\|_{2}]    \leq \frac{\log n}{t} + \sqrt{r} + \frac{t}{2}c_{4},
\end{align}
which holds for any $t>0$. Take $t=\sqrt{\frac{2 \log n}{c_{4}}}$, we get
\begin{align} \label{092732}
\mathbb{E}[\sup_{j}\|u_{j}\|_{2}]    \leq \sqrt{2c_{4} \log n} + \sqrt{r}.    
\end{align}
Note that
\begin{align}\label{092733}
\|\Sigma\|_{2}&=\|B^{T} (BB^{T}+I_{d})^{-1}  (B^{*}B^{* T}+ I_{d}) (BB^{T}+I_{d})^{-1} B\|_{2} \notag \\
& \leq \|B\|_{2}^{2} \cdot \|(BB^{T}+I_{d})^{-1}\|_{2}^{2} \cdot \|B^{*}B^{* T}+ I_{d}\| \notag \\
& \leq (D^{2} +1)^{2},
\end{align}
i.e., $\sup \|\Sigma^{\frac{1}{2}}\|_{2} \leq (D^{2} +1)$.
Combining \eqref{092730}, \eqref{092731}, \eqref{092732} and \eqref{092733}, we have
\begin{align}
R_n\left(\mathcal{G}_{\phi, \Psi}\right) &\leq \frac{2D}{\sqrt{n}} \mathbb{E}_{\{x_{j}\}_{j=1}^{n}} [ \sup_{j} \|B^{T} (BB^{T}+I_{d})^{-1} x_{j}\|_{2}] \notag    \\
&\leq \frac{2D}{\sqrt{n}} \sup \|\Sigma^{\frac{1}{2}}\|_{2}\mathbb{E}_{\{x_{j}\}_{j=1}^{n}} [  \sup_{j}  \| u_{j}\|_{2}] \notag \\
& \leq \frac{2D}{\sqrt{n}}  (D^{2} +1) (\sqrt{2c_{4} \log n} + \sqrt{r}) ,
\end{align}
which implies
\begin{align} \label{R_bound}
\max_{\phi \in \Phi}R_n\left(\tilde{\ell} \circ \mathcal{G}_{\phi, \Psi}\right) \leq 2\sqrt{L} \max_{\phi in \Phi} R_n\left(\mathcal{G}_{\phi, \Psi}\right) \leq 2\sqrt{L}\frac{2D}{\sqrt{n}}  (D^{2} +1) (\sqrt{2c_{4} \log n} + \sqrt{r})
\end{align}
We are now ready to bound the excess risk.
By Lemma \ref{truncation}, we have
\begin{align}
\operatorname{Error}_{\ell}(\hat{B}, \hat{\beta}) &\leq  
\mathbb{E}_{B^{*},\beta^{*}}[\tilde{\ell} (g_{\hat{B},\hat{\beta}}(x),y)] - \mathbb{E}_{B^{*},\beta^{*}}[\tilde{\ell} (g_{B^{*},\beta^{*}}(x),y)] + \sqrt{\frac{2Lc_{2}}{\pi}}e^{-\frac{L}{2c_{2}}} \notag \\
& \leq 2 \max _{\phi \in \Phi} R_n\left(\tilde{\ell} \circ \mathcal{G}_{\phi, \Psi}\right)+L \cdot \sqrt{\frac{2}{n} \log \frac{4}{\delta}}\notag\\
&\quad+12 \kappa L \cdot \sqrt{\frac{1}{m} \log \frac{2N_{\b}(\mathcal{P}_{\mathcal{X}}(\mathcal{B}), 1 / m)}{\delta}} + \sqrt{\frac{2Lc_{2}}{\pi}}e^{-\frac{L}{2c_{2}}} \notag \\
& \leq 4\sqrt{L}\frac{2D}{\sqrt{n}}  (D^{2} +1) (\sqrt{2c_{4} \log n} + \sqrt{r})+L \cdot \sqrt{\frac{2}{n} \log \frac{4}{\delta}}\notag \\
&\quad+ 12 \kappa L \sqrt{\frac{1}{m}(4dr \log (24mdr(D^{2}+1)) + \log(2/\delta))}  + \sqrt{\frac{2Lc_{2}}{\pi}}e^{-\frac{L}{2c_{2}}},
\end{align}
where the second inequality follows from \eqref{apply_main} and the last inequality follows from \eqref{bracketing_bound}, \eqref{R_bound}. Here $c_4$ is an absolute constant. Note that $c_2=(D^2+1)^3$ and $L=c_{2}\log n$. Thus, we have
\begin{align}
\operatorname{Error}_{\ell}(\hat{B}, \hat{\beta}) &\leq 8\sqrt{2c_{4}}L \sqrt{\frac{1}{n}} + 8L\sqrt{\frac{r}{n}} +L \cdot \sqrt{\frac{2}{n} \log \frac{4}{\delta}}\notag \\ 
&\quad + 12 \kappa L \sqrt{\frac{1}{m}(4dr \log (24mdr(D^{2}+1)) + \log(2/\delta))}  + L\sqrt{\frac{2}{\pi n}}  \notag \\
& \leq \tilde{\mathcal{O}}\bigg(\kappa L \sqrt{\frac{dr}{m}}+ L\sqrt{\frac{r}{n}}\bigg),
\end{align}
where $L=(D^2+1)^3\log n$ and $\kappa={c_{1}(\sigma_{max}^{*}+1)^{4}}/{\sigma_{min}^{*3}}$ for some absolute constants $c_1$.



\end{proof}

\subsection{Proofs for Theorem \ref{factor_fast_rate}} \label{fast_rate}
In this section, we provide a refined analysis for proving Theorem \ref{factor_fast_rate}. First notice that we can rewrite our model (without $z$) as 
\begin{align}
y=\beta^{* T}C^{*}x+w,   
\end{align}
where $\beta^{*} \in \mathbb{R}^{r \times 1}$, $C^{*}=B^{* T}(B^{*}B^{* T}+I_{d})^{-1}\in \mathbb{R}^{r \times d}$, $x \sim N(0, B^{*}B^{* T}+I_{d})$, $w \sim N(0,\epsilon^{2}+\|\beta^{*}\|_{2}^{2}-\beta^{* T}B^{* T}(B^{*}B^{* T}+I_{d})^{-1}B^{*}\beta^{*})$. Here $w$ and $x$ are independent. Therefore we can write our data as 
\begin{align}
Y=XC^{* T}\beta^{*} + W, 
\end{align}
where $Y=(y_1,\cdots,y_n)^{T}\in \mathbb{R}^{n \times 1}$, $X=(x_{1},\cdots,x_{n})^{T}\in \mathbb{R}^{n \times d}$, $W=(w_1,\cdots,w_n)^{T}\in \mathbb{R}^{n \times 1}$. 

In the first step (MLE), we obtain an estimator $\hat{B}$ and the corresponding estimator $\hat{C}=\hat{B}^{T}(\hat{B}\hat{B}^{T}+I_{d})^{-1}$. Then our estimator $\hat{\beta}$ for the second step (ERM) is given by 
\begin{align}
\hat{\beta} &= \argmin_{\beta} \|Y-X\hat{C}^{T}\beta\|_{2}^{2} \notag \\
&= ((X\hat{C}^{T})^{T}(X\hat{C}^{T}))^{-1}(X\hat{C}^{T})^{T}Y \notag \\
&= (\hat{C}X^{T}X\hat{C}^{T})^{-1}\hat{C}X^{T}Y.
\end{align}
Then our risk is given by
\begin{align}
{\rm Error}_{\ell}(\hat B,\hat\beta)&= \mathbb{E}_{\P_{B^* ,\beta^* }(x,y)}\big[\big(y-g_{\hat B,\hat\beta}(x)\big)^2\big]-\mathbb{E}_{\P_{B^* ,\beta^* }(x,y)}\big[\big(y-g_{B^* ,\beta^* }(x)\big)^2\big] \notag \\
&= \mathbb{E}[(\beta^{* T}C^{*}x + w - \hat{\beta}^{T}\hat{B}^{T}(\hat{B}\hat{B}^{T}+ I_d)^{-1}x)^{2}] - \mathbb{E}[w^{2}] \notag \\
&= \mathbb{E} [(\beta^{* T}C^{*}x - \hat{\beta}^{T}\hat{C}x)^2] \notag \\
&= (\beta^{* T}C^{*}-\hat{\beta}^{T}\hat{C}) (B^{*}B^{* T}+I_{d})     (\beta^{* T}C^{*}-\hat{\beta}^{T}\hat{C})^{T} \notag \\
& \leq \|B^{*}B^{* T}+I_{d}\|_{2} \|\hat{C}^{T}\hat{\beta}- C^{* T}\beta^{*}\|_{2}^{2}
\end{align}
Our goal is to bound $\|\hat{C}^{T}\hat{\beta}- C^{* T}\beta^{*}\|_{2}^{2}$. Consider the SVD of $C^{* T}$ and $\hat{C}^{T}$, i.e., $C^{* T}=U^{*}\Lambda^{*}V^{* T}$, $\hat{C}^{T}=\hat{U}\hat{\Lambda}\hat{V}^T$. Then, we have
\begin{align}
& \quad \hat{C}^{T}\hat{\beta}- C^{* T}\beta^{*} \notag \\  
&= \hat{C}^{T}(\hat{C}X^{T}X\hat{C}^{T})^{-1}\hat{C}X^{T}Y-C^{* T}\beta^{*} \notag \\ 
& = \hat{C}^{T}(\hat{C}X^{T}X\hat{C}^{T})^{-1}\hat{C}X^{T}(X C^{* T} \beta^{*} + W)-C^{* T}\beta^{*} \notag \\
& = (\hat{C}^{T}(\hat{C}X^{T}X\hat{C}^{T})^{-1}\hat{C}X^{T}X C^{* T} - C^{* T})\beta^{*} + \hat{C}^{T}(\hat{C}X^{T}X\hat{C}^{T})^{-1}\hat{C} X^{T} W \notag \\
&= (\hat{U}(\hat{U}^{T}X^{T}X\hat{U})^{-1}\hat{U}^{T}X^{T}X U^{*} - U^{*})\Lambda^{*}V^{* T} \beta^{*} + \hat{U}(\hat{U}^{T}X^{T}X\hat{U})^{-1}\hat{U}^{T}X^{T} W.
\end{align}
Therefore 
\begin{align} \label{bias_variance_decom}
\|\hat{C}^{T}\hat{\beta}- C^{* T}\beta^{*}\|_{2}^{2} &\leq 2 \|(\hat{U}(\hat{U}^{T}X^{T}X\hat{U})^{-1}\hat{U}^{T}X^{T}X U^{*} - U^{*})\|_{2}^{2} \|\Lambda^{*}\|^2_2 \|\beta^{*}\|_{2}^{2}   \notag \\
& \quad + 2\|\hat{U}(\hat{U}^{T}X^{T}X\hat{U})^{-1}\hat{U}^{T}X^{T} W\|_{2}^{2}
\end{align}
We give two lemmas for bounding the related terms. The first lemma considers the bias term: 
\begin{lemma} \label{bias_term}
Let $\Sigma :=B^{*}B^{* T}+I_{d} $. If $n \gtrsim \|\Sigma\|^{2} r \log (1/\delta)$, then with probability at least $1-\delta$, 
\begin{align}
\|(\hat{U}(\hat{U}^{T}X^{T}X\hat{U})^{-1}\hat{U}^{T}X^{T}X U^{*} - U^{*})\|_{2}^{2} \leq \mathcal{O} (\|\Sigma\|^{2} \Delta^{2}),    
\end{align}
where $\Delta=dist(\hat{U},U^{*}):=\|\hat{U}\hat{U}^{T}-U^{*}U^{* T} \| $. 
\end{lemma}

The second lemma considers the variance term:
\begin{lemma} \label{variance_term}
Let $\Sigma :=B^{*}B^{* T}+I_{d} $. If $n \gtrsim \|\Sigma\|^{2} r \log (1/\delta)$, then with probability at least $1-\delta$, 
\begin{align}
\|\hat{U}(\hat{U}^{T}X^{T}X\hat{U})^{-1}\hat{U}^{T}X^{T} W\|_{2}^{2} \leq \mathcal{O} \bigg(\frac{\sigma^{2}r \log(4/\delta)}{n}\bigg),    
\end{align}
where $\sigma^{2}:=\mathbb{E}(w^{2})=\epsilon^{2}+\|\beta^{*}\|_{2}^{2}-\beta^{* T}B^{* T}(B^{*}B^{* T}+I_{d})^{-1}B^{*}\beta^{*}$ is the variance of $w$.
\end{lemma}
Using this two lemmas together with the decomposition (\ref{bias_variance_decom}), we have 
\begin{align}
\|\hat{C}^{T}\hat{\beta}- C^{* T}\beta^{*}\|_{2}^{2} \leq \mathcal{O}\bigg(\|\beta^{*}\|^{2}\|\Lambda^{*}\|^2\|\Sigma\|^{2} \Delta^{2} + \frac{\sigma^{2}r \log(4/\delta)}{n } \bigg).    
\end{align}
Now it remains to control $\Delta$, which is related to the estimation error of the first step (MLE). The following lemma gives an upper bound for $\Delta$. 
\begin{lemma} \label{Delta_lemma}
If $m \gtrsim \|\Sigma\|^{2} d \log (1/\delta)$, then with probability at least $1-\delta$, 
\begin{align}
\Delta^{2} &\leq \mathcal{O} \bigg(\|\Sigma\|^{2}\frac{d \log (1/\delta)}{m} \lambda_{r}^{-2}(C^{* T}C^{*})\bigg),
\end{align}where $\lambda_{r}(C^{* T}C^{*})$ is the $r$-th (smallest) nonzero eigenvalue of $C^{* T}C^{*}$.
\end{lemma}
\begin{proof}[Proof for Theorem \ref{factor_fast_rate}]
By Lemma \ref{bias_term}, \ref{variance_term}, \ref{Delta_lemma}, we have 
\begin{align}
{\rm Error}_{\ell}(\hat B,\hat\beta) &\leq  \|\Sigma\|  \|\hat{C}^{T}\hat{\beta}- C^{* T}\beta^{*}\|_{2}^{2} \notag \\
&\leq \mathcal{O}(\|\beta^{*}\|^{2}\|\Lambda^{*}\|^2\|\Sigma\|^{3} \Delta^{2} + \|\Sigma\|\frac{\sigma^{2}r \log(4/\delta)}{n } ).    \notag \\
&\leq \mathcal{O}(\|\beta^{*}\|^{2}\|\Lambda^{*}\|^2\|\Sigma\|^{5} \lambda_{r}^{-2}(C^{* T}C^{*}) \frac{d \log (1/\delta)}{m} + \|\Sigma\|\frac{\sigma^{2}r \log(4/\delta)}{n } ).    \notag \\
\end{align}
Using the assumptions that $\|\beta^{*}\| \leq D$ and $\|B^{*}\| \leq D$, we can bound these terms by $D$ and quantities related to ground truth. First notice that
$\Sigma$ have eigenvalues $\sigma_{1}^{* 2}+1 \geq \sigma_{2}^{* 2}+1 \geq \cdots \geq \sigma_{r}^{* 2}+1 \geq 1 = \cdots = 1$, where $\sigma_{i}^{*}$ are singular values of $B^{*}$, therefore $\|\Sigma\| \leq D^{2}+1$. Also, since
\begin{align}
C^{* T}C^{*} &= (B^{*}B^{* T} + I_{d})^{-1} B^{*}B^{* T}   (B^{*}B^{* T} + I_{d})^{-1} \notag \\
&= (B^{*}B^{* T} + I_{d})^{-1} - (B^{*}B^{* T} + I_{d})^{-2} \notag \\
&= \Sigma^{-1} - \Sigma^{-2},
\end{align} 
we know that $C^{* T}C^{*}$ has $r$ nonzero eigenvalues $\{(\sigma_{i}^{*}+\sigma_{i}^{* -1})^{-2}\}_{i=1}^{r}$. Therefore $\|\Lambda^{*}\|^2 = \|C^{* T}C^{*}\| \leq 1/4$, 
\begin{align}
\lambda_{r}^{-2}(C^{* T}C^{*}) &\leq \max ((\sigma_{1}^{*}+\sigma_{1}^{* -1})^{4},(\sigma_{r}^{*}+\sigma_{r}^{* -1})^{4}) \notag \\
& \leq \mathcal{O} (D^{4} + \sigma_{r}^{* -4}).
\end{align}
For $\sigma^{2}$, we have
\begin{align}
\sigma^{2}&=\epsilon^{2}+\|\beta^{*}\|_{2}^{2}-\beta^{* T}B^{* T}(B^{*}B^{* T}+I_{d})^{-1}B^{*}\beta^{*} \notag \\
&\leq 1 +\|\beta^{*}\|^{2} \|I_{r}- B^{* T}(B^{*}B^{* T}+I_{d})^{-1}B^{*}\| \notag \\
&\leq 1 +D^{2}.
\end{align}
Combine all this bounds, we have
\begin{align}
{\rm Error}_{\ell}(\hat B,\hat\beta) &\leq \mathcal{O}(\|\beta^{*}\|^{2}\|\Lambda^{*}\|^2\|\Sigma\|^{5} \lambda_{r}^{-2}(C^{* T}C^{*}) \frac{d \log (1/\delta)}{m} + \|\Sigma\|\frac{\sigma^{2}r \log(4/\delta)}{n } ).    \notag \\
&\leq \mathcal{O} ((D^{2}+1)^{6}(D^{4}+\sigma_{min}^{* -4})\frac{d \log (1/\delta)}{m} + (D^{2}+1)^{2}\frac{r \log(4/\delta)}{n }).
\end{align}
\end{proof}

In the sequel, we give the proofs of Lemma \ref{bias_term}, \ref{variance_term} and \ref{Delta_lemma}. We first prove some additional technical lemmas. The following lemma, which is a simple corollary of \citet{tripuraneni2021provable} Lemma 20, shows the concentration property of empirical covariance matrix.
\begin{lemma} \label{concentration}
Let $\Sigma \in \mathbb{R}^{d}$ be a positive definite matrix. Let $\{x_{i}\}_{i=1}^{n}$ be $d-$dimensional Gaussian random vectors i.i.d. sample from $ N(0,\Sigma)$, $X=(x_{1},\cdots,x_{n})^{T}\in \mathbb{R}^{n \times d}$. Then for any $A,B \in \mathbb{R}^{d \times r}$, we have with probability at least $1-\delta$
\begin{align}
\|A^{T}(\frac{X^{T}X}{n})B - A^{T} \Sigma B \|_{2} \leq \mathcal{O} (\|A\| \|B\| \|\Sigma\| (\sqrt{\frac{r}{n}}+ \frac{r}{n} +\sqrt{ \frac{\log(1/\delta)}{n}}+ \frac{\log(1/\delta)}{n}).
\end{align}
\end{lemma}
\begin{proof}
We write the SVD of $A$ and $B$: $A=U_{1}\Lambda_{1}V_{1}^{T}$, $B=U_{2}\Lambda_{2}V_{2}^{T}$, where $U_{1}, U_{2} \in \mathbb{R}^{d \times r}$, $\Lambda_{1},\Lambda_{2},V_{1},V_{2} \in \mathbb{R}^{r \times r}$. Then 
\begin{align}
\|A^{T}(\frac{X^{T}X}{n})B - A^{T} \Sigma B \|_{2} &= \|V_{1}\Lambda_{1}U_{1}^{T}(\frac{X^{T}X}{n})U_{2}\Lambda_{2}V_{2}^{T} - V_{1}\Lambda_{1}U_{1}^{T}\Sigma U_{2}\Lambda_{2}V_{2}^{T} \|_{2}    \notag \\
&\leq \|V_{1}\Lambda_{1}\| \|U_{1}^{T}(\frac{X^{T}X}{n})U_{2}-U_{1}^{T}\Sigma U_{2}\| \|\Lambda_{2}V_{2}^{T}\| \notag \\
& \leq \|A\| \|B\| \|U_{1}^{T}(\frac{X^{T}X}{n})U_{2}-U_{1}^{T}\Sigma U_{2}\|.
\end{align}
Now since $U_{1},U_{2} \in \mathbb{R}^{d \times r}$ are projection matrices, we can apply \citet{tripuraneni2021provable} Lemma 20, therefore
\begin{align}
\|U_{1}^{T}(\frac{X^{T}X}{n})U_{2}-U_{1}^{T}\Sigma U_{2}\| \leq \mathcal{O} (\|\Sigma\| (\sqrt{\frac{r}{n}}+ \frac{r}{n} +\sqrt{ \frac{\log(1/\delta)}{n}}+ \frac{\log(1/\delta)}{n}))    
\end{align}
which gives what we want.
\end{proof}
The following lemma is a basic matrix perturbation result (see \citet{tripuraneni2021provable} Lemma 25).
\begin{lemma} \label{perturbation}
Let $A$ be a positive definite matrix and $E$ another matrix which satisfies $\|EA^{-1}\|\leq \frac{1}{4}$, then $F:=(A+E)^{-1}-A^{-1}$ satisfies $\|F\| \leq \frac{4}{3} \|A^{-1}\| \|EA^{-1}\|$.
\end{lemma}
With these two technical lemmas, we are able to prove Lemma \ref{bias_term}, \ref{variance_term}.
\begin{proof}[Proof of Lemma \ref{bias_term}] 
We consider $\hat{U} \in \mathbb{R}^{d \times r}$ and $\hat{U}_{\perp}^{T} \in \mathbb{R}^{d \times (d-r)}$ be orthonormal projection matrices spanning orthogonal subspaces which are rank $r$ and rank $d-r$ respectively, so that $ \operatorname{range}(\hat{U}) \oplus \operatorname{range}(\hat{U}_{\perp})=\mathbb{R}^d.$ Then $\Delta = dist (\hat{U},U^{*}) = \|\hat{U}_{\perp}^{T}U^{*}\|_{2}$ (see \citet{9584021} Lemma 2.5). Notice that $I_{d}=\hat{U}\hat{U}^{T} + \hat{U}_{\perp}\hat{U}_{\perp}^{T}$, we have 
\begin{align}
& \quad \hat{U}(\hat{U}^{T}X^{T}X\hat{U})^{-1}\hat{U}^{T}X^{T}X U^{*} - U^{*} \notag \\ 
&= \hat{U}(\hat{U}^{T}X^{T}X\hat{U})^{-1}\hat{U}^{T}X^{T}X (\hat{U}\hat{U}^{T} + \hat{U}_{\perp}\hat{U}_{\perp}^{T})U^{*} - U^{*} \notag \\
&= \hat{U}(\hat{U}^{T}X^{T}X\hat{U})^{-1}\hat{U}^{T}X^{T}X  \hat{U}\hat{U}^{T}U^{*} + \hat{U}(\hat{U}^{T}X^{T}X\hat{U})^{-1}\hat{U}^{T}X^{T}X\hat{U}_{\perp}\hat{U}_{\perp}^{T}U^{*} - U^{*} \notag \\
&=\hat{U}(\hat{U}^{T}X^{T}X\hat{U})^{-1}\hat{U}^{T}X^{T}X\hat{U}_{\perp}\hat{U}_{\perp}^{T}U^{*} + \hat{U}\hat{U}^{T}U^{*} - U^{*} \notag \\
&=\hat{U}(\hat{U}^{T}X^{T}X\hat{U})^{-1}\hat{U}^{T}X^{T}X\hat{U}_{\perp}\hat{U}_{\perp}^{T}U^{*} -\hat{U}_{\perp}\hat{U}_{\perp}^{T} U^{*} 
\end{align}
Therefore 
\begin{align} \label{decomposition2}
\|\hat{U}(\hat{U}^{T}X^{T}X\hat{U})^{-1}\hat{U}^{T}X^{T}X U^{*} - U^{*}\|_{2}^{2} \leq 2\|\hat{U}(\hat{U}^{T}X^{T}X\hat{U})^{-1}\hat{U}^{T}X^{T}X\hat{U}_{\perp}\hat{U}_{\perp}^{T}U^{*}\|_{2}^{2} + 2\|\hat{U}_{\perp}\hat{U}_{\perp}^{T} U^{*} \|_{2}^{2}.
\end{align}
For the second term, 
\begin{align} \label{second_term}
\|\hat{U}_{\perp}\hat{U}_{\perp}^{T} U^{*} \|_{2}^{2} \leq \|\hat{U}_{\perp}|^{2}   \|\hat{U}_{\perp}^{T} U^{*}\|^{2} \leq \Delta^{2}.
\end{align}
For the first term,
\begin{align} \label{decomposition1}
& \quad \|\hat{U}(\hat{U}^{T}X^{T}X\hat{U})^{-1}\hat{U}^{T}X^{T}X\hat{U}_{\perp}\hat{U}_{\perp}^{T}U^{*}\| \notag \\    
& = \|\hat{U}(\hat{U}^{T}\frac{X^{T}X}{n}\hat{U})^{-1}\hat{U}^{T}\frac{X^{T}X}{n}\hat{U}_{\perp}\hat{U}_{\perp}^{T}U^{*}\| \notag \\ 
& = \|\hat{U}((\hat{U}^{T}\Sigma\hat{U})^{-1}+F)(\hat{U}^{T}\Sigma\hat{U}_{\perp}\hat{U}_{\perp}^{T}U^{*}+E_{1})\| \notag \\ 
& \leq \|(\hat{U}^{T}\Sigma\hat{U})^{-1}(\hat{U}^{T}\Sigma\hat{U}_{\perp}\hat{U}_{\perp}^{T}U^{*})\|+\|(\hat{U}^{T}\Sigma\hat{U})^{-1}E_{1}\|+ \|F\hat{U}^{T}\Sigma\hat{U}_{\perp}\hat{U}_{\perp}^{T}U^{*}\| + \|FE_{1}\|,
\end{align}
where $E_{1}=\hat{U}^{T}\frac{X^{T}X}{n}\hat{U}_{\perp}\hat{U}_{\perp}^{T}U^{*} - \hat{U}^{T}\Sigma\hat{U}_{\perp}\hat{U}_{\perp}^{T}U^{*}$, $F =(\hat{U}^{T}\frac{X^{T}X}{n}\hat{U})^{-1} -(\hat{U}^{T}\Sigma\hat{U})^{-1}$. In order to bound $\|F\|$, let $E=\hat{U}^{T}\frac{X^{T}X}{n}\hat{U} -\hat{U}^{T}\Sigma\hat{U}$, then by Lemma \ref{concentration}, with probability at least $1-\delta$, \begin{align}
\|E\|\leq \mathcal{O}(\|\Sigma\|(\sqrt{\frac{r}{n}}+ \frac{r}{n} +\sqrt{ \frac{\log(1/\delta)}{n}}+ \frac{\log(1/\delta)}{n})).    
\end{align}
Therefore, since $\lambda_{min}(\Sigma)=1$,
\begin{align}
\|E(\hat{U}^{T}\Sigma\hat{U})^{-1}\| &\leq \|E\|  \|(\hat{U}^{T}\Sigma\hat{U})^{-1}\|   \notag \\
& \leq \|E\| \lambda_{min}(\Sigma)^{-1}\notag \\
& \leq \mathcal{O}(\|\Sigma\|(\sqrt{\frac{r}{n}}+ \frac{r}{n} +\sqrt{ \frac{\log(1/\delta)}{n}}+ \frac{\log(1/\delta)}{n}))
\end{align}
Notice that $n \gtrsim \|\Sigma\|^{2} r \log (1/\delta)$ implies $\sqrt{\frac{r}{n}}+ \frac{r}{n} +\sqrt{ \frac{\log(1/\delta)}{n}}+ \frac{\log(1/\delta)}{n}\lesssim \|\Sigma\|^{-1}$. Thus, we show that when $n$ is large enough, we have $\|E(\hat{U}^{T}\Sigma\hat{U})^{-1}\| \leq \frac{1}{4}$. Therefore we can apply Lemma \ref{perturbation}, which gives 
\begin{align} \label{F}
\|F\| &\leq \frac{4}{3} \|E(\hat{U}^{T}\Sigma\hat{U})^{-1}\|  \|(\hat{U}^{T}\Sigma\hat{U})^{-1}\|   \notag \\
&\leq \frac{4}{3} \times \frac{1}{4} \|(\hat{U}^{T}\Sigma\hat{U})^{-1}\| \notag \\
&\leq \frac{1}{3}.
\end{align}
As for $\|E_{1}\|$, directly applying Lemma \ref{concentration}, using $n \gtrsim \|\Sigma\|^{2} r \log (1/\delta)$, we get 
\begin{align} \label{E_1}
\|E_{1}\| &\leq \mathcal{O}(\|\Sigma\|\|\hat{U}_{\perp}\hat{U}_{\perp}^{T}U^{*}\|(\sqrt{\frac{r}{n}}+ \frac{r}{n} +\sqrt{ \frac{\log(1/\delta)}{n}}+ \frac{\log(1/\delta)}{n})) \notag \\
& \leq \mathcal{O}(\|\Sigma\|\Delta \|\Sigma\|^{-1}) \notag \\
& \leq \mathcal{O} (\Delta)
\end{align}
Combining (\ref{decomposition1}),(\ref{F})and(\ref{E_1}), we have 
\begin{align} \label{first_term}
& \quad \|\hat{U}(\hat{U}^{T}X^{T}X\hat{U})^{-1}\hat{U}^{T}X^{T}X\hat{U}_{\perp}\hat{U}_{\perp}^{T}U^{*}\| \notag \\
&\leq  \|(\hat{U}^{T}\Sigma\hat{U})^{-1}(\hat{U}^{T}\Sigma\hat{U}_{\perp}\hat{U}_{\perp}^{T}U^{*})\|+\|(\hat{U}^{T}\Sigma\hat{U})^{-1}E_{1}\|+ \|F\hat{U}^{T}\Sigma\hat{U}_{\perp}\hat{U}_{\perp}^{T}U^{*}\| + \|FE_{1}\|   \notag \\
&\leq  \|(\hat{U}^{T}\Sigma\hat{U})^{-1}\|\|(\hat{U}^{T}\Sigma\hat{U}_{\perp}\hat{U}_{\perp}^{T}U^{*})\|+\|(\hat{U}^{T}\Sigma\hat{U})^{-1}\|\|E_{1}\|+ \|F\|\|\hat{U}^{T}\Sigma\hat{U}_{\perp}\hat{U}_{\perp}^{T}U^{*}\| + \|F\|\|E_{1}\|   \notag \\
&\leq  \lambda_{min}(\Sigma)^{-1}\|\Sigma\|\|\hat{U}_{\perp}^{T}U^{*}\|+\lambda_{min}(\Sigma)^{-1}\|E_{1}\|+ \|F\|\|\Sigma\|\|\hat{U}_{\perp}^{T}U^{*}\| + \|F\|\|E_{1}\|   \notag \\
&\leq  \lambda_{min}(\Sigma)^{-1}\|\Sigma\| \Delta +\lambda_{min}(\Sigma)^{-1}\mathcal{O} (\lambda_{min}(\Sigma) \Delta)+ \frac{1}{3} \lambda_{min}(\Sigma)^{-1}\|\Sigma\|\Delta + \frac{1}{3} \lambda_{min}(\Sigma)^{-1}\mathcal{O} (\lambda_{min}(\Sigma) \Delta)   \notag \\
&\leq \mathcal{O} (\|\Sigma\| \Delta)
\end{align}
Finally, combining (\ref{decomposition2}),(\ref{second_term}) and (\ref{first_term}), we get 
\begin{align}
\|(\hat{U}(\hat{U}^{T}X^{T}X\hat{U})^{-1}\hat{U}^{T}X^{T}X U^{*} - U^{*})\|_{2}^{2} \leq \mathcal{O} (\|\Sigma\|^2\Delta^{2}),    
\end{align}
with probability at least $1-\delta$, which is what we want.
\end{proof}

\begin{proof}[Proof of Lemma \ref{variance_term}]
\begin{align}
\|\hat{U}(\hat{U}^{T}X^{T}X\hat{U})^{-1}\hat{U}^{T}X^{T} W\|_{2}^{2} &\leq \|(\hat{U}^{T}X^{T}X\hat{U})^{-1}\hat{U}^{T}X^{T} W\|_{2}^{2} \notag \\
&=((\hat{U}^{T}X^{T}X\hat{U})^{-1}\hat{U}^{T}X^{T} W)^{T}((\hat{U}^{T}X^{T}X\hat{U})^{-1}\hat{U}^{T}X^{T} W) \notag \\
&=W^{T} (\frac{1}{n}\frac{X\hat{U}}{\sqrt{n}} (\hat{U}^{T}\frac{X^{T}X}{n}\hat{U} )^{-2}\frac{\hat{U}^{T}X^{T}}{\sqrt{n}})W.
\end{align}
Let $A=\frac{1}{n}\frac{X\hat{U}}{\sqrt{n}} (\hat{U}^{T}\frac{X^{T}X}{n}\hat{U} )^{-2}\frac{\hat{U}^{T}X^{T}}{\sqrt{n}}$, $W=\sigma V$, then $V \sim N(0,I_{n})$. By Hanson-Wright inequality (see \citet{vershynin2018high} Theorem 6.2.1), 
\begin{align}
\mathbb{P}(|V^{T}AV-\mathbb{E}[V^{T}AV]| \geq t) \leq 2 \exp (-c \min(\frac{t^{2}}{\|A\|_{F}^{2}},\frac{t}{\|A\|_{2}})).   
\end{align}
Hence with probability at least $1-\delta$,
\begin{align}
V^{T}AV \leq \mathbb{E}[V^{T}AV] + \mathcal{O}(\|A\|_{F}\sqrt{\log\frac{2}{\delta}})  + \mathcal{O}(\|A\|_{2}\log\frac{2}{\delta}).
\end{align}
Notice that $\mathbb{E}[V^{T}AV]=\text{Tr}(A)$, therefore it remains to bound $\text{Tr}(A)$, $\|A\|_{F}$ and $\|A\|_{2}$. If we define $B=\frac{X\hat{U}}{\sqrt{n}} \in \mathbb{R}^{n \times r}$, then $A=\frac{1}{n}B(B^{T}B)^{-2}B^{T}$. Therefore 
\begin{align}
\text{Tr}(A) &= \text{Tr}(\frac{1}{n}B(B^{T}B)^{-2}B^{T}) \notag \\   
&= \frac{1}{n}\text{Tr}((B^{T}B)^{-2}B^{T}B) \notag \\
&= \frac{1}{n}\text{Tr}((B^{T}B)^{-1}) \notag \\
& \leq \frac{r}{n} \|(B^{T}B)^{-1}\|_{2}
\end{align}
Let the SVD of $B$ be $B=PMQ^{T}$, where $P \in \mathbb{R}^{n \times r}$, $M,Q \in \mathbb{R}^{r \times r}$, then
\begin{align}
\|A\|_{2} &=   \frac{1}{n} \|B(B^{T}B)^{-2}B^{T}\|_{2} \notag \\  
&=   \frac{1}{n} \|PMQ^{T}(QM^{2}Q^{T})^{-2}QMP^{T}\|_{2} \notag \\ 
&=   \frac{1}{n} \|PM^{-2}P^{T}\|_{2} \notag \\ 
& \leq \frac{1}{n}  \|M^{-2}\|_{2} \notag \\
&= \frac{1}{n} \|(B^{T}B)^{-1}\|_{2}
\end{align}
Also notice that $A$ is rank $r$, therefore $\|A\|_{F} \leq \sqrt{r} \|A\|_{2}$. Thus it remains to bound $\|(B^{T}B)^{-1}\|_{2} = \|(\hat{U}^{T}\frac{X^{T}X}{n}\hat{U} )^{-1}\|_{2}$. Let $F =(\hat{U}^{T}\frac{X^{T}X}{n}\hat{U})^{-1} -(\hat{U}^{T}\Sigma\hat{U})^{-1}$. Recall (\ref{F}), which states that with probability at least $1-\delta$, we have 
$\|F\|\leq \frac{1}{3}\lambda_{min}(\Sigma)^{-1}$. Therefore 
\begin{align}
\|(\hat{U}^{T}\frac{X^{T}X}{n}\hat{U} )^{-1}\|  &= \|(\hat{U}^{T}\Sigma \hat{U})^{-1}+F\|  \notag \\
&\leq \|(\hat{U}^{T}\Sigma \hat{U})^{-1}\|+\|F\|  \notag \\
&\leq \mathcal{O}(\lambda_{min}(\Sigma)^{-1}). 
\end{align}
Thus $\|A\| \leq \mathcal{O}(\frac{1}{n}\lambda_{min}(\Sigma)^{-1})$, $\|A\|_{F} \leq \mathcal{O}(\frac{\sqrt{r}}{n}\lambda_{min}(\Sigma)^{-1})$, 
$\text{Tr}(A) \leq \mathcal{O}(\frac{r}{n}\lambda_{min}(\Sigma)^{-1})$. Therefore with probability at least $1-2\delta$,
\begin{align}
V^{T}AV &\leq \mathbb{E}[V^{T}AV] + \mathcal{O}(\|A\|_{F}\sqrt{\log\frac{2}{\delta}})  + \mathcal{O}(\|A\|_{2}\log\frac{2}{\delta})  \notag \\
& \leq \mathcal{O}(\frac{r}{n}\lambda_{min}(\Sigma)^{-1}) + \mathcal{O}(\frac{\sqrt{r}}{n}\lambda_{min}(\Sigma)^{-1}\sqrt{\log \frac{2}{\delta}}) + \mathcal{O}(\frac{1}{n}\lambda_{min}(\Sigma)^{-1}\log \frac{2}{\delta}) \notag \\
& \leq \mathcal{O}(\frac{r}{n}\lambda_{min}(\Sigma)^{-1}\log \frac{2}{\delta}) \notag \\
&= \mathcal{O}(\frac{r}{n}\log \frac{2}{\delta}).
\end{align}
The last line holds since $\lambda_{min}(\Sigma)=1$.
Recall 
\begin{align}
\|\hat{U}(\hat{U}^{T}X^{T}X\hat{U})^{-1}\hat{U}^{T}X^{T} W\|_{2}^{2} = W^{T}AW =\sigma^{2} V^{T}AV,
\end{align}
combining this with the above bound for $V^{T}AV$ yields our desired result.
\end{proof}

Finally we prove Lemma \ref{Delta_lemma} in the following.
\begin{proof}[Proof of Lemma \ref{Delta_lemma}]
In the first step, we have $m$ unlabeled data $\{x_{i}\}_{i=1}^{m}$ i.i.d. sample from $N (0,\Sigma)$. Let $\hat{\Sigma}=\frac{1}{m}\sum_{i=1}^{m}x_{i}x_{i}^{T}$ be the empirical covariance matrix. Then by Lemma \ref{concentration}, with probability at least $1-\delta$, 
\begin{align}
\|\Sigma- \hat{\Sigma}\| \leq \mathcal{O}(\|\Sigma\|(\sqrt{\frac{d}{m}}+\frac{d}{m}+ \sqrt{\frac{\log (1/\delta)}{m}} + \frac{\log (1/\delta)}{m}))    
\end{align}
We claim that
\begin{align}\label{claim_rankk_approx}
\|\hat{B}\hat{B}^{T} - (\hat{\Sigma}-I_{d})\|_{2} \leq \|\hat{\Sigma}-\Sigma\|,   
\end{align}
and the proof of this claim will be at the end of this section.
With the claim,
\begin{align}
\|\hat{B}\hat{B}^{T} - B^{*}B^{* T}\| &= \|\hat{B}\hat{B}^{T} - (\hat{\Sigma}-I_{d}) + (\hat{\Sigma}-I_{d}) - (\Sigma - I_{d})\|  \notag \\
& \leq \|\hat{B}\hat{B}^{T} - (\hat{\Sigma}-I_{d})\| + \|\Sigma - \hat{\Sigma}\| \notag \\
& \leq 2\|\Sigma-\hat{\Sigma}\|.
\end{align}
Notice that
\begin{align}
C^{* T}C^{*} &= (B^{*}B^{* T} + I_{d})^{-1} B^{*}B^{* T}   (B^{*}B^{* T} + I_{d})^{-1} \notag \\
&= (B^{*}B^{* T} + I_{d})^{-1} - (B^{*}B^{* T} + I_{d})^{-2}
\end{align}
Similarly
\begin{align}
\hat{C}^{T}\hat{C} = (\hat{B}\hat{B}^{T} + I_{d})^{-1} - (\hat{B}\hat{B}^{T} + I_{d})^{-2}.
\end{align}
Let $E_{2}=(\hat{B}\hat{B}^{T} + I_{d}) - (B^{*}B^{* T} + I_{d})$, $F_{2} = (\hat{B}\hat{B}^{T} + I_{d})^{-1} - (B^{*}B^{* T} + I_{d})^{-1}$. Then 
\begin{align}
\|E_{2}\| \leq 2\|\Sigma-\hat{\Sigma}\| \leq \mathcal{O}(\|\Sigma\|(\sqrt{\frac{d}{m}}+\frac{d}{m}+ \sqrt{\frac{\log (1/\delta)}{m}} + \frac{\log (1/\delta)}{m})).
\end{align}
Therefore when $m \gtrsim \|\Sigma\|^{2} d \log (1/\delta)$, $\|E_{2}\| \leq \mathcal{O}(\|\Sigma\|\sqrt{\frac{d \log (1/\delta)}{m}})$, $\|E_{2}\Sigma^{-1}\| \leq \|E_{2}\|\|\Sigma^{-1}\| \leq 1/4$. Then we can apply Lemma \ref{perturbation}, which gives 
\begin{align}
\|F_{2}\| &\leq \frac{4}{3}    \|\Sigma^{-1}\| \|E_{2}\Sigma^{-1}\| \notag \\
& \leq \frac{4}{3}    \|\Sigma^{-1}\|^{2} \|E_{2}\| \notag \\
&\leq \mathcal{O}(\lambda_{min}^{-2}(\Sigma)\|\Sigma\|\sqrt{\frac{d \log (1/\delta)}{m}}) \notag \\
&= \mathcal{O}(\|\Sigma\|\sqrt{\frac{d \log (1/\delta)}{m}}).
\end{align}
The last line holds since $\lambda_{min}(\Sigma)=1$.
Thus 
\begin{align}
\|C^{* T}C^{*}- \hat{C}^{T}\hat{C}\| &= \|(\Sigma^{-1}+F_{2})-(\Sigma^{-1}+F_{2})^{2} - (\Sigma^{-1}-\Sigma^{-2})\| \notag \\
&= \|F_{2}- \Sigma^{-1}F_{2}-F_{2}\Sigma^{-1} - F_{2}^{2}\| \notag \\
& \leq \|F_{2}\| + 2\|\Sigma^{-1}\|\|F_{2}\| + \|F_{2}\|^{2} \notag \\
& \leq \mathcal{O}(\|\Sigma\|\sqrt{\frac{d \log (1/\delta)}{m}}).
\end{align}
Therefore by Davis-Kahan theorem, 
\begin{align}
\Delta = dist (U^{*},\hat{U}) &\leq \mathcal{O} (\lambda_{r}^{-1}(C^{* T}C^{*})\|C^{* T}C^{*}- \hat{C}^{T}\hat{C}\|).
\end{align}
Combining the above three inequalities, we have 
\begin{align}
\Delta^{2} &\leq \mathcal{O} (\|\Sigma\|^{2}\frac{d \log (1/\delta)}{m} \lambda_{r}^{-2}(C^{* T}C^{*})). 
\end{align}
Finally we will need to prove the claim (\ref{claim_rankk_approx}).
Notice that the MLE estimator $\hat{B}$ is given by 
\begin{align}
\hat{B} &= \argmax_{B\in \mathbb{R}^{d \times r}} \sum_{i=1}^{m}p_{B}(x_{i})    \notag \\
& =  \argmax_{B\in \mathbb{R}^{d \times r}} (-\log \det (BB^{T}+I_{d})-\text{Tr}(\hat{\Sigma}(BB^{T}+I_{d})^{-1})) \notag \\
& = \argmin_{B\in \mathbb{R}^{d \times r}} (\log \det (BB^{T}+I_{d})+\text{Tr}(\hat{\Sigma}(BB^{T}+I_{d})^{-1}))
\end{align}
Let $\hat{\Sigma}=\hat U\hat{\Lambda}\hat{U}^{T}$ and  $(BB^{T}+I_{d})=U \Lambda U^{T}$, where $\hat U$ and $U$ are orthogonal matrices, $\hat{\Lambda}=\text{diag}(\hat{\lambda}_{1},\cdots, \hat{\lambda}_{d})$, $\Lambda=\text{diag}(\lambda_{1},\cdots, \lambda_{d})$ and $\hat\lambda_1\geq\ldots\geq\hat\lambda_d$, $\lambda_1\geq\ldots\geq\lambda_d$. Since $\operatorname{rank}(BB^T)\leq r$, we have $\lambda_{r+1}= \ldots \lambda_{d}=1$. By Ruhe’s trace inequality (see P341 of \citet{marshall11}), we have 
\begin{align}
\text{Tr}(\hat{\Sigma}(BB^{T}+I_{d})^{-1})) \geq \  \sum_{j=1}^{d}   \lambda_{j}^{-1}\hat{\lambda}_{j}, 
\end{align}
and the equality holds only when the two matrices have simultaneous ordered spectral decomposition, i.e., $U=\hat{U}$. Therefore
\begin{align}
& \quad \min_{B\in \mathbb{R}^{d \times r}} (\log \det (BB^{T}+I_{d})+\text{Tr}(\hat{\Sigma}(BB^{T}+I_{d})^{-1})) \notag \\
&= \min_{\{\lambda_{j}\}_{j=1}^{d}} \sum_{j=1}^{d}( \log \lambda_{j} + \lambda_{j}^{-1}\hat{\lambda}_{j} ) \quad \text{subject to } \lambda_{1} \geq  \cdots \geq \lambda_{r} \geq \lambda_{r+1}=\cdots =\lambda_{d}=1
\end{align}
and the minimum is achieved when $\lambda_{j}=\hat{\lambda}_{j}$, for $j=1, \cdots, r$. Therefore the MLE estimator $\hat{B}$ satisfies $(\hat{B}\hat{B}^{T}+I_{d})=\hat U \Lambda \hat{U}^{T}$ 
where $\Lambda=\text{diag}(\hat{\lambda}_{1},\cdots, \hat{\lambda}_{r},1,\cdots,1)$. Thus, we have $\hat{B}\hat{B}^{T}=\hat U (\Lambda-I_{d}) \hat{U}^{T}$, which implies
\begin{align}
& \quad \|\hat{B}\hat{B}^{T} - (\hat{\Sigma}-I_{d})\|_{2} \notag \\
&= \|\hat U (\Lambda-I_{d}) \hat{U}^{T} - \hat U (\hat{\Lambda}-I_{d}) \hat{U}^{T}\|  \notag \\
& \leq \|\Lambda-\hat{\Lambda}\| \notag \\
& = \max_{j=r+1,\cdots, d} |\hat{\lambda}_{j}-1|\notag \\
& \leq \max_{j=1,\cdots, d} |\hat{\lambda}_{j}-\lambda_{j}(\Sigma)|\notag \\
&\leq \|\hat{\Sigma}-\Sigma\|.
\end{align}
Here the last inequality follows from Weyl's Theorem. Thus, we prove claim (\ref{claim_rankk_approx}).
\end{proof}
\section{Proofs for Section \ref{gmm}}
In Section \ref{gmm1}, we show that GMM with classification as downstream tasks has $c_2$-transferability for some absolute constants $c_2$ (Lemma \ref{gmm_ti}). In Section \ref{gmm2} and Section \ref{gmm3}, we prove two lemmas that will be used in the proof of Theorem \ref{gmm_main}. To be specific, in Section \ref{gmm2}, we upper bound the bracketing number of the set $\mP(\mathcal{U})$ by using $\epsilon$-discretization (Lemma \ref{gmm_bn}). In Section \ref{gmm3}, we prove Lemma \ref{gmm_rc}, which will be used to upper bound the Rademacher complexity of the function class $\ell\circ\mathcal{G}_{\bu,\Psi}$. Finally, in Section \ref{gmm4}, we prove Theorem \ref{gmm_main}.

\subsection{Proofs for Lemma \ref{gmm_ti}}\label{gmm1}
Before going to the proof of this theorem, we first state some basic definitions and useful lemmas. We define the balls of radius $8\sqrt{d\log K}$ around each $u_{i}^{*}$ and $u_{i}$ as
\%
\Omega_{i}^{*}:=\left\{ x \in \R^{d} \mid \|x-u_{i}^{*}\| \leq 8 \sqrt{d\log K}\right\}
\%
\%
\Omega_{i}:=\left\{ x \in \R^{d} \mid \|x-u_{i}\| \leq 8 \sqrt{d\log K}\right\}
\%
We denote the p.d.f of $\mathcal{N}\left(u_i, I_d\right)$ and $\mathcal{N}\left(u_i^{*}, I_d\right)$ by $P_{i}$ and $P_{i}^{*}$ respectively.

\begin{lemma} \label{initial_match}
If 
\%
d_{\mathrm{TV}}\left(p_{\mathbf{u}}(x), p_{\mathbf{u}^* }(x)\right) \leq \frac{1}{4K},
\%
then there exists a permutation of $\boldsymbol{u}$ such that $\|u_{i}^{*}-u_{i}\| \leq 16\sqrt{d\log{K}}$ holds for every $1\leq i \leq K$.
\end{lemma}

Before proving Lemma \ref{initial_match}, we first state a useful result of Gaussian norm concentration.
\begin{lemma} \label{norm_subGaussian}
Let $X \sim \mathcal{N}\left(0, I_d\right) $, then 
\%
\P (\|X\| \geq t) \leq 2\exp(-\frac{t^{2}}{16d}).
\%
\end{lemma}
\begin{proof}
This is a simple application of \cite{jin2019short} Lemma 1.3. Notice that $X$ is $1$-subGaussian, therefore taking $\sigma=\sqrt{d}$ in \cite{jin2019short} Lemma 1.3 yields what we want.
\end{proof}

\begin{proof}[Proof of Lemma \ref{initial_match}]
We prove by contradiction. If the statement is not true, since the separation satisfies $100\sqrt{d\log{K}} \geq 2 \cdot 16\sqrt{d\log{K}},$
there must exist a $u_{i}^{*}$ (W.L.O.G., denote it by $u_{1}^{*}$), such that $\|u_{1}^{*}-u_{j}\| > 16\sqrt{d\log{K}} $
for any $1 \leq j \leq K$.  Then 
\begin{align}
2 d_{\mathrm{TV}}\left(p_{\mathbf{u}}(x), p_{\mathbf{u}^* }(x)\right) 
&= \int_{\R^{d}} \bigg|\frac{1}{K}\sum_{i=1}^{K} P_{i}^{*}- \frac{1}{K}\sum_{i=1}^{K} P_{i}  \bigg| \mathrm{d}x \notag\\
& \geq \int_{\Omega_{1}^{*}} \bigg|\frac{1}{K}\sum_{i=1}^{K} P_{i}^{*}- \frac{1}{K}\sum_{i=1}^{K} P_{i}  \bigg| \mathrm{d}x \notag\\
& \geq \int_{\Omega_{1}^{*}} \frac{1}{K}\sum_{i=1}^{K} P_{i}^{*}\mathrm{d}x - \int_{\Omega_{1}^{*}} \frac{1}{K}\sum_{i=1}^{K} P_{i}\mathrm{d}x \notag\\
& \geq \int_{\Omega_{1}^{*}} \frac{1}{K} P_{1}^{*} \mathrm{d}x - \frac{1}{K}\sum_{i=1}^{K} \int_{\Omega_{1}^{*}}  P_{i}\mathrm{d}x \notag\\
& = \frac{1}{K} \P (\mathcal{N}\left(u_1^{*}, I_d\right) \in \Omega_{1}^{*}) - \frac{1}{K}\sum_{i=1}^{K} \P (\mathcal{N}\left(u_i, I_d\right) \in \Omega_{1}^{*})
\end{align}
Since $\|u_{1}^{*}-u_{i}\| > 16\sqrt{d\log{K}} $, therefore $\Omega_{1}^{*} \cap \Omega_{i} = \emptyset$, which implies (by Lemma \ref{norm_subGaussian})
\%
\P (\mathcal{N}\left(u_i, I_d\right) \in \Omega_{1}^{*}) \leq \P (\mathcal{N}\left(u_i, I_d\right) \in \Omega_{i}^{\mathbf{C}}) \leq 2\exp (-\frac{(8\sqrt{d\log K})^{2}}{16d}) = 2e^{-4\log K}
\%
Also, by Lemma \ref{norm_subGaussian},
\%
\P (\mathcal{N}\left(u_1^{*}, I_d\right) \in \Omega_{1}^{*}) \geq 1- 2\exp (-\frac{(8\sqrt{d\log K})^{2}}{16d}) = 1- 2e^{-4\log K}
\%
Therefore,
\begin{align}
2 d_{\mathrm{TV}}\left(p_{\mathbf{u}}(x), p_{\mathbf{u}^* }(x)\right)
& \geq \frac{1}{K} \P (\mathcal{N}\left(u_1^{*}, I_d\right) \in \Omega_{1}^{*}) - \frac{1}{K}\sum_{i=1}^{K} \P (\mathcal{N}\left(u_i, I_d\right) \in \Omega_{1}^{*}) \notag\\
& \geq \frac{1}{K} (1- 2e^{-4 \log K}) - \frac{1}{K}\sum_{i=1}^{K} 2e^{-4 \log K} \notag\\
& = \frac{1}{K} - (2+\frac{2}{K})e^{-4\log K} \notag\\
& \geq \frac{1}{K} - 3e^{-4\log K} \notag\\
& = \frac{1}{K} - 3(\frac{1}{K})^{4} \notag\\
& = \frac{1}{2K}
\end{align}
which is a contradiction.
\end{proof}

We then state the core lemmas of proving Lemma \ref{gmm_ti}. 

\begin{lemma} \label{lowerbound}
If for any i, $\|u_i-u_i^{*}\| \leq 16\sqrt{d \log K}$, then for $\Omega_{1}^{*}$ (corresponding results hold for each $\Omega_{i}^{*}$),
\%
\int_{\Omega_{1}^{*}}|P_{1}^{*}-P_{1}| \mathrm{d}x \geq c_{1} \min \left\{ \|u_{1}^{*}-u_{1}\|, 1 \right\},
\% where $c_{1}= \frac{1}{200}$.
\end{lemma}

\begin{lemma} \label{upperbound}
If for any i, $\|u_i-u_i^{*}\| \leq 16\sqrt{d \log K}$, then for $\Omega_{1}^{*}$ (corresponding results hold for each $\Omega_{i}^{*}$), then for every $j \neq 1$, 
\%
\int_{\Omega_{1}^{*}}|P_{j}^{*}-P_{j}| \mathrm{d}x \leq \frac{c_{2}}{K} \min \left\{ \|u_{j}^{*}-u_{j}\|, 1 \right\},
\% where $c_{2}= 2688 \left(\frac{1}{2}\right)^{69}.  $

\end{lemma}

With these lemmas, we are now able to prove Lemma \ref{gmm_ti}.
\begin{proof}[Proof of Lemma \ref{gmm_ti}]
By Lemma \ref{initial_match}, there exists a permutation of $\boldsymbol{u}$ such that $\|u_{i}^{*}-u_{i}\| \leq 16\sqrt{d\log{K}}$ holds for every $1\leq i \leq K$. Therefore Lemma \ref{lowerbound}, \ref{upperbound} can be applied. Notice that
\begin{align}
\int_{\Omega_{1}^{*}} |p_{\mathbf{u}}(x)- p_{\mathbf{u}^* }(x)| \mathrm{d}x 
&= \int_{\Omega_{1}^{*}} \bigg|\frac{1}{K}\sum_{i=1}^{K} P_{i}^{*}- \frac{1}{K}\sum_{i=1}^{K} P_{i}  \bigg| \mathrm{d}x \notag \\
& \geq \int_{\Omega_{1}^{*}} \bigg|\frac{1}{K} P_{1}^{*}- \frac{1}{K} P_{i}  \bigg| \mathrm{d}x -\int_{\Omega_{1}^{*}} \bigg|\frac{1}{K}\sum_{i=2}^{K} P_{i}^{*}- \frac{1}{K}\sum_{i=2}^{K} P_{i}  \bigg| \mathrm{d}x \notag \\
& \geq \frac{1}{K} \int_{\Omega_{1}^{*}} | P_{1}^{*}-  P_{i}  | \mathrm{d}x -\frac{1}{K}\sum_{i=2}^{K} \int_{\Omega_{1}^{*}} | P_{i}^{*}- P_{i}  | \mathrm{d}x \notag\\
& \geq \frac{c_{1}}{K} \min \left\{ \|u_{1}^{*}-u_{1}\|, 1 \right\} - \frac{c_{2}}{K^{2}} \sum_{i=2}^{K} \min \left\{ \|u_{i}^{*}-u_{i}\|, 1 \right\},
\end{align}
where the last line comes from Lemma \ref{lowerbound}, \ref{upperbound}. \\

Sum up all the equations above for corresponding $1 \leq i \leq K$, since $\{\Omega_{i}^{*}\}_{i=1}^{K}$ are disjoint, we have
\begin{align}
d_{\mathrm{TV}}\left(p_{\mathbf{u}}(x), p_{\mathbf{u}^* }(x)\right) &= \frac{1}{2} \int_{\R^{d}} |p_{\mathbf{u}}(x)- p_{\mathbf{u}^* }(x)| \mathrm{d}x \notag\\
&\geq \frac{1}{2}\sum_{i=1}^{K} \int_{\Omega_{i}^{*}} |p_{\mathbf{u}}(x)- p_{\mathbf{u}^* }(x)| \mathrm{d}x \notag \\
&\geq \frac{1}{2}\left( \frac{c_1}{K}- \frac{(K-1)c_{2}}{K^{2}} \right) \sum_{i=1}^{K} \min \left\{ \|u_{i}^{*}-u_{i}\|, 1 \right\} \notag \\
&\geq \frac{1}{2} \left( c_{1}- c_{2} \right) \cdot \frac{1}{K} \sum_{i=1}^{K} \min \left\{ \|u_{i}^{*}-u_{i}\|, 1 \right\} \notag \\
&=\frac{1}{2}\left(\frac{1}{200} - 2688 \left(\frac{1}{2}\right)^{69}\right) \cdot \frac{1}{K} \sum_{i=1}^{K} \min \left\{ \|u_{i}^{*}-u_{i}\|, 1 \right\} \notag \\
&\geq \frac{1}{500} \cdot \frac{1}{K} \sum_{i=1}^{K} \min \left\{ \|u_{i}^{*}-u_{i}\|, 1 \right\}.
\end{align}
In the end, we refer to Lemma \ref{tv_norm}, which states that 
\%
d_{\mathrm{TV}}(\mathcal{N}(u_{i}^{*}, I_{d}), \mathcal{N}(u_{i}, I_{d})) \leq \min ( \|u_{i}^{*}-u_{i}\|, 1).
\%
Take $\sigma(\boldsymbol{u})= \left\{u_i\right\}_{i=1}^K$,
\begin{align}
d_{\mathrm{TV}}\left(p_{\sigma(\mathbf{u})}(x, z), p_{\mathbf{u}^* }(x, z)\right) &= \sum_{i=1}^{K} \P (z=i) d_{\mathrm{TV}} (\mathcal{N}(u_{i}^{*},I_{d}),\mathcal{N}(u_{i},I_{d})) \notag \\
&\leq  \sum_{i=1}^{K} \frac{1}{K} \min ( \|u_{i}^{*}-u_{i}\|, 1) \notag \\
&\leq 500 d_{\mathrm{TV}}\left(p_{\mathbf{u}}(x), p_{\mathbf{u}^* }(x)\right).
\end{align}
\end{proof}
Finally we state the proof of Lemma \ref{lowerbound} and \ref{upperbound}. 

\begin{proof}[Proof of Lemma \ref{lowerbound}]
W.L.O.G.,let $u_{1}^{*}=0$, $\Delta:=\|u_{1}\| \leq 16\sqrt{d \log K}$, and $u_{1}=(-\Delta, 0, 0, \cdots, 0)$. The densities are given by
\%
P_{1}^{*}(x)=(\frac{1}{\sqrt{2\pi}})^{d} e^{-\frac{1}{2}\|x\|^{2}}
\%
\%
P_{1}(x)=(\frac{1}{\sqrt{2\pi}})^{d} e^{-\frac{1}{2}\|x-u_{1}\|^{2}}
\%
We consider an area $S \subset \Omega_{1}^{*}$:
\%
S:=\left\{ x=(x_{1},\cdots,x_{d}) \bigg| x \in \Omega_{1}^{*}, x_{1} \geq \frac{1}{10} \right\}
\%
Then for any $x \in S$, $\|x\|^{2} \leq \|x-u_{1}\|^{2}$, which implies $P_{1}^{*}(x) \geq P_{1}(x)$. Therefore
\begin{align}
\int_{\Omega_{1}^{*}} |P_{1}^{*}-P_{1}| \mathrm{d}x &\geq \int_{S}|P_{1}^{*}-P_{1}| \mathrm{d}x \notag \\
&= \int_{S} (\frac{1}{\sqrt{2\pi}})^{d} \left(e^{-\frac{1}{2}\|x\|^{2}} - e^{-\frac{1}{2}\|x-u_{1}\|^{2}} \right)\mathrm{d}x \notag\\
&=\int_{S} (\frac{1}{\sqrt{2\pi}})^{d} e^{-\frac{1}{2}\|x\|^{2}} \left(1- e^{\frac{1}{2}\|x\|^{2}-\frac{1}{2}\|x-u_{1}\|^{2}} \right) \mathrm{d}x \notag\\
&\geq \min_{x \in S} \left(1- e^{\frac{1}{2}\|x\|^{2}-\frac{1}{2}\|x-u_{1}\|^{2}} \right) \int_{S} (\frac{1}{\sqrt{2\pi}})^{d} e^{-\frac{1}{2}\|x\|^{2}} \mathrm{d}x \notag\\
&= \min_{x \in S} \left(1- e^{\frac{1}{2}\|x\|^{2}-\frac{1}{2}\|x-u_{1}\|^{2}} \right) \P(\mathcal{N}\left(0, I_d\right)\in S) 
\end{align}
For $\min_{x \in S} \left(1- e^{\frac{1}{2}\|x\|^{2}-\frac{1}{2}\|x-u_{1}\|^{2}} \right)$, notice that for any $x=(x_{1},\cdots,x_{d}) \in S$, 
\%
\frac{1}{2}\|x\|^{2}-\frac{1}{2}\|x-u_{1}\|^{2}=-x_1 \Delta - \frac{1}{2}\Delta^{2} \leq -\frac{1}{10}\Delta
\%
Thus 
\%
\min_{x \in S} \left(1- e^{\frac{1}{2}\|x\|^{2}-\frac{1}{2}\|x-u_{1}\|^{2}} \right) \geq 1-e^{-\frac{1}{10}\Delta}
\%
Take $c_{3}= \frac{1}{20}$. We claim that 
\%
\min_{x \in S} \left(1- e^{\frac{1}{2}\|x\|^{2}-\frac{1}{2}\|x-u_{1}\|^{2}} \right) \geq c_{3} \min\{\Delta,1\}.
\%
In fact, when $0 \leq \Delta \leq 1$,
\%
\min_{x \in S} \left(1- e^{\frac{1}{2}\|x\|^{2}-\frac{1}{2}\|x-u_{1}\|^{2}} \right) \geq 1-e^{-\frac{1}{10}\Delta} \geq \frac{1}{20} \Delta.
\%
The last inequality holds, since if we let $f(x)=e^{-\frac{1}{10}x}+\frac{1}{20}x-1$, Then $f(0)=0$,
\%
f'(x)=-\frac{1}{10}e^{-\frac{1}{10}x}+\frac{1}{20} \leq 0
\%
for any $x \in [0,10\log 2].$ Thus for any $\Delta\in [0,1]$, 
\%
e^{-\frac{1}{10}\Delta}+\frac{1}{20}\Delta-1 = f(\Delta) \leq f(0)=0.
\%
When $1\leq \Delta \leq 16\sqrt{d \log K}$, 
\%
\min_{x \in S} \left(1- e^{\frac{1}{2}\|x\|^{2}-\frac{1}{2}\|x-u_{1}\|^{2}} \right) \geq 1-e^{-\frac{1}{10}\Delta} \geq 1-e^{-\frac{1}{10}} \geq \frac{1}{20} \cdot 1
\%
Therefore we have shown that \%
\min_{x \in S} \left(1- e^{\frac{1}{2}\|x\|^{2}-\frac{1}{2}\|x-u_{1}\|^{2}} \right) \geq c_{3} \min\{\Delta,1\}.
\%
where $c_{3}=\frac{1}{20}$.\\\\
As for $\P(\mathcal{N}\left(0, I_d\right)\in S)$, take 
\%
S':=\left\{ x=(x_{1},\cdots,x_{d}) \bigg| 2\sqrt{d\log2} \geq x_{1} \geq \frac{1}{10}, x_{2}^{2}+ \cdots + x_{d}^{2} \leq 60d \log K \right\}.
\%
Then $S' \subset S$. Therefore 
\begin{align}
\P(\mathcal{N}\left(0, I_d\right)\in S) &\geq \P(\mathcal{N}\left(0, I_d\right)\in S') \notag \\
&= \P(2\sqrt{d\log2} \geq x_{1} \geq \frac{1}{10},x_{2}^{2}+ \cdots + x_{d}^{2} \leq 60d \log K,x \sim \mathcal{N}\left(0, I_d\right))\notag \\
&=\P\left(2\sqrt{d\log2} \geq \mathcal{N}(0, 1)\geq \frac{1}{10}\right) \P\left(\left\|\mathcal{N}(0, I_{d-1})\right\|^{2} \leq 60d \log K\right) \notag \\
& \geq  \P\left(2\sqrt{\log2} \geq \mathcal{N}(0, 1)\geq \frac{1}{10}\right) \P\left(\left\|\mathcal{N}(0, I_{d-1})\right\|^{2} \leq 60(d-1) \log 2\right) \notag \\
& > \P\left(2\sqrt{\log2} \geq \mathcal{N}(0, 1)\geq \frac{1}{10}\right) \cdot (1-2e^{-2}) \quad \text{(by Lemma \ref{norm_subGaussian})}\notag \\
& > \frac{1}{4} \cdot (1-2e^{-2}) \quad \notag \\
&> \frac{1}{10}
\end{align}
Combine all these results, we have
\begin{align}
\int_{\Omega_{1}^{*}} |P_{1}^{*}-P_{1}| \mathrm{d}x  &\geq \min_{x \in S} \left(1- e^{\frac{1}{2}\|x\|^{2}-\frac{1}{2}\|x-u_{1}\|^{2}} \right) \P(\mathcal{N}\left(0, I_d\right)\in S) \notag \\
&\geq c_{3} \min\{\Delta,1\} \cdot \frac{1}{10} \notag\\
&=\frac{1}{200} \min\{\|u_{1}^{*}-u_{1}\|,1\} 
\end{align}
\end{proof}

\begin{proof}[Proof of Lemma \ref{upperbound}]
For any $i \neq 1$, 
\%
\int_{\Omega_{1}^{*}} |P_{i}^{*}-P_{i}| \mathrm{d}x = \int_{\Omega_{1}^{*}} (\frac{1}{\sqrt{2\pi}})^{d} | e^{-\frac{1}{2}\|x-u_{i}^{*}\|^{2}} - e^{-\frac{1}{2}\|x-u_{i}\|^{2}}  | \mathrm{d}x.
\%
Notice that if we denote $a(x):= \|x-u_{i}^{*}\|$, $\delta(x):= \|x-u_{i}^{*}\|- \|x-u_{i}\|$ ,$\Delta:=\|u_{i}-u_{i}^{*}\|$, then $|\delta(x)| \leq \Delta \leq 16\sqrt{d\log K}$, and for any $x\in \Omega_{1}^{*}$, $a(x) \geq 92\sqrt{d \log K}$ (due to separation condition). Therefore
\begin{align}
& \max_{x \in \Omega_{1}^{*}} \bigg| e^{-\frac{1}{2}\|x-u_{i}^{*}\|^{2}} - e^{-\frac{1}{2}\|x-u_{i}\|^{2}}  \bigg| \notag\\
&= \max_{x \in \Omega_{1}^{*}} \bigg| e^{-\frac{1}{2}a(x)^{2}} -  e^{-\frac{1}{2}(a(x)-\delta (x))^{2}} \bigg| \notag\\
& \leq \max \left\{ \bigg| e^{-\frac{1}{2}a(x)^{2}} -  e^{-\frac{1}{2}(a(x)-\delta (x))^{2}} \bigg| \Bigg| a(x) \geq 92\sqrt{d \log K}, |\delta(x)| \leq \Delta   \right\} \notag\\
&\leq \max_{a \geq 92\sqrt{d \log K}} \{ \max ( |e^{-\frac{a^{2}}{2}}- e^{-\frac{(a-\Delta)^{2}}{2}}| , |e^{-\frac{a^{2}}{2}}- e^{-\frac{(a+\Delta)^{2}}{2}}|) \} \notag\\
&= \max_{a \geq 92\sqrt{d \log K}} \{ \max ( e^{-\frac{(a-\Delta)^{2}}{2}} - e^{-\frac{a^{2}}{2}} , e^{-\frac{a^{2}}{2}}- e^{-\frac{(a+\Delta)^{2}}{2}}) \} \notag\\
& \leq \max (\max_{a \geq 92\sqrt{d \log K}} e^{-\frac{(a-\Delta)^{2}}{2}} - e^{-\frac{a^{2}}{2}}, \max_{a \geq 92\sqrt{d \log K}} e^{-\frac{a^{2}}{2}}- e^{-\frac{(a+\Delta)^{2}}{2}})) \notag\\
& \leq \max (\max_{a \geq 76\sqrt{d \log K}} e^{-\frac{a^{2}}{2}}- e^{-\frac{(a+\Delta)^{2}}{2}}, \max_{a \geq 92\sqrt{d \log K}} e^{-\frac{a^{2}}{2}}- e^{-\frac{(a+\Delta)^{2}}{2}})) .
\end{align}
The last inequality holds since $\Delta \leq 16\sqrt{d\log K}$. For fixed $\Delta$, let $f(a)=e^{-\frac{a^{2}}{2}}- e^{-\frac{(a+\Delta)^{2}}{2}}$. Then 
\%
f'(a)=-a e^{-\frac{a^{2}}{2}} + (a+\Delta)e^{-\frac{(a+\Delta)^{2}}{2}}
\%
We first show that $f'(a) \leq 0$, for any $a\geq 76\sqrt{d\log K}$.
Notice that 
\begin{align}
f'(a) &=-a e^{-\frac{a^{2}}{2}} + (a+\Delta)e^{-\frac{(a+\Delta)^{2}}{2}} \leq 0 \notag\\
& \iff (a+\Delta)e^{-\frac{(a+\Delta)^{2}}{2}} \leq a e^{-\frac{a^{2}}{2}} \notag\\
& \iff 1+ \frac{\Delta}{a} \leq e^{a\Delta + \frac{1}{2}\Delta^{2}}
\end{align}
The last statement is true because 
\%
e^{a\Delta + \frac{1}{2}\Delta^{2}} \geq 1 + a\Delta + \frac{1}{2}\Delta^{2} \geq 1+ \frac{\Delta}{a}
\%
when $a\geq 76\sqrt{d\log K} > 1.$ \\
Since $f'(a) \leq 0$ for any $a\geq 76\sqrt{d\log K}$, we have
\begin{align}
f(a) &\leq f(76\sqrt{d \log{K}}) \notag\\
&= \exp (-\frac{1}{2} (76\sqrt{d \log{K}})^{2}) - \exp (-\frac{1}{2} (76\sqrt{d \log{K}} + \Delta )^{2})\notag\\
&= e^{-\frac{1}{2} (76\sqrt{d \log{K}})^{2}} (1- e^{-76\sqrt{d \log K} \Delta -\frac{1}{2}\Delta^{2} } ) \notag\\
&\leq e^{-\frac{1}{2} (76\sqrt{d \log{K}})^{2}} (76\sqrt{d \log K} \Delta +\frac{1}{2}\Delta^{2}) \notag\\
&\leq e^{-\frac{1}{2} (76\sqrt{d \log{K}})^{2}} \cdot 84\sqrt{d \log K} \Delta \quad \text{(since $\Delta \leq 16 \sqrt{d \log K}$)}.
\end{align}
Which shows 
\%
\max_{a \geq 76\sqrt{d \log K}} e^{-\frac{a^{2}}{2}}- e^{-\frac{(a+\Delta)^{2}}{2}} \leq e^{-\frac{1}{2} (76\sqrt{d \log{K}})^{2}} \cdot 84\sqrt{d \log K} \Delta
\%
Similarly
\%
\max_{a \geq 92\sqrt{d \log K}} e^{-\frac{a^{2}}{2}}- e^{-\frac{(a+\Delta)^{2}}{2}} \leq e^{-\frac{1}{2} (92\sqrt{d \log{K}})^{2}} \cdot 100\sqrt{d \log K} \Delta
\%
Therefore 
\%
\max_{x \in \Omega_{1}^{*}} \bigg| e^{-\frac{1}{2}\|x-u_{i}^{*}\|^{2}} - e^{-\frac{1}{2}\|x-u_{i}\|^{2}}  \bigg| \leq e^{-\frac{1}{2} (76\sqrt{d \log{K}})^{2}} \cdot 84\sqrt{d \log K} \Delta \leq c_4 \min \left\{ \|u_{i}^{*}-u_{i}\|, 1 \right\}
\% 
where $c_4=e^{-\frac{1}{2} (76\sqrt{d \log{K}})^{2}} \cdot 1344 d \log K$ (Since $\Delta \leq 16\sqrt{d \log K} \min \left\{ \Delta, 1 \right\}$).
Notice that 
\begin{align}
c_{4} &= e^{-\frac{1}{2} (76\sqrt{d \log{K}})^{2}} \cdot 1344 d \log K \notag \\
&\leq e^{-\frac{1}{2} (76\sqrt{d \log{K}})^{2}} \cdot 1344 k^{d} K \notag\\
&\leq e^{-\frac{1}{2} (76\sqrt{d \log{K}})^{2}} \cdot 1344 k^{2d}
\notag\\
&= 1344 e^{-2886d \log K} \notag \\
& \leq 1344 e ^{-\frac{1}{2} (70\sqrt{d \log{K}})^{2}}
\end{align}
W.L.O.G., let $u_{1}^{*}=0$, and define $ u'= (50\sqrt{d \log K},0, \cdots, 0)$, then 

\begin{align}
\int_{\Omega_{1}^{*}} |P_{i}^{*}-P_{i}| \mathrm{d}x &= \int_{\Omega_{1}^{*}} (\frac{1}{\sqrt{2\pi}})^{d} | e^{-\frac{1}{2}\|x-u_{i}^{*}\|^{2}} - e^{-\frac{1}{2}\|x-u_{i}\|^{2}}  | \mathrm{d}x \notag\\
&\leq \int_{\Omega_{1}^{*}} (\frac{1}{\sqrt{2\pi}})^{d} \max_{x\in \Omega_{1}^{*}} | e^{-\frac{1}{2}\|x-u_{i}^{*}\|^{2}} - e^{-\frac{1}{2}\|x-u_{i}\|^{2}}  | \mathrm{d}x \notag\\
&\leq \int_{\Omega_{1}^{*}} (\frac{1}{\sqrt{2\pi}})^{d} 1344 e ^{-\frac{1}{2} (70\sqrt{d \log{K}})^{2}} \min \left\{ \|u_{i}^{*}-u_{i}\|, 1 \right\} \mathrm{d}x\notag\\
& = \min \left\{ \|u_{i}^{*}-u_{i}\|, 1 \right\} \int_{\Omega_{1}^{*}} (\frac{1}{\sqrt{2\pi}})^{d} 1344 e ^{-\frac{1}{2} (70\sqrt{d \log{K}})^{2}}  \mathrm{d}x \notag\\
& \leq 1344 \min \left\{ \|u_{i}^{*}-u_{i}\|, 1 \right\} \int_{\Omega_{1}^{*}} (\frac{1}{\sqrt{2\pi}})^{d} e^{-\frac{1}{2}\|x-u'\|^{2}}\mathrm{d}x \notag\\
& \leq 1344 \min \left\{ \|u_{i}^{*}-u_{i}\|, 1 \right\} \P (\mathcal{N}\left(u', I_d\right)\in \Omega_{1}^{*}) \notag\\
& \leq 1344 \min \left\{ \|u_{i}^{*}-u_{i}\|, 1 \right\} \P (\|\mathcal{N}\left(u', I_d\right)-u'\| \geq 34\sqrt{d \log K}) \notag\\
& \leq 1344 \min \left\{ \|u_{i}^{*}-u_{i}\|, 1 \right\} \cdot 2\exp(-\frac{(34\sqrt{d \log K})^{2}}{16d})\quad\text{(by Lemma \ref{norm_subGaussian})} \notag\\
&\leq 1344 \min \left\{ \|u_{i}^{*}-u_{i}\|, 1 \right\} \cdot 2\exp(-70\log K)\notag\\
&= 2688 \min \left\{ \|u_{i}^{*}-u_{i}\|, 1 \right\} \left(\frac{1}{K}\right)^{70} \notag\\
&\leq 2688 \left(\frac{1}{2}\right)^{69} \left(\frac{1}{K}\right) \min \left\{ \|u_{i}^{*}-u_{i}\|, 1 \right\}
\end{align}

\end{proof}
\subsection{Bracketing Number}\label{gmm2}

We upper bound the bracketing number of $\mP_{\mathcal{X}}(\mathcal{U})$ as follows.
\begin{lemma}\label{gmm_bn}
Let
\$
\mP_{\mathcal{X}}(\mathcal{U}):=\bigg\{\sum^K_{i=1}\frac{1}{K}\mN(u_i,I_d)\,\bigg|\,\bu=\{u_i\}^K_{i=1}\in\mathcal{U}\bigg\}.
\$
We assume there exists $D>0$ such that for any $\mathbf{u}=\{u_i\}^K_{i=1}\in\mathcal{U}$, it holds that
\$
\|u_i\|_2\leq D\sqrt{d\log K},~\forall i\in[K].
\$
Then the entropy can be bounded as follows,
\$
\log N\big(\mP_{\mathcal{X}}(\mathcal{U}),1/m\big)\leq 2dK\log(6mdKD).
\$
\end{lemma}

\begin{proof}[Proof of Lemma \ref{gmm_bn}]
First of all, we consider a set of standard Gaussian distribution
\$
\mP_{\mathcal{X}}(\mathcal{A}):=\bigg\{p_a(x)=\frac{1}{\sqrt{2\pi}}e^{-\frac{\|x-a\|^2_2}{2}}\,\bigg|\, a\in \mathcal{A}\bigg\},
\$
where $\mathcal{A}=\{a\in\R^d\,|\, \|a\|_2\leq D\sqrt{d\log K}\}$. Our goal is to find a $1/m$-bracket $\mN_{\b}(\mP_{\mathcal{X}}(\mathcal{A}),1/m)$ of $\mP_{\mathcal{X}}(\mathcal{A})$. In other words, for any $p_a(x)\in\mP_{\mathcal{X}}(\mathcal{A})$, we need to define $\bar p_a(x)\in\mN_{\b}(\mP_{\mathcal{X}}(\mathcal{A}),1/m)$ such that
\begin{itemize}
    \item $\bar p_a(x)\geq p_a(x),~\forall x\in\R^d$
    \item $\int |\bar p_a(x)-p_a(x)|\,dx\leq 1/m$.
\end{itemize}
We consider $\bar p_a(x)$ of the form
\$
\bar p_a(x)=\frac{1}{\sqrt{2\pi}}e^{-\frac{c_1\|x-\bar a\|^2_2}{2}+c_2}.
\$
We then specify $\bar a\in\R^d$, $c_1\in\R$ and $c_2\in\R$. Let $a=(a_1,\ldots,a_d)$ and $\epsilon>0$ be a parameter that will be chosen later. If $a_i\in[k\epsilon,(k+1)\epsilon)$ for some $k\in\mathbb{Z}$, we define $\bar a_i:=k\epsilon$ and $\bar a:=(\bar a_1,\ldots,\bar a_d)$, which implies 
\%\label{092201}
\|a-\bar a\|^2_2\leq d\epsilon^2.
\%
Note that $\bar p_a(x)\geq p_a(x)$ holds for any $x\in\R^d$ if and only if 
\$
(c_1-1)\bigg\|x+\frac{a-c_1\bar a}{c_1-1}\bigg\|^2_2+\frac{c_1}{1-c_1}\|a-\bar a\|^2_2\leq 2c_2,~\forall x\in\R^d.
\$
Let $c_1=1-\epsilon$. Then, we have $\bar p_a(x)\geq p_a(x)$ if and only if
\$
-\epsilon\bigg\|x+\frac{a-c_1\bar a}{c_1-1}\bigg\|^2_2+\frac{1-\epsilon}{\epsilon}\|a-\bar a\|^2_2\leq 2c_2,~\forall x\in\R^d.
\$
Note that
\$
-\epsilon\bigg\|x+\frac{a-c_1\bar a}{c_1-1}\bigg\|^2_2+\frac{1-\epsilon}{\epsilon}\|a-\bar a\|^2_2\leq \frac{1-\epsilon}{\epsilon}\|a-\bar a\|^2_2\leq d(1-\epsilon)\epsilon,
\$
where the last inequality follows from \eqref{092201}. Thus, by choosing $c_2=d(1-\epsilon)\epsilon/2$, we obtain $\bar p_a(x)\geq p_a(x)$ for any $x\in\R^d$. Note that
\$
\int |\bar p_a(x)-p_a(x)|\,dx=\frac{1}{\sqrt{c_1}}\cdot e^{c_2}-1=\frac{e^{\frac{d(1-\epsilon)\epsilon}{2}}}{\sqrt{1-\epsilon}}-1\leq \big(1+d(1-\epsilon)\epsilon\big)\cdot(1+\epsilon)-1\leq (1+2d)\epsilon.
\$
Here the first inequality follows from the fact that $e^x\leq 1+2x$ and $\frac{1}{\sqrt{1-x}}\leq 1+x$ for any $0<x<1/2$. Let $(1+2d)\epsilon=m^{-1}$. It then holds that
\$
\int |\bar p_a(x)-p_a(x)|\,dx\leq (1+2d)\epsilon=\frac{1}{m}.
\$
Recall that for any $a\in\mathcal{A}$, it holds that $\|a\|_2\leq D\sqrt{d\log K}$. Thus, we have
\$
N_{\b}(\mP_{\mathcal{X}}(\mathcal{A}),1/m)\leq \bigg(\frac{2D\sqrt{d\log K}}{\epsilon}\bigg)^d= \Big(2mD(1+2d)\sqrt{d\log K}\Big)^d.
\$
Then, we consider a set of Gaussian mixture model
\$
\mP_{\mathcal{X}}(\mathcal{U}):=\bigg\{\sum^K_{i=1}\frac{1}{K}\mN(u_i,I_d)\,\bigg|\,\bu=\{u_i\}^K_{i=1}\in\mathcal{U}\bigg\},
\$
where $\mathcal{U}=\{\{u_i\}^{K}_{i=1}\,|\,\|u_i\|_2\leq D\sqrt{d\log K}, \forall i\in [K]\}$. Our goal is to find a $1/m$-bracket $\mN(\mP_{\mathcal{X}}(\mathcal{U}),1/m)$ of $\mathcal{\mP_{\mathcal{X}}(\mathcal{U})}$. For any $p_{\bu}(x)\in\mathcal{\mP_{\mathcal{X}}(\mathcal{U})}$, it holds that
\$
p_{\bu}(x)=\sum^K_{i=1}\frac{1}{K}p_{u_i}(x),
\$
where $p_{u_i}(x)\in\mP_{\mathcal{X}}(\mathcal{A})$. Note that for any $i\in [K]$, there exists $\bar p_{u_i}(x)\in\mN_{\b}(\mP_{\mathcal{X}}(\mathcal{A}),1/m)$, such that
\begin{itemize}
    \item $\bar p_{u_i}(x)\geq p_{u_i}(x),~\forall x\in\R^d$
    \item $\int |\bar p_{u_i}(x)-p_{u_i}(x)|\,dx\leq 1/m$.
\end{itemize}
We define 
\$
\bar p_{\bu}(x)=\sum^K_{i=1}\frac{1}{K}\bar p_{u_i}(x).
\$
It then holds that
\$
\bar p_{\bu}(x)=\sum^K_{i=1}\frac{1}{K}\bar p_{u_i}(x)\geq \sum^K_{i=1}\frac{1}{K}p_{u_i}(x)=p_{\bu}(x),~\forall x\in\R^d
\$
and 
\$
\int |\bar p_{\bu}(x)-p_{\bu}(x)|\,dx\leq\sum^K_{i=1}\frac{1}{K}\int |\bar p_{u_i}(x)-p_{u_i}(x)|\,dx\leq \sum^K_{i=1}\frac{1}{mK}=\frac{1}{m}.
\$
Thus, we obtain that
\$
N_{\b}(\mP_{\mathcal{X}}(\mathcal{U}),1/m)\leq \Big(N_{\b}(\mP_{\mathcal{X}}(\mathcal{A}),1/m)\Big)^K\leq \Big(2mD(1+2d)\sqrt{d\log K}\Big)^{dK},
\$
which implies that
\$
\log N_{\b}(\mP_{\mathcal{X}}(\mathcal{U}),1/m)\leq dK\log\Big(2mD(1+2d)\sqrt{d\log K}\Big)\leq 2dK\log(6mdKD).
\$

\end{proof}

\subsection{Rademacher Complexity}\label{gmm3}

Given labeled data $\{x_j,y_j\}^{n}_{j=1}$ and the pretrained $\hat \bu$, the function class
\$
\big\{(\mathds{1}_{g_{\hat\bu,\psi}(x_1)\neq y_1},\ldots,\mathds{1}_{g_{\hat\bu,\psi}(x_n)\neq y_n})\,\big|\,\psi\in\Psi\big\}
\$
is a finite function class, whose Rademacher complexity can be bounded by the following lemma.

\begin{lemma}\label{gmm_rc}
Let $A = \{a^{1},\ldots,a^{N}\}$ be a finite set of vectors in $\R^n$. Then, the Rademacher complexity can be bounded as follows,
\$
R_n(A)\leq \max_{a\in A}\|a\|_2\cdot\frac{2\sqrt{2\log N}}{n}.
\$
\end{lemma}
\begin{proof}
Note that for any $\lambda>0$
\%\label{092204}
R_n(A)&=\E\bigg[\sup_{a\in A}\frac{2}{n}\sum^n_{i=1}\sigma_i a_i\bigg]\leq \frac{1}{\lambda}\log \E\Big[e^{\sup_{a\in A}\frac{2\lambda}{n}\sum^n_{i=1}\sigma_i a_i}\Big]\notag\\
&\leq \frac{1}{\lambda}\log \sum_{a\in A}\E\Big[e^{\frac{2\lambda}{n}\sum^n_{i=1}\sigma_i a_i}\Big]=\frac{1}{\lambda}\log \sum_{a\in A}\prod^n_{i=1}\E\Big[e^{\frac{2\lambda}{n}\sigma_i a_i}\Big],
\%
where the first inequality follows from Jensen's inequality. Recall that $\sigma_i$ is a Rademacher random variable. Thus, we have
\%\label{092205}
\E\Big[e^{\frac{2\lambda}{n}\sigma_i a_i}\Big]=\frac{1}{2}e^{\frac{2\lambda}{n} a_i}+\frac{1}{2}e^{-\frac{2\lambda}{n} a_i}\leq e^{\frac{2\lambda^2a_i^2}{n^2}},
\%
where the last inequality follows from the fact that $(e^x+e^{-x})/2\leq e^{x^2/2}$. By \eqref{092204} and \eqref{092205}, we have
\%
R_n(A)\leq \frac{1}{\lambda}\log \sum_{a\in A}e^{\frac{2\lambda^2\|a\|^2}{n^2}}\leq \frac{1}{\lambda}\log |A|e^{\frac{2\lambda^2}{n^2}\cdot\max_{a\in A} \|a\|^2}=\frac{1}{\lambda}\log N+\frac{2\lambda}{n^2}\cdot\max_{a\in A} \|a\|^2.
\%
Let $\lambda=\sqrt{n\log N/2\max_{a\in A} \|a\|^2}$. We obtain that
\$
R_n(A)\leq \max_{a\in A} \|a\|\cdot\frac{2
\sqrt{2\log N}}{n}.
\$

\end{proof}
\subsection{Proofs for Theorem \ref{gmm_main}}\label{gmm4}

In the sequel, we prove Theorem \ref{gmm_main}.

\begin{proof}
Let $\Phi=\mathcal{U}$ and $\Psi$ be the set of $2^K$ classifications. Recall that the loss function is defined as $\ell(x,y)=\mathds{1}_{\{x\neq y\}}$, which is upper bound by $1$. Let $m=\tilde{\Omega}(dK^3)$. By Theorem \ref{tv_mle} and Lemma \ref{gmm_bn}, it holds that
\$
\TV\big(\P_{\hat\phi}(x),\P_{\phi^* }(x)\big)\lesssim \sqrt{\frac{1}{m}\log\frac{N_{\b}(\mP_{\mathcal{X}}(\Phi),1/m)}{\delta}}\lesssim\sqrt{\frac{dK}{m}\log\frac{mdKD}{\delta}}\lesssim\frac{1}{K}.
\$
Then, by Lemma \ref{gmm_ti}, Assumption \ref{invariance} holds for Gaussian mixture models. By Theorem \ref{error_bound}, with probability at least $1-\delta$, we have the following excess risk bound,
\$
{\rm Error}_{\ell}(\hat\phi,\hat\psi)&\leq 2\max_{\phi\in\Phi} R_n(\ell\circ \mathcal{G}_{\phi,\Psi})+\sqrt{\frac{2}{n}\log\frac{4}{\delta}}+12\kappa\cdot\sqrt{\frac{1}{m}\log\frac{2N(\mP_{\mathcal{X}}(\Phi),1/m)}{\delta}},
\$
where $\kappa=c_2$ is some absolute constants that represents the transferability of the model. By Lemma \ref{gmm_bn}, we further have
\%\label{092206}
{\rm Error}_{\ell}(\hat\phi,\hat\psi)\leq2\max_{\phi\in\Phi} R_n(\ell\circ \mathcal{G}_{\phi,\Psi})+\sqrt{\frac{2}{n}\log\frac{4}{\delta}}+12\kappa\cdot\sqrt{\frac{2dK}{m}\log\frac{12mdKD}{\delta}}.
\%
For any $\phi\in\Phi$, we have
\%
R_n(\ell\circ \mathcal{G}_{\phi,\Psi})=\E\bigg[\sup_{\psi\in\Psi}\frac{1}{n}\sum^n_{i=1}\sigma_i\mathds{1}_{\{g_{\phi,\psi}(x_i)\neq y_i\}}\bigg].
\%
Note that $|\Psi|= 2^K$. By Lemma \ref{gmm_rc}, it holds for any $\phi\in\Phi$ that
\%\label{092207}
R_n(\ell\circ \mathcal{G}_{\phi,\Psi})\leq \sqrt{n}\cdot\frac{2\sqrt{2\log 2^{K}}}{n}=2\sqrt{\frac{2K\log 2}{n}}.
\%
By \eqref{092206} and \eqref{092207}, we have
\$
{\rm Error}_{\ell}(\hat\phi,\hat\psi)&\leq 4\sqrt{\frac{2K\log 2}{n}}+\sqrt{\frac{2}{n}\log\frac{4}{\delta}}+12\kappa\cdot\sqrt{\frac{2dK}{m}\log\frac{12mdKD}{\delta}}\\
&=\mathcal{O}\bigg(\sqrt{\frac{K\log \frac{1}{\delta}}{n}}+\kappa\sqrt{\frac{dK\log \frac{mdKD}{\delta}}{m}}\bigg)\\
&=\tilde{\mathcal{O}}\bigg(\sqrt{\frac{K}{n}}+\kappa\sqrt{\frac{dK}{m}}\bigg),
\$
where $\kappa=c_2$ is some absolute constants that represents the transferability of the model.
\end{proof}
Thus, we prove Theorem \ref{gmm_main}.
\section{Proofs for Section \ref{contrastive_learning}}\label{proof_contrastive_learning}
In Section \ref{proof_contrastive_ti}, we show that contrastive learning with linear regression as downstream tasks is $\kappa^{-1}$-weakly-informative by proving Lemma \ref{contrastive_ti}. In Section \ref{proof_contrastive_main}, we prove Theorem \ref{contrastive_main}.

\subsection{Proofs for Lemma \ref{contrastive_ti}}\label{proof_contrastive_ti}
Recall that in the setting of contrastive learning, we assume that $x$ and $x'$ are sampled independently from the same distribution $\P(x)$. And we assume the label $t$ that captures the similarity between $x$ and $x'$ satisfies
\$
&\P(t=1\,|\,x,x')=\frac{1}{1+e^{-f_{\theta^* }(x)^Tf_{\theta^* }(x')}},\notag\\
&\P(t=-1\,|\,x,x')=\frac{1}{1+e^{f_{\theta^* }(x)^Tf_{\theta^* }(x')}}.
\$
Lemma \ref{contrastive_ti} directly follows from the following lemma.
\begin{lemma}\label{contrastive_ti_weak}
There exists $O\in\R^{r\times r}$, $O^TO=OO^T=I_{r}$ such that
\$
\TV\big(\P_{Of_{\theta}}(x,z),\P_{f_{\theta^* }}(x,z)\big)\leq c\cdot\sqrt{\frac{1}{\sigma_{\min}(\E[f_{\theta^* }(x)f_{\theta^* }(x)^{T}])}}\cdot H\big(\P_{f_{\theta}}(x,x',t),\P_{f_{\theta^* }}(x,x',t)\big).
\$
Here $c$ is some absolute constants.
\end{lemma}

We first prove the following lemma, which is the core of the proof of Lemma \ref{contrastive_ti_weak}.
\begin{lemma}\label{ti_1}
Suppose that $\E[f_{\theta}(x)f_{\theta^* }(x)^{T}]=\E[f_{\theta^* }(x)f_{\theta}(x)^{T}]$ are positive semi-definite matrices. Then we have
\$
\E\big[\big(f_{\theta}(x)^{T}f_{\theta}(x')-f_{\theta^* }(x)^{T}f_{\theta^* }(x')\big)^2\big]\geq (2\sqrt{2}-2)\sigma_{\min}\big(\E[f_{\theta^* }(x)f_{\theta^* }(x)^{T}]\big)\cdot\E[\|f_{\theta^* }(x)-f_{\theta}(x)\|^2_2].
\$
\end{lemma}
\begin{proof}[Proof of Lemma \ref{ti_1}]
For notation simplicity, we denote $\Delta(x):=f_{\theta^* }(x)-f_{\theta}(x)$. It then holds that
\%\label{101401}
&\E\big[\big(f_{\theta}(x)^{T}f_{\theta}(x')-f_{\theta^* }(x)^{T}f_{\theta^* }(x')\big)^2\big]\notag\\
&=\E\big[\big(f_{\theta^* }(x)^T\Delta(x')+\Delta(x)^Tf_{\theta^* }(x')-\Delta(x)^T\Delta(x')\big)^2\big]\notag\\
&=\E\big[\big(\Delta(x)^T\Delta(x')\big)^2-2\sqrt{2}\Delta(x)^T\Delta(x')f_{\theta^* }(x')^T\Delta(x)+2f_{\theta^* }(x)^T\Delta(x')f_{\theta^* }(x')^T\Delta(x)\big]\notag\\
&\quad +(4-2\sqrt{2})\E[f_{\theta}(x')^T\Delta(x)\Delta (x)^Tf_{\theta^* }(x')]+(2\sqrt{2}-2)\E[f_{\theta^* }(x')^T\Delta(x)\Delta (x)^Tf_{\theta^* }(x')].
\%

For the first term of \eqref{101401}, we have
\%
&\E\big[\big(\Delta(x)^T\Delta(x')\big)^2-2\sqrt{2}\Delta(x)^T\Delta(x')f_{\theta^* }(x')^T\Delta(x)+2f_{\theta^* }(x)^T\Delta(x')f_{\theta^* }(x')^T\Delta(x)\big]\notag\\
&=\Tr\Big(\E[\Delta(x')\Delta(x')^T\Delta(x)\Delta(x)^T-2\sqrt{2}\Delta(x')f_{\theta^* }(x')^T\Delta(x)\Delta(x)^T+2\Delta(x')f_{\theta^* }(x')^T\Delta(x)f_{\theta^* }(x)^T]\Big)\notag\\
&=\Tr\Big(\big(\E[\Delta(x)\Delta(x)^T]\big)^2-2\sqrt{2}\E[\Delta(x)f_{\theta^* }(x)^T]\cdot\E[\Delta(x)\Delta(x)^T]+2\big(\E[\Delta(x)f_{\theta^* }(x)^T]\big)^2\Big)\notag\\
&=\Tr\Big(\big(\E[\Delta(x)\Delta(x)^T]-\sqrt{2}\E[\Delta(x)f_{\theta^* }(x)^T]\big)^2\Big),
\%
where the second equation follows from our assumption that $x,x'$ are i.i.d. Note that $\E[f_{\theta}(x)f_{\theta^* }(x)^{T}]=\E[f_{\theta^* }(x)f_{\theta}(x)^{T}]$. Thus, we obtain
\%
\Big(\E[\Delta(x)\Delta(x)^T]-\sqrt{2}\E[\Delta(x)f_{\theta^* }(x)^T]\Big)^T&=\E[\Delta(x)\Delta(x)^T]-\sqrt{2}\E[f_{\theta^* }(x)\Delta(x)^T]\notag\\
&=\E[\Delta(x)\Delta(x)^T]-\sqrt{2}\E[\Delta(x)f_{\theta^* }(x)^T],
\%
which implies that $\E[\Delta(x)\Delta(x)^T]-\sqrt{2}\E[\Delta(x)f_{\theta^* }(x)^T]$ is symmetric. It then holds that
\%\label{101402}
&\E\big[\big(\Delta(x)^T\Delta(x')\big)^2-2\sqrt{2}\Delta(x)^T\Delta(x')f_{\theta^* }(x')\Delta(x)+2f_{\theta^* }(x)^T\Delta(x')f_{\theta^* }(x')^T\Delta(x)\big]\notag\\
&=\Tr\Big(\big(\E[\Delta(x)\Delta(x)^T]-\sqrt{2}\E[\Delta(x)f_{\theta^* }(x)^T]\big)^2\Big)\geq 0.
\%

For the second term of \eqref{101401}, we have
\%\label{101403}
\E[f_{\theta}(x')^T\Delta(x)\Delta (x)^Tf_{\theta^* }(x')]=\Tr\Big(\E[f_{\theta^* }(x')f_{\theta}(x')^{T}]\cdot\E[\Delta(x)\Delta(x)^T]\Big)\geq 0,
\%
where the inequality follows from the fact $\E[f_{\theta^* }(x')f_{\theta}(x')^{T}]\succcurlyeq 0$ and $\E[\Delta(x)\Delta(x)^T]\succcurlyeq 0$.

For the third term of \eqref{101401}, we have
\%\label{101404}
\E[f_{\theta^* }(x')^T\Delta(x)\Delta (x)^Tf_{\theta^* }(x')]&=\Tr\Big(\E[f_{\theta^* }(x)f_{\theta^* }(x)^T]\cdot\E[\Delta(x)\Delta (x)^T]\Big)\notag\\
&\geq \sigma_{\min}\big(\E[f_{\theta^* }(x)f_{\theta^* }(x)^T]\big)\Tr\Big(\E[\Delta(x)\Delta (x)^T]\Big)\notag\\
&=\sigma_{\min}\big(\E[f_{\theta^* }(x)f_{\theta^* }(x)^T]\big)\E[\|\Delta(x)\|^2_2].
\%
Combining \eqref{101401}, \eqref{101402}, \eqref{101403} and \eqref{101404}, we have
\%
\E\big[\big(f_{\theta}(x)^{T}f_{\theta}(x')-f_{\theta^* }(x)^{T}f_{\theta^* }(x')\big)^2\big]\geq (2\sqrt{2}-2)\sigma_{\min}\big(\E[f_{\theta^* }(x)f_{\theta^* }(x)^T]\big)\E[\|\Delta(x)\|^2_2]
\%
\end{proof}

With Lemma \ref{ti_1}, we prove Lemma \ref{contrastive_ti_weak} in the following.

\begin{proof}[Proof of Lemma \ref{contrastive_ti_weak}]
We consider the singular value decomposition (SVD) of $\E[f_{\theta}(x)f_{\theta^* }(x)^T]=U_1\Sigma_1 V^T_1$ and $\E[f_{\theta^* }(x)f_{\theta}(x)^{T}]=(\E[f_{\theta}(x)f_{\theta^* }(x)^T])^T=V_1\Sigma_1 U^T_1$. We define $O:=V_1 U^T_1\in\R^{r\times r}$, which satisfies $O^TO=OO^T=I_r$. It then holds that
\%
\E[Of_{\theta}(x)f_{\theta^* }(x)^T]=\E\big[f_{\theta^* }(x)\big(Of_{\theta}(x)\big)^{T}\big]=V_1\Sigma_1 V^T_1,
\%
which are positive semi-definite matrices. By Lemma \ref{ti_1}, we have
\%\label{101501}
&\E\big[\big(f_{\theta}(x)^{T}f_{\theta}(x')-f_{\theta^* }(x)^{T}f_{\theta^* }(x')\big)^2\big]\notag\\
&\quad\geq (2\sqrt{2}-2)\sigma_{\min}\big(\E[f_{\theta^* }(x)f_{\theta^* }(x)^{T}]\big)\cdot\E[\|f_{\theta^* }(x)-Of_{\theta}(x)\|^2_2].
\%
For Hellinger distance, we have
\%\label{101502}
&2H^2\big(\P_{f_{\theta}}(x,x',t),\P_{f_{\theta^* }}(x,x',t)\big)\notag\\
&=\int \Big(\sqrt{p_{f_{\theta}}(x,x',t)}-\sqrt{p_{f_{\theta^* }}(x,x',t)}\Big)^2\,dtdxdx'\notag\\
&=\int \Big(\sqrt{p_{f_{\theta}}(t=1\,|\,x,x')}-\sqrt{p_{f_{\theta^* }}(t=1\,|\,x,x')}\Big)^2p(x,x')\,dxdx'\notag\\
&\quad+\int \Big(\sqrt{p_{f_{\theta}}(t=0\,|\,x,x')}-\sqrt{p_{f_{\theta^* }}(t=0\,|\,x,x')}\Big)^2p(x,x')\,dxdx'
\%
For the first term of \eqref{101502}, we have
\%
&\int \Big(\sqrt{p_{f_{\theta}}(t=1\,|\,x,x')}-\sqrt{p_{f_{\theta^* }}(t=1\,|\,x,x')}\Big)^2p(x,x')\,dxdx'\notag\\
&\quad=\int \Big(\sqrt{h\big(f_{\theta}(x)^{T}f_{\theta}(x')\big)}-\sqrt{h\big(f_{\theta^* }(x)^{T}f_{\theta^* }(x')}\Big)^2p(x,x')\,dxdx',
\%
where 
\%
h(a):=\frac{1}{1+e^{-a}}.
\%
By Cauchy-Schwartz inequality, we have $|f_{\theta}(x)^{T}f_{\theta}(x')|\leq \|f_{\theta}(x)\|_2\|f_{\theta}(x')\|_2\leq 1$. Note that for any $a,b\in[-1,1]$, we have
\%
&\Big(\sqrt{h(a)}-\sqrt{h(b)}\Big)^2\notag\\
&=\frac{\big(h(a)-h(b)\big)^2}{\Big(\sqrt{h(a)}+\sqrt{h(b)}\Big)^2}\geq \frac{1}{4}\big(h(a)-h(b)\big)^2=\frac{1}{4}h'(\xi)^2(a-b)^2\geq \frac{1}{2+e+e^{-1}}(a-b)^2.
\%
Thus, it holds that
\%\label{101503}
&\int \Big(\sqrt{p_{f_{\theta}}(t=1\,|\,x,x')}-\sqrt{p_{f_{\theta^* }}(t=1\,|\,x,x')}\Big)^2p(x,x')\,dxdx'\notag\\
&\geq \frac{1}{2+e+e^{-1}}\int\big(f_{\theta}(x)^{T}f_{\theta}(x')-f_{\theta^* }(x)^{T}f_{\theta^* }(x')\big)^2p(x,x')\,dxdx'\notag\\
&=\frac{1}{2+e+e^{-1}}\E\big[\big(f_{\theta}(x)^{T}f_{\theta}(x')-f_{\theta^* }(x)^{T}f_{\theta^* }(x')\big)^2\big].
\%

Similarly, For the second term of \eqref{101502}, we have 
\%\label{101504}
&\int \Big(\sqrt{p_{f_{\theta}}(t=0\,|\,x,x')}-\sqrt{p_{f_{\theta^* }}(t=0\,|\,x,x')}\Big)^2p(x,x')\,dxdx'\notag\\
&\geq\frac{1}{2+e+e^{-1}}\E\big[\big(f_{\theta}(x)^{T}f_{\theta}(x')-f_{\theta^* }(x)^{T}f_{\theta^* }(x')\big)^2\big]
\%
Combining \eqref{101502}, \eqref{101503} and \eqref{101504}, we have
\%\label{101505}
H^2\big(\P_{f_{\theta}}(x,x',t),\P_{f_{\theta^* }}(x,x',t)\big)\geq \frac{1}{2+e+e^{-1}}\E\big[\big(f_{\theta}(x)^{T}f_{\theta}(x')-f_{\theta^* }(x)^{T}f_{\theta^* }(x')\big)^2\big].
\%

We choose $O\in\R^{r\times r}$ that satisfies \eqref{101501}. For the TV distance, we have
\%
\TV\big(\P_{Of_{\theta}}(x,z),\P_{f_{\theta^* }}(x,z)\big)=\frac{1}{2}\int |p_{Of_{\theta}}(z\,|\,x)-p_{f_{\theta^* }}(z\,|\,x)|p(x)\,dx.
\%
Note that $z\,|\,x\sim\mN(f_{\theta}(x),I_r)$. By Lemma \ref{tv_norm}, we have 
\%\label{101506}
\TV\big(\P_{Of_{\theta}}(x,z),\P_{f_{\theta^* }}(x,z)\big)&=\frac{1}{2}\int |p_{Of_{\theta}}(z\,|\,x)-p_{f_{\theta^* }}(z\,|\,x)|p(x)\,dx\notag\\
&\leq \frac{1}{2}\int\min\{1,\|Of_{\theta}(x)-f_{\theta^* }(x)\|_2\}p(x)\,dx\notag\\
&\leq \frac{1}{2}\min\bigg\{1, \int\|Of_{\theta}(x)-f_{\theta^* }(x)\|_2 p(x)\,dx\bigg\}\notag\\
&= \frac{1}{2}\min\big\{1,\E[\|Of_{\theta}(x)-f_{\theta^* }(x)\|_2]\big\}.
\%
Combining \eqref{101501}, \eqref{101505} and \eqref{101506}, we show that
\%
&\TV\big(\P_{Of_{\theta}}(x,z),\P_{f_{\theta^* }}(x,z)\big)\notag\\
&\leq \frac{1}{2}\E[\|Of_{\theta}(x)-f_{\theta^* }(x)\|_2]\notag\\
&\leq \frac{1}{2}\sqrt{\E[\|Of_{\theta}(x)-f_{\theta^* }(x)\|^2_2]}\notag\\
&\leq \frac{1}{2}\sqrt{\frac{1}{(2\sqrt{2}-2)\sigma_{\min}\big(\E[f_{\theta^* }(x)f_{\theta^* }(x)^{T}]\big)}\E\big[\big(f_{\theta}(x)^{T}f_{\theta}(x')-f_{\theta^* }(x)^{T}f_{\theta^* }(x')\big)^2\big]}\notag\\
&\leq\frac{1}{2}\sqrt{\frac{2+e+e^{-1}}{(2\sqrt{2}-2)\sigma_{\min}\big(\E[f_{\theta^* }(x)f_{\theta^* }(x)^{T}]\big)}} H\big(\P_{f_{\theta}}(x,x',t),\P_{f_{\theta^* }}(x,x',t)\big).
\%
Thus, we prove Lemma \ref{contrastive_ti_weak}.
\end{proof}

Lemma \ref{contrastive_ti_weak} directly implies Lemma \ref{contrastive_ti}.
\begin{proof}[Proof of Lemma \ref{contrastive_ti}]
For any $\theta\in\Theta$, we choose $O\in\R^{r\times r}$ that satisfies Lemma \ref{contrastive_ti_weak}. It then holds that
\$
\TV\big(\P_{f_{\theta},O^{T}\beta^* }(x,y),\P_{f_{\theta^* },\beta^* }(x,y)\big)&=\TV\big(\P_{Of_{\theta},\beta^* }(x,y),\P_{f_{\theta^* },\beta^* }(x,y)\big)\\
&\leq\TV\big(\P_{Of_{\theta}}(x,z),\P_{f_{\theta^* }}(x,z)\big)\notag\\
&\leq c\cdot\sqrt{\frac{1}{\sigma_{\min}(\E[f_{\theta^* }(x)f_{\theta^* }(x)^{T}])}} \cdot H\big(\P_{f_{\theta}}(x,x',t),\P_{f_{\theta^* }}(x,x',t)\big).
\$
Thus, we prove that the model is $\kappa^{-1}$-weakly-informative, where 
\%
\kappa=c\cdot\sqrt{\frac{1}{\sigma_{\min}(\E[f_{\theta^* }(x)f_{\theta^* }(x)^{T}])}}.
\%
Here $c$ is some absolute constants.

\end{proof}
\subsection{Proofs for Theorem \ref{contrastive_main}}\label{proof_contrastive_main}
In this section, we prove Theorem \ref{contrastive_main}. Suppose that $\hat\theta,\hat\beta$ are the outputs of Algorithm \ref{mle+erm}. Let $\ell$ be the squared loss and $\tilde{\ell}$ be its truncation with truncation level $L$. The optimal predictor defined in \eqref{opt_est} has the following closed form solution
\%
g_{\theta,\beta}(x)=\E_{\theta,\beta}[y\,|\,x]=\beta^T f_{\theta}(x).
\%

We have the following guarantees.

\begin{lemma}\label{contrastive_lemma1} Let the truncation level $L=36(D^2+1)\log n$. It then holds that
\%
\sup_{\theta,\beta}\big\{\E_{\theta^* ,\beta^* }\big[\ell\big(g_{\theta,\beta}(x),y\big)\big]-\E_{\theta^* ,\beta^* }\big[\tilde{\ell}\big(g_{\theta,\beta}(x),y\big)\big]\big\}\leq \sqrt{\frac{18(D^2+1)\log n}{\pi n}}.
\%
\end{lemma}
\begin{proof}[Proof of Lemma \ref{contrastive_lemma1}]
Note that
\%
\big(g_{\theta,\beta}(x)-y\big)\big|x=\big(\beta^T f_{\theta}(x)-y\big)\big|x\sim\mN\big(\beta^T f_{\theta}(x)-\beta^{* T} f_{\theta^* }(x),1\big)
\%
We denote by $c(x):=\beta^T f_{\theta}(x)-\beta^{* T} f_{\theta^* }(x)$. It holds that $|c(x)|\leq 2D$. Thus, it holds for any $\theta,\beta$ that
\%
&\E_{\theta^* ,\beta^* }\big[\ell\big(g_{\theta,\beta}(x),y\big)-\tilde{\ell}\big(g_{\theta,\beta}(x),y\big)\,\big|\, x\big]\notag\\
&=\E_{\theta^* ,\beta^* }\Big[\Big(\big(g_{\theta,\beta}(x)-y\big)^2-L\Big)\mathds{1}_{\{(g_{\theta,\beta}(x)-y)^2>L\}}\,\big|\,x\Big]\notag\\
&=\int^{{+\infty}}_{\sqrt{L}} (u^2-L)\cdot\frac{1}{\sqrt{2\pi}}e^{-\frac{\big(u-c(x)\big)^2}{2}}\,du\notag\\
&=\int^{{+\infty}}_{\sqrt{L}-c(x)}\big((u+c(x))^2-L\big)\cdot\frac{1}{\sqrt{2\pi}}e^{-\frac{u^2}{2}}\,du\notag\\
&=\frac{\sqrt{L}+c(x)}{\sqrt{2\pi}}e^{-\frac{\big(\sqrt{L}-c(x)\big)^2}{2}}+\frac{1+c(x)^2-L}{\sqrt{2\pi}}\int^{{+\infty}}_{\sqrt{L}-c(x)}e^{-\frac{u^2}{2}}\,du\notag\\
&\leq \frac{\sqrt{L}+c(x)}{\sqrt{2\pi}}e^{-\frac{\big(\sqrt{L}-c(x)\big)^2}{2}} \quad(L\geq 4D^2+1\geq c(x)^2+1)\notag\\
&\leq \frac{2(\sqrt{L}-c(x))}{\sqrt{2\pi}}e^{-\frac{\big(\sqrt{L}-c(x)\big)^2}{2}} \quad(L\geq 36D^2\geq (3c(x))^2)\notag\\
&\leq \frac{2(\sqrt{L}-2D)}{\sqrt{2\pi}}e^{-\frac{\big(\sqrt{L}-2D\big)^2}{2}}\notag\\
&\leq \frac{\sqrt{L}}{\sqrt{2\pi}}e^{-\frac{L}{8}}\quad(\sqrt{L}-2D\geq\frac{\sqrt{L}}{2})
\%
As a result, we show that
\%
&\sup_{\theta,\beta}\big\{\E_{\theta^* ,\beta^* }\big[\ell\big(g_{\theta,\beta}(x),y\big)\big]-\E_{\theta^* ,\beta^* }\big[\tilde{\ell}\big(g_{\theta,\beta}(x),y\big)\big]\big\}\notag\\
&\leq \E_{\theta^* ,\beta^* }\Big[\sup_{\theta,\beta}\E_{\theta^* ,\beta^* }\big[\ell\big(g_{\theta,\beta}(x),y\big)-\tilde{\ell}\big(g_{\theta,\beta}(x),y\big)\,\big|\, x\big]\Big]\notag\\
&\leq \frac{\sqrt{L}}{\sqrt{2\pi}}e^{-\frac{L}{8}}\notag\\
&\leq \sqrt{\frac{18(D^2+1)\log n}{\pi n}}. \quad(L=36(D^2+1)\log n)
\%
\end{proof}

\begin{lemma}\label{contrastive_lemma2} Suppose that $\hat\theta,\hat\beta$ are the outputs of Algorithm \ref{mle+erm}. Let $\tilde{\ell}$ be the truncated squared loss with truncation level $L$. Then there exists an absolute constant $c$ such that with probability at least $1-\delta$ that
\%
&\E_{\theta^* ,\beta^* }\big[\tilde{\ell}\big(g_{\hat\theta,\hat\beta}(x),y\big)\big]-\E_{\theta^* ,\beta^* }\big[\tilde{\ell}\big(g_{\theta^* ,\beta^* }(x),y\big)\big]\notag\\
&\leq c\kappa L\cdot\sqrt{\frac{1}{m}\log\frac{N_{\b}\big(\mP_{\mathcal{X}\times\mathcal{S}}(\mathcal{F}_{\theta}),1/m^2\big)}{\delta}}+cL\sqrt{\frac{\log 1/\delta}{n}}+c\sqrt{L}\sup_{\theta\in\Theta}R_n(\mathcal{G}_{\theta,\mathcal{B}}),
\%
where
\$
\kappa=c_3\sqrt{\frac{1}{\sigma_{\min}\big(\E[f_{\theta^* }(x)f_{\theta^* }(x)^{T}]\big)}}
\$
for some absolute constants $c_3$. Here $R_n(\mathcal{G}_{\theta,\mathcal{B}})$ is the Rademacher complexity defined as
\%
R_n(\mathcal{G}_{\theta,\mathcal{B}})=\E\bigg[\sup_{\beta\in\mathcal{B}}\frac{2}{n}\sum^n_{i=1}\sigma_ig_{\theta,\beta}(x_i)\bigg],
\%
where $\sigma_i$ are Rademacher random variables.
\end{lemma}
\begin{proof}[Proof of Lemma \ref{contrastive_lemma2}]
With Lemma \ref{same_prediction} and Lemma \ref{contrastive_ti} in hand, Lemma \ref{contrastive_lemma2} follows directly from Theorem \ref{weak_error_bound} and the fact that $\tilde\ell$ is $2\sqrt{L}$-Lipschitz.

\end{proof}

With Lemma \ref{contrastive_lemma1} and Lemma \ref{contrastive_lemma2} in hand, we are now ready to prove Theorem \ref{contrastive_main}.

\begin{proof}[Proof of Theorem \ref{contrastive_main}]
Note that
\%
{\rm Error}_{\ell}(\hat\theta,\hat\beta)&=\E_{\theta^* ,\beta^* }\big[\ell\big(g_{\hat\theta,\hat\beta}(x),y\big)\big]-\E_{\theta^* ,\beta^* }\big[\ell\big(g_{\theta^* ,\beta^* }(x),y\big)\big]\notag\\
&=\E_{\theta^* ,\beta^* }\big[\ell\big(g_{\hat\theta,\hat\beta}(x),y\big)\big]-\E_{\theta^* ,\beta^* }\big[\tilde{\ell}\big(g_{\hat\theta,\hat\beta}(x),y\big)\big]\notag\\
&\quad+\E_{\theta^* ,\beta^* }\big[\tilde\ell\big(g_{\hat\theta,\hat\beta}(x),y\big)\big]-\E_{\theta^* ,\beta^* }\big[\tilde\ell\big(g_{\theta^* ,\beta^* }(x),y\big)\big]\notag\\
&\quad+\E_{\theta^* ,\beta^* }\big[\tilde\ell\big(g_{\theta^* ,\beta^* }(x),y\big)\big]-\E_{\theta^* ,\beta^* }\big[\ell\big(g_{\theta^* ,\beta^* }(x),y\big)\big]\notag\\
&\leq \sup_{\theta,\beta}\big\{\E_{\theta^* ,\beta^* }\big[\ell\big(g_{\theta,\beta}(x),y\big)\big]-\E_{\theta^* ,\beta^* }\big[\tilde{\ell}\big(g_{\theta,\beta}(x),y\big)\big]\big\}\notag\\
&\quad +\E_{\theta^* ,\beta^* }\big[\tilde\ell\big(g_{\hat\theta,\hat\beta}(x),y\big)\big]-\E_{\theta^* ,\beta^* }\big[\tilde\ell\big(g_{\theta^* ,\beta^* }(x),y\big)\big].
\%
Let the truncation level be $L=36(D^2+1)\log n$. By Lemma \ref{contrastive_lemma1} and Lemma \ref{contrastive_lemma2}, we have
\%\label{101510}
&{\rm Error}(\hat\theta,\hat\beta)\notag\\
&\leq c\kappa L\cdot\sqrt{\frac{1}{m}\log\frac{N_{\b}\big(\mP_{\mathcal{X}\times\mathcal{S}}(\mathcal{F}_{\theta}),1/m^2\big)}{\delta}}+cL\sqrt{\frac{\log 1/\delta}{n}}+c\sqrt{L}\sup_{\theta\in\Theta}R_n(\mathcal{G}_{\theta,\mathcal{B}})\notag\\
&\quad+\sqrt{\frac{18(D^2+1)\log n}{\pi n}}.
\%
For the Rademacher complexity, we have
\%\label{101511}
R_n(\mathcal{G}_{\theta,\mathcal{B}})&=\E\bigg[\sup_{\beta\in\mathcal{B}}\frac{2}{n}\sum^n_{i=1}\sigma_i g_{\theta,\beta}(x_i)\bigg]\notag\\
&=\E\bigg[\sup_{\beta\in\mathcal{B}}\frac{2}{n}\sum^n_{i=1}\sigma_i \beta^T f_{\theta}(x_i)\bigg]\notag\\
&\leq \frac{2D}{\sqrt{n}},
\%
where the last inequality follows from Lemma \ref{factor_rc}. Combining \eqref{101510} and \eqref{101511}, we have
\%
&{\rm Error}(\hat\theta,\hat\beta)\notag\\
&\leq c\kappa L\cdot\sqrt{\frac{1}{m}\log\frac{N_{\b}\big(\mP_{\mathcal{X}\times\mathcal{S}}(\mathcal{F}_{\theta}),1/m^2\big)}{\delta}}+cL\sqrt{\frac{\log 1/\delta}{n}}+2cD\sqrt{\frac{L}{n}}\notag\\
&\quad+\sqrt{\frac{18(D^2+1)\log n}{\pi n}}\notag\\
&=\tilde{\mathcal{O}}\bigg(\kappa L\sqrt{\frac{\log N_{\b}\big(\mP_{\mathcal{X}\times\mathcal{S}}(\mathcal{F}_{\theta}),1/m^2\big)}{m}}+L\sqrt{\frac{1}{n}}\bigg),
\%
where $L=36(D^2+1)\log n$ and 
\$
\kappa=c_3\sqrt{\frac{1}{\sigma_{\min}\big(\E[f_{\theta^* }(x)f_{\theta^* }(x)^{T}]\big)}}
\$
for some absolute constants $c_3$.
\end{proof}

\section{Failure of Two-Phase MLE}\label{counter_example}

For simplicity, in the sequel, we consider the case where no side information is available, i.e., we have access to unlabeled data $\{x_i\}^m_{i=1}$ and labeled data $\{x_j,y_j\}^n_{j=1}$.  Another natural scheme is to use a two-phase MLE (Algorithm \ref{mle+mle}). To be specific, in the first phase, we use MLE to estimate $\phi^* $ based on the unlabeled data $\{x_i\}^m_{i=1}$. In the second phase, we use MLE again to estimate $\psi^* $ based on pretrained $\hat\phi$ and the labeled data $\{x_j,y_j\}^{n}_{j=1}$. 


\begin{algorithm}[H]
\caption{Two-phase MLE}\label{mle+mle}
\begin{algorithmic}[1]
\STATE {\bf Input:} $\{x_i\}^m_{i=1}$, $\{(x_j,y_j)\}^n_{j=1}$
\STATE Use unlabeled data $\{x_i\}^m_{i=1}$ to learn $\hat\phi$ via MLE:
\$
\hat\phi\leftarrow\argmax_{\phi\in\Phi}\sum^m_{i=1}\log p_{\phi}(x_i).
\$
\STATE Fix $\hat\phi$ and use labeled data $\{(x_j,y_j)\}^n_{j=1}$ to learn $\hat\psi$ via MLE:
\$
\hat\psi\leftarrow\argmax_{\psi\in\Psi}\sum^n_{j=1}\log p_{\hat\phi,\psi}(x_j,y_j).
\$
\STATE {\bf Output:} $\hat\phi$ and $\hat\psi$.
\end{algorithmic}
\end{algorithm}

Note that the two-phase MLE does not directly associate the learning process with the loss function. Thus, the only way to evaluate the excess risk is to study the total variation distance between $\P_{\hat\phi, \hat\psi}(x,y)$ and $\P_{\phi^* ,\psi^* }(x,y)$. In the pretraining phase, MLE guarantees that the estimator $\mathbb{P}_{\hat{\phi}}$ is close to $\mathbb{P}_{\phi^{*}}$ in the sense of total variation distance (Theorem \ref{tv_mle}). However, it's still possible that for some $x$, $\mathbb{P}_{\hat{\phi}}(x)=0$ while $\mathbb{P}_{\phi^{*}}(x) \neq 0$. This phenomenon may result in $\log p_{\hat\phi,\psi^* }(x_j,y_j)=-\infty$ for some labeled data in the learning of downstream tasks, which will dramatically influence the behaviour of MLE for estimating $\psi^{*}$ and finally lead to the failure of the second phase. Inspired by this idea, we give the following theorem.

\begin{theorem}\label{counter}
There exists $\Phi, \Psi, \phi^{*}\in{\Phi}, \psi^{*}\in{\Psi}$, such that for any constant $c>0$, there exists $m,n \geq c$ such that with probability at least $\frac{1}{2}(1-e^{-1})e^{-1}$, we have
\$
\TV\big(\P_{\hat\phi,\hat\psi}(x,y),\P_{\phi^* ,\psi^* }(x,y)\big)\geq \frac{1}{8},
\$
where $\hat\phi$ and $\hat\psi$ are the outputs of Algorithm \ref{mle+mle}.
\end{theorem}

\begin{proof}[Proof of Theorem \ref{counter}]
We construct the counter example as follows. Let $(x,y,z)\in\N_{+}\times\N_{+}\times\N_{+}$. We assume that the true parameter $(\phi^* ,\psi^* )=(\phi_1,\psi_1)$, which satisfies
\$
&\P_{\phi_1}(x=k,z=k)=\frac{1}{2^{k}}~~\forall k\in\N_+,\quad \P_{\phi_1}(x=m,z=n)=0 ~~\forall m\neq n,\notag\\
&\P_{\psi_1}(y=k|z=k)=1,~\forall k\in\N_{+}.
\$
For $i\geq 2$, we define $\P_{\phi_i}$ as follows,
\$
&\P_{\phi_i}(x=1,z=1)=\frac{1}{2}+\frac{1}{2^i},\quad\P_{\phi_i}(x=k,z=k)=\frac{1}{2^k}~\forall k\notin\{1,i\}\notag\\
&\P_{\phi_i}(x=m,z=n)=0~\forall m\neq n~{\rm or}~ m=n=i.
\$
We define $\P_{\psi_2}$ as follows, for any $k\in\N_+$,
\$
&\P_{\psi_2}(y=1|z=k)=\frac{1}{4},\quad\P_{\psi_2}(y=2|z=k)=\frac{1}{2}\notag\\
&\P_{\psi_2}(y=j|z=k)=\frac{1}{2^j}~\forall j\notin\{1,2\}.
\$
We denote $
\Phi:=\{\phi_i\,|\, i\in\N_+\}$ and $\Psi:=\{\psi_1,\psi_2\}$. In the sequel, we show that Algorithm \ref{mle+mle} fails on this case. Recall that we denote by $\{x_i\}^m_{i=1}$ and $\{x_j,y_j\}^n_{j=1}$ the unlabeled data and labeled data, respectively. We have the following observations:
\begin{itemize}
    \item We define $i:=\min\{k\neq 1\,|\, k\notin\{x_i\}^m_{i=1}\}$. If we have $1\in\{x_i\}^m_{i=1}$, then the maximizer of likelihood function $\hat\phi$ satisfies $\hat\phi=\phi_i$.
    \item Suppose that $\hat\phi=\phi_i$ for some $i\neq 1$ and $i\in\{y_j\}^n_{j=1}$. We then have $\hat\psi=\psi_2$.
\end{itemize}

We define the event $\mathcal{E}:=\{\exists i\neq 1, \text{ such that } \hat\phi=\phi_i\text{ and }i\in\{y_j\}^n_{j=1}\}$.
Under event $\mathcal{E}$, we have $\hat\phi=\phi_i$ for some $i\neq 1$ and $\hat\psi=\psi_2$, which implies
\%
\TV\big(\P_{\hat\phi,\hat\psi}(x,y),\P_{\phi^* ,\psi^* }(x,y)\big)&=\frac{1}{2}\int\int |p_{\phi_i,\psi_2}(x,y)-p_{\phi_1,\psi_1}(x,y)|\,dxdy\notag\\
&\geq \frac{1}{2}\int\bigg|\int p_{\phi_i,\psi_2}(x,y)-p_{\phi_1,\psi_1}(x,y)\,dx\bigg|\,dy\notag\\
&=\frac{1}{2}\int |p_{\phi_i,\psi_2}(y)-p_{\phi_1,\psi_1}(y)|\,dy\notag\\
&\geq \frac{1}{2}|\P_{\phi_i,\psi_2}(y=2)-\P_{\phi_1,\psi_1}(y=2)|=\frac{1}{8}
\%
In the following, we only need to lower bound the probability of event $\mathcal{E}$. Note that
\%
\P(\mathcal{E})&=\P\big(\cup^{\infty}_{i=2}\big\{\hat\phi=\phi, i\in\{y_j\}^n_{j=1}\big\}\big)\notag\\
&=\sum^{\infty}_{i=2}\P\big(\hat\phi=\phi_i, i\in\{y_j\}^n_{j=1}\big)\notag\\
&=\sum^{\infty}_{i=2}\P(\hat\phi=\phi_i)\cdot\P\big(i\in\{y_j\}^n_{j=1}\big)\notag\\
&=\sum^{\infty}_{i=2}\bigg(1-\bigg(1-\frac{1}{2^i}\bigg)^n\bigg)\cdot\P(\hat\phi=\phi_i).
\%
Thus, it holds for any $L\geq 2$ that
\%\label{092110}
\P(\mathcal{E})&\geq\sum^{L}_{i=2}\bigg(1-\bigg(1-\frac{1}{2^i}\bigg)^n\bigg)\cdot\P(\hat\phi=\phi_i)\notag\\
&\geq \bigg(1-\bigg(1-\frac{1}{2^L}\bigg)^n\bigg)\cdot\P\big(\exists 2\leq i\leq L, \hat\phi=\phi_i\big).
\%
Note that
\%\label{092111}
&\P\big(\exists 2\leq i\leq L, \hat\phi=\phi_i\big)\notag\\
&=\P\Big(\big\{1\in\{x_i\}^m_{i=1}\big\}\cap\big\{\exists 2\leq i\leq L, i\notin\{x_i\}^m_{i=1}\big\}\Big)\notag\\
&\geq\P\Big(\big\{1\in\{x_i\}^m_{i=1}\big\}\cap\big\{ L\notin\{x_i\}^m_{i=1}\big\}\Big)\notag\\
&\geq \P\big(1\in\{x_i\}^m_{i=1}\big)+\P\big(L\notin\{x_i\}^m_{i=1}\big)-1\notag\\
&=\P\big(L\notin\{x_i\}^m_{i=1}\big)-\P\big(1\notin\{x_i\}^m_{i=1}\big)\notag\\
&=\bigg(1-\frac{1}{2^L}\bigg)^m-\frac{1}{2^m}.
\%
Combining \eqref{092110} and \eqref{092111}, we have for any $L\geq 2$ 
\%
\P(\mathcal{E})\geq\bigg(1-\bigg(1-\frac{1}{2^L}\bigg)^n\bigg)\cdot\bigg(\bigg(1-\frac{1}{2^L}\bigg)^m-\frac{1}{2^m}\bigg).
\%
Setting $2^L=m=n$, we obtain that
\%
\P(\mathcal{E})\geq\bigg(1-\bigg(1-\frac{1}{m}\bigg)^m\bigg)\cdot\bigg(\bigg(1-\frac{1}{m}\bigg)^m-\frac{1}{2^m}\bigg)\rightarrow (1-e^{-1})\cdot e^{-1},\text{ as }m\rightarrow \infty.
\%
Thus, for any $c>0$, there exists $m,n\geq c$ such that 
\$
\P(\mathcal{E})\geq\frac{1}{2}(1-e^{-1})\cdot e^{-1}.
\$
\end{proof}

\end{document}